\documentclass[11pt]{article}

\usepackage{fullpage}
\usepackage{graphicx}
\usepackage[dvipsnames]{xcolor}

\usepackage{float}

\usepackage{amsmath}
\usepackage{amssymb}
\usepackage{amsthm}
\usepackage{bm}
\usepackage{bbm}

\usepackage{multirow}
\usepackage[backref=page]{hyperref}
\usepackage{booktabs}
\usepackage{multirow}

\newtheorem{theorem}{Theorem}
\newtheorem{proposition}{Proposition}
\newtheorem{lemma}{Lemma}
\newtheorem{corollary}{Corollary}
\newtheorem{definition}{Definition}
\newtheorem{assumption}{Assumption}
\newtheorem{example}{Example}
\newtheorem{remark}{Remark}

\usepackage[sort]{natbib}

\usepackage{color}

\hypersetup{
    colorlinks=true,
    citecolor=Blue,
    linkcolor=Blue,
    urlcolor=Blue,
}

\title{Towards Understanding the Mechanism of Contrastive Learning via Similarity Structure: A Theoretical Analysis}

\author{%
        \begin{tabular}{c}
             Hiroki Waida$^1$  \\
             waida.h.aa@m.titech.ac.jp
        \end{tabular} \and 
        \begin{tabular}{c}
             Yuichiro Wada$^{2,3}$  \\
             wada.yuichiro@fujitsu.com
        \end{tabular} \and
        \begin{tabular}{c}
             Léo Andéol$^{4,5,6,7}$  \\
             leo.andeol@math.univ-toulouse.fr
        \end{tabular} \and
        \begin{tabular}{c}
             Takumi Nakagawa$^{1,3}$  \\
             nakagawa.t.as@m.titech.ac.jp
        \end{tabular} \and
        \begin{tabular}{c}
             Yuhui Zhang$^1$  \\
             zhang.y.av@m.titech.ac.jp
        \end{tabular} \and
        \begin{tabular}{c}
             Takafumi Kanamori$^{1,3}$  \\
             kanamori@c.titech.ac.jp
        \end{tabular}}

\date{$^1$Tokyo Institute of Technology, Japan\\
      $^2$Fujitsu, Japan\\
      $^3$RIKEN AIP, Japan\\
      $^4$Institut de Mathématiques de Toulouse, France\\
      $^5$SNCF, France\\
      $^6$Université de Toulouse, France\\
      $^7$CNRS, France}

\begin{document}

\maketitle

\begin{abstract}
Contrastive learning is an efficient approach to self-supervised representation learning.
Although recent studies have made progress in the theoretical understanding of contrastive learning, the investigation of how to characterize the clusters of the learned representations is still limited.
In this paper, we aim to elucidate the characterization from theoretical perspectives. 
To this end, we consider a kernel-based contrastive learning framework termed Kernel Contrastive Learning (KCL), where kernel functions play an important role when applying our theoretical results to other frameworks.
We introduce a formulation of the similarity structure of learned representations by utilizing a statistical dependency viewpoint.
We investigate the theoretical properties of the kernel-based contrastive loss via this formulation.
We first prove that the formulation characterizes the structure of representations learned with the kernel-based contrastive learning framework.
We show a new upper bound of the classification error of a downstream task, which explains that our theory is consistent with the empirical success of contrastive learning. 
We also establish a generalization error bound of KCL.
Finally, we show a guarantee for the generalization ability of KCL to the downstream classification task via a surrogate bound.
\end{abstract}

\section{Introduction}
\label{sec:introduction}

Recently, many studies on self-supervised representation learning have been paying much attention to contrastive learning~\citep{chen2020simple,chen2020improved,caron2020unsupservised,haochen2021provable,dwibedi2021little,li2021selfsupervised}.
Through contrastive learning, encoder functions acquire how to encode unlabeled data to good representations by utilizing some information of similarity behind the data, where recent works~\citep{chen2020simple,chen2020improved,dwibedi2021little} use several data augmentation techniques to produce pairs of similar data.
It is empirically shown by many works~\citep{chen2020simple,chen2020improved,caron2020unsupservised,haochen2021provable,dwibedi2021little} that contrastive learning produces effective representations that are fully adaptable to downstream tasks, such as classification and transfer learning.

Besides the practical development of contrastive learning, the theoretical understanding of contrastive learning is essential to construct more efficient self-supervised learning algorithms.
In this paper, we tackle the following fundamental question of contrastive learning from the theoretical side: 
\textit{How are the clusters of feature vectors output from an encoder model pretrained by contrastive learning characterized?}

Recently, several works have shed light on several theoretical perspectives on this problem to study the generalization guarantees of contrastive learning to downstream classification tasks~\citep{pmlr-v97-saunshi19a,dufumier2022rethinking,haochen2021provable,huang2021towards,wang2022chaos,haochen2022theoretical,zhao2023arcl}.
One of the primary approaches of these works is to introduce some similarity measures in the data. \citet{pmlr-v97-saunshi19a} has introduced the conditional independence assumption, which assumes that data $x$ and its positive data $x^{+}$ are sampled independently according to the conditional probability distribution $\mathcal{D}_{c}$, given the \textit{latent class} $c$ drawn from the latent class distribution.
Although the concepts of latent classes and conditional independence assumption are often utilized to formulate semantic similarity of $x$ and $x^{+}$~\citep{pmlr-v97-saunshi19a,ash2022investigating,awasthi2022do,bao2022on,zou2023generalization}, it is pointed out by several works~\citep{haochen2021provable,wang2022chaos} that this assumption can be violated in practice.
Several works~\citep{haochen2021provable,wang2022chaos} have introduced different ideas about the similarity between data to alleviate this assumption.
\citet{haochen2021provable} have introduced the notion called \textit{population augmentation graph} to provide a theoretical analysis for Spectral Contrastive Learning (SCL) without the conditional independence assumption on $x$ and $x^{+}$.
Some works also focus on various graph structures~\citep{dufumier2022rethinking,haochen2022theoretical,wang2022chaos}.
Although these studies give interesting insights into contrastive learning, the applicable scope of their analyses is limited to specific objective functions.
Recently, several works~\citep{huang2021towards,zhao2023arcl} consider the setup where raw data in the same latent class are aligned well in the sense that the \textit{augmented distance} is small enough.
Although their theoretical guarantees can apply to multiple contrastive learning frameworks, their assumptions on the function class of encoders are strong, and it needs to be elucidated whether their guarantees can hold in practice.
Therefore, the investigation of the above question from unified viewpoints is ongoing, and more perspectives are required to understand the structure learned by contrastive learning.

\subsection{Contributions}
\label{subsec:our contributions}

In this paper, we aim to theoretically investigate the above question from a unified perspective by introducing a formulation based on a statistical similarity between data.
The main contributions of this paper are summarized below:

\begin{enumerate}
\item Since we aim to elucidate the mechanism of contrastive learning, we need to consider a unified framework that can apply to others, not specific frameworks such as SimCLR~\citep{chen2020simple} and SCL~\citep{haochen2021provable}.
\citet{li2021selfsupervised} pointed out that kernel-based self-supervised learning objectives are related to other contrastive losses, such as the InfoNCE loss~\citep{oord2018representation,chen2020simple}.
Therefore, via a kernel-based contrastive learning framework, other frameworks can be investigated through the lens of kernels.
Motivated by this, we utilize the framework termed \textit{Kernel Contrastive Learning} (KCL) as a tool for achieving the goal.
The loss of KCL, which is called \textit{kernel contrastive loss}, is a contrastive loss that has a simple and general form, where the similarity between two feature vectors is measured by a reproducing kernel~\citep{berlinet2004reproducing,steinwart2008support,aronszajn1950theory} (Section~\ref{subsec:introduction to kernel contrastive loss}).
One of our contributions is employing KCL to study the mechanism of contrastive learning from a new unified theoretical perspective.
\item We introduce a new formulation of similarity between data (Section~\ref{subsec:assumptions and definitions}). Our formulation of similarity begins with the following intuition: if raw or augmented data $x$ and $x'$ belong to the same class, then the similarity measured by some function should be higher than a threshold.
Following this, we introduce a formulation (Assumption~\ref{assumption:mixture of clusters}) based on the  similarity function \eqref{def:similarity function}.
\item We present the theoretical analyses towards elucidating the above question (Section~\ref{sec:theoretical results}).
We first show that KCL can distinguish the clusters of representations according to this formulation (Section~\ref{subsec:explaining what's going on contrastive learning}).
This result shows that our formulation is closely connected to the mechanism of contrastive learning.
Next, we establish a new upper bound for the classification error of the downstream task (Section~\ref{subsec:a surrogate upper bound of NN classification error}), which indicates that our formulation does not contradict the practical efficiency of contrastive learning shown by a line of work~\citep{chen2020simple,chen2020improved,haochen2021provable,dwibedi2021little}.
Notably, our upper bound is valid under more realistic assumptions on the encoder functions, compared to the previous works~\citep{huang2021towards,zhao2023arcl}.
We also establish the generalization error bound for KCL~(Section~\ref{subsec:rethinking generalization of contrastive learning}).
Finally, applying our theoretical results, we show a guarantee for the generalization of KCL to the downstream classification task via a surrogate bound (Section~\ref{subsec:main result section 5.4}).
\end{enumerate}

\subsection{Related Work}
\label{subsec:related work}

Contrastive learning methods have been investigated from the empirical side~ \citep{caron2020unsupservised,chen2020simple,chen2020improved,chen2021intriguing}.
\citet{chen2020simple} propose a method called SimCLR, which utilizes a variant of InfoNCE~\citep{oord2018representation,chen2020simple}.
Several works have recently improved contrastive methods from various viewpoints~\citep{robinson2020contrastive,dwibedi2021little,robinson2021shortcuts,caron2020unsupservised}.
Contrastive learning is often utilized in several fundamental tasks, such as clustering~\citep{vangansbeke2020scan} and domain adaptation~\citep{singh2021clda}, and applied to some domains such as vision~\citep{chen2020simple}, natural language processing~\citep{gao2021simcse}, and speech~\citep{jiang2021speech}.
Besides the contrastive methods, several works~\citep{grill2020bootstrap,chen2021exploring} also study non-contrastive methods.
Investigation toward the theoretical understanding of contrastive learning is also a growing focus.
For instance, the generalization ability of contrastive learning to the downstream classification task has been investigated from many kinds of settings~\citep{pmlr-v97-saunshi19a,tosh2021contrastive,haochen2021provable,wang2022chaos,bao2022on,saunshi2022understanding,huang2021towards,haochen2022theoretical,zhao2023arcl}.
Several works investigate contrastive learning from various theoretical and empirical viewpoints to elucidate its mechanism, such as the geometric properties of contrastive losses~\citep{wang2020understanding,huang2021towards},
formulation of similarity between data~\citep{pmlr-v97-saunshi19a,haochen2021provable,kugelgen2021self,wang2022chaos,huang2021towards,dufumier2022rethinking,zhao2023arcl},
inductive bias~\citep{saunshi2022understanding,haochen2022theoretical},
transferablity~\citep{haochen2022beyond,shen2022connect,zhao2023arcl},
feature suppression~\citep{chen2021intriguing,robinson2021shortcuts}, negative sampling methods~\citep{Chuan2020debiased,robinson2020contrastive}, and optimization viewpoints~\citep{pmlr-v139-wen21c,tian2022understanding}.

Several works~\citep{li2021selfsupervised,zhang2022fmutual,tsai2022conditional,johnson2022contrastive,dufumier2022rethinking,kiani2022joint} study the connection between contrastive learning and the theory of kernels.
\citet{li2021selfsupervised} investigate some contrastive losses, such as InfoNCE, from a kernel perspective.
\citet{zhang2022fmutual} show a relation between the kernel method and $f$-mutual information and apply their theory to contrastive learning.
\citet{tsai2022conditional} tackle the conditional sampling problem using kernels as similarity measurements.
\citet{dufumier2022rethinking} consider incorporating prior information in contrastive learning by using the theory of kernel functions.
\citet{kiani2022joint} connect several self-supervised learning algorithms to kernel methods through optimization problem viewpoints.
Note that different from these works, our work employs kernel functions to investigate a new unified perspective of contrastive learning via the statistical similarity.

Many previous works investigate various interpretations of self-supervised representation learning objectives. For instance, the InfoMax principle~\citep{poole2019variational,Tschannen2020On}, spectral clustering~\citep{haochen2021provable} (see \citet{ng2002spectral} for spectral clustering), and Hilbert-Schmidt Independence Criterion (HSIC)~\citep{li2021selfsupervised} (see~\citet{gretton2005measuring} for HSIC).
However, the investigation of contrastive learning from unified perspectives is worth addressing to elucidate its mechanism, as recent works on self-supervised representation learning tackle it from the various standpoints~\citep{huang2021towards,tian2022understanding,johnson2022contrastive,kiani2022joint,dubois2022improving}.

\section{Preliminaries}
\label{sec:kcl}

\subsection{Problem Setup}
\label{subsec:problem setup}

We give the standard setup of contrastive learning.
Our setup closely follows that of \citet{haochen2021provable}, though we also introduce additional technically necessary settings to maintain the mathematical rigorousness.
Let $\overline{\mathbb{X}}\subset\mathbb{R}^{p}$ be a topological space consisting of raw data, and let $P_{\overline{\mathbb{X}}}$ be a Borel probability measure on $\overline{\mathbb{X}}$.
A line of work on contrastive learning~\citep{chen2020simple,chen2020improved,dwibedi2021little,haochen2021provable} uses data augmentation techniques to obtain similar augmented data points.
Hence, we define a set $\mathcal{T}$ of maps transforming a point $\overline{x}\in\overline{\mathbb{X}}$ into $\mathbb{R}^{p}$, where we assume that $\mathcal{T}$ includes the identity map on $\mathbb{R}^{p}$.
Then, let us define $\mathbb{X}=\bigcup_{t\in\mathcal{T}}\{t(\overline{x})\,:\,\overline{x}\in\overline{\mathbb{X}}\}$.
Every element $t$ in $\mathcal{T}$ can be regarded as a map returning an augmented data $x=t(\overline{x})$ for a raw data point $\overline{x}\in\overline{\mathbb{X}}$.
Note that since the identity map belongs to $\mathcal{T}$, $\overline{\mathbb{X}}$ is a subset of $\mathbb{X}$.
We endow $\mathbb{X}$ with some topology.
Let $\nu_{\mathbb{X}}$ be a $\sigma$-finite and non-negative Borel measure in $\mathbb{X}$.
Following the idea of \citet{haochen2021provable}, we denote $a(x|\overline{x})$ as the conditional probability density function of $x$ given $\overline{x}\sim P_{\overline{\mathbb{X}}}$ and define the weight function $w:\mathbb{X}\times \mathbb{X}\to\mathbb{R}$ as $w(x,x')=\mathbb{E}_{\overline{x}\sim P_{\overline{\mathbb{X}}}}\left[a(x|\overline{x})a(x'|\overline{x})\right].$
From the definition, $w$ is a joint probability density function on $\mathbb{X}\times\mathbb{X}$.
Let us define the marginal $w(\cdot)$ of the weight function to be $w(x)=\int w(x,x')d\nu_{\mathbb{X}}(x')$.
The marginal $w(\cdot)$ is also a probability density function on $\mathbb{X}$, and the corresponding probability measure is denoted by $dP_{\mathbb{X}}(x)=w(x)d\nu_{\mathbb{X}}(x)$.
Denote by $\mathbb{E}_{x,x^{+}}[\cdot], \mathbb{E}_{x,x^{-}}[\cdot]$ respectively, the expectation w.r.t. the probability measure $w(x,x')d\nu_{\mathbb{X}}^{\otimes 2}(x,x'), w(x)w(x')d\nu_{\mathbb{X}}^{\otimes 2}(x,x')$ on $\mathbb{X}\times\mathbb{X}$, where $\nu_{\mathbb{X}}^{\otimes 2}:=\nu_{\mathbb{X}}\otimes \nu_{\mathbb{X}}$ is the product measure on $\mathbb{X}\times\mathbb{X}$.
To rigorously formulate our framework of contrastive learning, we assume that the marginal $w$ is positive on $\mathbb{X}$.
Indeed, a point $x\in\mathbb{X}$ satisfying $w(x)=0$ is not included in the support of $a(\cdot|\overline{x})$ for $P_{\overline{\mathbb{X}}}$-almost surely $\overline{x}\in\overline{\mathbb{X}}$, which means that such a point $x$ merely appears as augmented data.

Let $f_{0}:\mathbb{X}\to\mathbb{R}^{d}$ be an encoder model mapping augmented data to the feature space, and let $\mathcal{F}_{0}$ be a class of functions consisting of such encoders.
In practice, $f_{0}$ is defined by a backbone architecture~(e.g., ResNet~\citep{he2016deep}; see~\citet{chen2020simple}), followed by the additional multi-layer perceptrons called \textit{projection head}~\citep{chen2020simple}.
We assume that $\mathcal{F}_{0}$ is uniformly bounded, i.e., there exists a universal constant $c\in\mathbb{R}$ such that $\sup_{f_{0}\in\mathcal{F}_{0}}\sup_{x\in\mathbb{X}}\|f_{0}(x)\|_{2}\leq c$.
For instance, a function space of bias-free fully connected neural networks on a bounded domain, where every neural network has the continuous activate function at each layer, satisfies this condition.
Since a feature vector output from the encoder model is normalized using the Euclidean norm in many empirical studies~\citep{chen2020simple,dwibedi2021little} and several theoretical studies~\citep{wang2020understanding,wang2022chaos}, we consider the function space of normalized functions $\mathcal{F}=\{f\;|\;\exists f_{0}\in\mathcal{F}_{0},\;f(x)=f_{0}(x)/\|f_{0}(x)\|_{2}\textup{ for }\forall x\in\mathbb{X}\}$.
Here, to guarantee that every $f\in\mathcal{F}$ is well-defined, suppose that $\mathfrak{m}(\mathcal{F}_{0}):=\inf_{f\in\mathcal{F}}\inf_{x\in\mathbb{X}}\|f_{0}(x)\|_{2}>0$ holds.

Finally, we introduce several notations used throughout this paper.
Let $\mathbb{M}\subset\mathbb{X}$ be a measurable set, then we write
\begin{align*}
\mathbb{E}[f(x)|\mathbb{M}]:=\int_{\mathbb{X}}f(x)P_{\mathbb{X}}(dx|\mathbb{M})=P_{\mathbb{X}}(\mathbb{M})^{-1}\int_{\mathbb{M}}f(x)w(x)d\nu_{\mathbb{X}}(x).
\end{align*}
We also use the notation $\mathbb{E}[f(x);\mathbb{M}]:=\int_{\mathbb{M}}f(x)w(x)d\nu_{\mathbb{X}}(x)$.
Denote by $\mathbbm{1}_{\mathbb{M}}(\cdot)$, the indicator function of a set $\mathbb{M}$.
We also use $[n]:=\{1,\cdots,n\}$ for $n\in\mathbb{N}$.

\subsection{Reproducing Kernels}
\label{subsec:reproducing kernel}

We provide several notations of reproducing kenrels~\citep{berlinet2004reproducing,steinwart2008support,aronszajn1950theory}.
Let $k:\mathbb{S}^{d-1}\times\mathbb{S}^{d-1}\to\mathbb{R}$ be a real-valued, continuous, symmetric, and positive-definite kernel, where $\mathbb{S}^{d-1}$ denotes the unit hypersphere centered at the origin $\bm{0}\in\mathbb{R}^{d}$, and the positive-definiteness means that for every $\{z_{i}\}_{i=1}^{n}\subset\mathbb{S}^{d-1}$ and $\{c_{i}\}_{i=1}^{n}\subset\mathbb{R}$, $\sum_{i,j=1}^{n}c_{i}c_{j}k(z_{i},z_{j})\geq 0$ holds~\citep{berlinet2004reproducing}.
Let $\mathcal{H}_{k}$ be the Reproducing Kernel Hilbert Space (RKHS) with kernel $k$~\citep{aronszajn1950theory}, which satisfies $\phi(z)=\langle \phi,k(\cdot,z)\rangle_{\mathcal{H}_{k}}$ for all $\phi\in\mathcal{H}_{k}$ and $z\in\mathbb{S}^{d-1}$.
Denote $h(z)=k(\cdot, z)$ for $z\in\mathbb{S}^{d-1}$, where such a map is often called feature map \citep{steinwart2008support}.
Here, we impose the following condition.
\begin{assumption}
\label{assumption:kernel}
For the kernel function $k$, there exists some $\rho$-Lipschitz function $\psi:[-1,1]\to\mathbb{R}$ such that for every $z,z'\in\mathbb{S}^{d-1}$, $k(z,z')=\psi(z^{\top}z')$ holds.
\end{assumption}
Several popular kernels in machine learning such as the linear kernel, quadratic kernel, and Gaussian kernel, satisfy Assumption~\ref{assumption:kernel}.
Note that the Lipschitz condition in Assumption~\ref{assumption:kernel} is useful to derive the generalization error bound for the kernel contrastive loss (see Section~\ref{subsec:rethinking generalization of contrastive learning}), and sometimes it can be removed when analyzing for a specific kernel. We use this assumption to present general results.
Here, we also use the following notion in this paper:

\begin{proposition}
\label{prop:kernel mean embedding}
Let $\mathbb{M}\subset \mathbb{X}$ be a measurable set and $f\in\mathcal{F}$. 
Define $\mu_{\mathbb{M}}(f):=\mathbb{E}_{P_{\mathbb{X}}}[h(f(x))|\mathbb{M}]$.
Then, $\mu_{\mathbb{M}}(f)\in\mathcal{H}_{k}$.
\end{proposition}

The quantity $\mu_{\mathbb{M}}(f)$ with a measurable subset $\mathbb{M}$ can be regarded as a variant of the kernel mean embedding~\citep{muandet2017kernel}.
The proof of Proposition~\ref{prop:kernel mean embedding} is a slight modification of ~\citet{muandet2017kernel}; see Appendix~\ref{appsubsec:proof of proposition kernel mean embedding}.

\section{Kernel Contrastive Learning}
\label{subsec:introduction to kernel contrastive loss}

In this section, we introduce a contrastive learning framework to analyze the mechanism of contrastive learning.
In representation learning, the InfoNCE loss~\citep{oord2018representation,chen2020simple} are widely used in application domains such as vision~\citep{chen2020simple,chen2021intriguing,dwibedi2021little}.
Following previous works~\citep{oord2018representation,chen2020simple,bao2022on}, we define the InfoNCE loss as,
\begin{align*}
    L_{\textup{NCE}}(f;\tau)=-\mathbb{E}_{\substack{x,x^{+}\\\{x_{i}^{-}\}{\sim}P_{\mathbb{X}}}}\left[\log\frac{e^{f(x)^{\top}f(x^{+})/\tau}}{e^{f(x)^{\top}f(x^{+})/\tau}+\sum_{i=1}^{M}e^{f(x)^{\top}f(x_{i}^{-})/\tau}}\right],
\end{align*}
where $\{x_{i}^{-}\}$ are i.i.d. random variables, $\tau>0$, and $M$ is the number of negative samples.
\citet{wang2020understanding} introduce the asymptotic of the InfoNCE loss:
\begin{align*}
    L_{\infty\textup{-NCE}}(f;\tau)=-\mathbb{E}_{x,x^{+}}\left[\frac{f(x)^{\top}f(x^{+})}{\tau}\right]+\mathbb{E}_{x}\left[\log\mathbb{E}_{x'}\left[e^{\frac{f(x)^{\top}f(x')}{\tau}}\right]\right].
\end{align*}
According to the theoretical analysis of~\citet{wang2020understanding}, they show that the first term represents the \textit{alignment}, i.e., the averaged closeness of feature vectors of the pair $(x,x^{+})$, while the second one indicates the \textit{uniformity}, i.e., how far apart the feature vectors of negative samples $x,x'$ are.
Besides, \citet{chen2021intriguing} report the efficiency of the generalized contrastive losses, which have the additional weight hyperparameter.
Meanwhile, since we aim to study the mechanism of contrastive learning, a simple and general form of contrastive losses related to other frameworks is required.
Here, \citet{li2021selfsupervised} find the connection between self-supervised learning and kernels by showing that some HSIC criterion is proportional to the objective function $\mathbb{E}_{x,x^{+}}[k(f(x),f(x^{+}))]-\mathbb{E}_{x,x^{-}}[k(f(x),f(x^{-}))]$ (for more detail, see Appendix~\ref{appsubsubsec:relations to ssl-hsic}).
Motivated by this connection, we consider a contrastive learning objective where a kernel function measures the similarity of the feature vectors of augmented data points.
More precisely, for the kernel function $k$ introduced in Section~\ref{subsec:reproducing kernel}, we define the kernel contrastive loss as,
\begin{align*}
    L_{\textup{KCL}}(f;\lambda)=-\mathbb{E}_{x,x^{+}}\left[k(f(x),f(x^{+}))\right]+\lambda \mathbb{E}_{x,x^{-}}\left[k(f(x),f(x^{-}))\right],
\end{align*}
where the weight hyperparameter $\lambda$ is inspired by \citet{chen2021intriguing}.
Here, the kernel contrastive loss $L_{\textup{KCL}}$ is minimized during the pretraining stage of contrastive learning.
Throughout this paper, the contrastive learning framework with the kernel contrastive loss is called \textit{Kernel Contrastive Learning} (KCL).

Next, we show the connections to other contrastive learning objectives.
First, for the InfoNCE loss, we consider the linear kernel contrastive loss $L_{\textup{LinKCL}}(f;\lambda)$ defined by selecting $k(z,z')=z^{\top}z'$.
Note that $L_{\textup{LinKCL}}$ and its empirical loss are also discussed in several works~\citep{wang2021understanding,huang2021towards}.
For $L_{\textup{LinKCL}}(f;1)$, we have,
\begin{align}
\label{eq:lin vs nce}
    \tau^{-1}L_{\textup{LinKCL}}(f;1)\leq L_{\textup{NCE}}(f;\tau)+\log M^{-1}.
\end{align}
In Appendix~\ref{appsubsec:supplementary information of section 3.2}, we show a generalized inequality of \eqref{eq:lin vs nce} for the generalized loss~\citep{chen2021intriguing}.
Note that similar relations hold when $L_{\textup{NCE}}$ is replaced with the asymptotic loss~\citep{wang2020understanding} or decoupled contrastive learning loss~\citep{yeh2021decoupled}; see Appendix~\ref{appsubsec:supplementary information of section 3.2}.
Therefore, it is possible to analyze the InfoNCE loss and its variants via $L_{\textup{LinKCL}}(f;\lambda)$.

The kernel contrastive loss is also related to other contrastive learning objectives.
For instance, the quadratic kernel contrastive loss with the quadratic kernel $k(z,z')=(z^{\top}z')^{2}$ becomes a lower bound of the spectral contrastive loss \citep{haochen2021provable} up to an additive constant (see Appendix~\ref{appsubsec:supplementary information of section 3.2}).
Thus, theoretical analyses of the kernel contrastive loss can apply to other contrastive learning objectives.

Note that we empirically demonstrate that the KCL frameworks with the Gaussian kernel and quadratic kernel work, although simple; see Appendix~\ref{appsubsec:experimental setup} for the experimental setup and Appendix~\ref{appsubsec:results of linear evaluation} and \ref{appsubsec:ablation study} for the results in the supplementary material. The experimental results also motivate us to use KCL as a theoretical tool for studying contrastive learning.

\section{A Formulation Based on Statistical Similarity}
\label{subsec:assumptions and definitions}

\subsection{Key Ingredient: Similarity Function}
\label{subsubsec:key ingredient}

To study the mechanism of contrastive learning, we introduce a notion of similarity between two augmented data points, which is a key component in our analysis.
Let us define,
\begin{align}
\label{def:similarity function}
\text{sim}(x,x';\lambda):=\frac{w(x,x')}{w(x)w(x')}-\lambda,
\end{align}
where $\lambda\geq 0$ is the weight parameter of $L_{\textup{KCL}}$, and $w(x,x')$ and $w(x)$ have been introduced in Section~\ref{subsec:problem setup}.
Note that \eqref{def:similarity function} is well-defined since $w(x)>0$ holds for every $x\in\mathbb{X}$.
The quantity $\text{sim}(x,x';\lambda)$ represents how much statistical dependency $x$ and $x'$ have.
The density ratio $w(x,x')/(w(x)w(x'))$ can be regarded as an instance of \textit{point-wise dependency} introduced by \citet{TsaiNeural2020}.
The hyperparameter $\lambda$ controls the degree of relevance between two augmented data $x,x'$ via their (in-)dependency.
For instance, with the fixed $\lambda=1$, $\text{sim}(x,x';1)$ is positive if $w(x,x')>w(x)w(x')$, i.e., $x$ and $x'$ are correlated.

Here, we note that several theoretical works on representation learning~\citep{tosh2021contrastive,johnson2022contrastive,wang2022spectral} use the density ratio to study the optimal representations of several contrastive learning objectives.
\citet{tosh2021contrastive} focus on the fact that the minimizer of a logistic loss can be written in terms of the density ratio and utilize it to study \textit{landmark embedding}~\citep{tosh2021jmlrcontrastive}.
\citet{johnson2022contrastive} connect the density ratio to the minimizers of several contrastive learning objectives and investigate the quality of representations related to the minimizers.
In addition, \citet{wang2022spectral} study the minimizer of the spectral contrastive loss~\citep{haochen2021provable}.
Meanwhile, we emphasize that the purpose of using the density ratio in~\eqref{def:similarity function} is not to study the optimal representation of KCL but to give a formulation based on the statistical similarity between augmented data.

\begin{remark}
We can show that the kernel contrastive loss can be regarded as a relaxation of the population-level Normalized Cut problem~\citep{terada2019kernel}, where the integral kernel is defined with \eqref{def:similarity function}. 
Thus, \eqref{def:similarity function} defines the similarity structure utilized by KCL.
Detailed arguments and comparison to related works \citep{haochen2021provable,tian2022understanding} can be found in Appendix~\ref{appsec:connections between kcl and normalized cut}.
\end{remark}

\subsection{Formulation and Example}
\label{subsubsec:assumptions and examples}
We introduce the following formulation based on our problem setting.

\begin{assumption}
\label{assumption:mixture of clusters}
There exist some $\delta\in\mathbb{R}$, number of clusters $K\in\mathbb{N}$, measurable subsets $\mathbb{M}_{1},\cdots,\mathbb{M}_{K}\subset\mathbb{X}$, and a deterministic labeling function $y:\mathbb{X}\to[K]$ such that the following conditions hold:
\begin{enumerate}
    \item[\textbf{\textup{(A)}}] $\bigcup_{i=1}^{K}\mathbb{M}_{i}=\mathbb{X}$ holds.
    \item[\textbf{\textup{(B)}}] For every $i\in[K]$, any points $x,x'\in\mathbb{M}_{i}$ satisfy $\textup{sim}(x,x';\lambda)\geq \delta$.
    \item[\textbf{\textup{(C)}}] For every $x\in\mathbb{X}$ and the set of indices $J_{x}=\{j\in[K]\;|\;x\in\mathbb{M}_{j}\}$, $y(x)\in J_{x}$ holds. Moreover, each set $\{x\in\mathbb{X}\;|\;y(x)=i\}$ is measurable.
\end{enumerate}
\end{assumption}

Assumption~\ref{assumption:mixture of clusters} does not require that $\mathbb{M}_{1},\cdots,\mathbb{M}_{K}$ are disjoint, which is a realistic setting, as \citet{wang2022chaos} show that clusters of augmented data can have inter-cluster connections depending on the strength of data augmentation.
The conditions \textbf{(A)} and \textbf{\textup{(B)}} in Assumption~\ref{assumption:mixture of clusters} guarantee that each subset $\mathbb{M}_{i}$ consists of augmented data that have high similarity.
The condition \textbf{\textup{(C)}} enables to incorporate label information in our analysis.
Note that several works on contrastive learning~\citep{haochen2021provable,saunshi2022understanding,haochen2022theoretical} also employ deterministic labeling functions.

These conditions are useful to analyze the theory of contrastive learning.
To gain more intuition, we provide a simple example that satisfies Assumption~\ref{assumption:mixture of clusters}.

\begin{example}[The proof can be found in Appendix~\ref{appsec:proofs in examples}]
\label{example:when the space meets the assumptions}
Suppose that $\mathbb{X}$ consists of disjoint open balls $\mathbb{B}_{1},\cdots,\mathbb{B}_{K}$ of the same radius in $\mathbb{R}^{p}$, and $\overline{\mathbb{X}}=\mathbb{X}$.
Let $a(x|\overline{x})=\textup{vol}(\mathbb{B}_{1})^{-1}\sum_{i=1}^{K}\mathbbm{1}_{\mathbb{B}_{i}\times\mathbb{B}_{i}}(x,\overline{x})$, and let $p_{\overline{\mathbb{X}}}(\overline{x})=(K\textup{vol}(\mathbb{B}_{1}))^{-1}$ be the probability density function of $P_{\overline{\mathbb{X}}}$. Then, $w(x)>0$ and $\textup{sim}(x,x';\lambda)=K\mathbbm{1}_{\bigcup_{i\in[K]}\mathbb{B}_{i}\times\mathbb{B}_{i}}(x,x')-\lambda$ hold.
Hence, for instance let $\lambda=1$ and $\delta=K-1$, and take  $\mathbb{M}_{i}:=\mathbb{B}_{i}$.
Also, define $y:\mathbb{X}\to[K]$ as $y(x)=i$ if $x\in\mathbb{B}_{i}$ for some $i\in[K]$.
Then, Assumption~\ref{assumption:mixture of clusters} is satisfied in this setting.
\end{example}

Here, theoretical formulations of similarity have been investigated by several works on contrastive learning~\citep{pmlr-v97-saunshi19a,haochen2021provable,kugelgen2021self,wang2022chaos,huang2021towards,haochen2022theoretical,dufumier2022rethinking,zhao2023arcl,parulekar2023infonce}.
The basic notions introduced by these works are: 
\textit{latent classes} of unlabeled data and conditional independence assumption~\citep{pmlr-v97-saunshi19a}, 
graph structures of augmented data~\citep{haochen2021provable,wang2022chaos,haochen2022theoretical,dufumier2022rethinking},
the decomposition of unlabled data into the 
\textit{content} (i.e., invariant against data augmentation) and \textit{style} (i.e., changeable by data augmentation) variables~\citep{kugelgen2021self,parulekar2023infonce},
and the geometric structure based on \textit{augmented distance} and the concentration within class subsets~\citep{huang2021towards,zhao2023arcl}.
Meanwhile, our formulation in Assumption~\ref{assumption:mixture of clusters} uses the similarity function~\eqref{def:similarity function}, which differs from the previous works.
Note that Assumption~\ref{assumption:mixture of clusters} has some relation to Assumption~3 in~\citet{haochen2022theoretical}; see Appendix~\ref{appsubsec:relation to haochen ma}.
Our formulation gives deeper insights into contrastive learning, as shown in Section~\ref{sec:theoretical results}.

\section{Theoretical Results}
\label{sec:theoretical results}

\subsection{KCL as Representation Learning with Statistical Similarity}
\label{subsec:explaining what's going on contrastive learning}

First, we connect the kernel contrastive loss to the formulation based on the similarity $\textup{sim}(x,x';\lambda)$.
The following theorem indicates that KCL has two effects on the way of representation learning, where the clusters $\mathbb{M}_{1},\cdots,\mathbb{M}_{K}$ involve to explain the mathematical relation.

\begin{theorem}
\label{thm:decomposition}
Suppose that Assumption~\ref{assumption:kernel} and \ref{assumption:mixture of clusters} hold.
Take $\delta\in\mathbb{R}$, $K\in\mathbb{N}$, and $\mathbb{M}_{1},\cdots,\mathbb{M}_{K}$ such that the conditions \textbf{\textup{(A)}} and \textbf{\textup{(B)}} in Assumption~\ref{assumption:mixture of clusters} are satisfied.
Then, the following inequality holds for every $f\in\mathcal{F}$:
\begin{align}
\label{eq:decomposition}
    \frac{\delta}{2}\cdot\mathfrak{a}(f)+\lambda\cdot\mathfrak{c}(f)\leq L_{\textup{KCL}}(f;\lambda)+R(\lambda).
\end{align}
where $R(\lambda)$ is a function of $\lambda$, $\mu_{i}(f)=\mu_{\mathbb{M}_{i}}(f)$, and
\begin{align*}
    \mathfrak{a}(f)&:=\sum_{i=1}^{K}\mathbb{E}_{x,x^{-}}\left[\|h(f(x))-h(f(x^{-}))\|_{\mathcal{H}_{k}}^{2};\mathbb{M}_{i}\times\mathbb{M}_{i}\right],\\
    \mathfrak{c}(f)&:=\sum_{i\neq j}P_{\mathbb{X}}(\mathbb{M}_{i})P_{\mathbb{X}}(\mathbb{M}_{j})\langle \mu_{i}(f),\mu_{j}(f)\rangle_{\mathcal{H}_{k}}.
\end{align*}
\end{theorem}

\begin{figure}[t]
    \centerline{\includegraphics[scale=0.4]{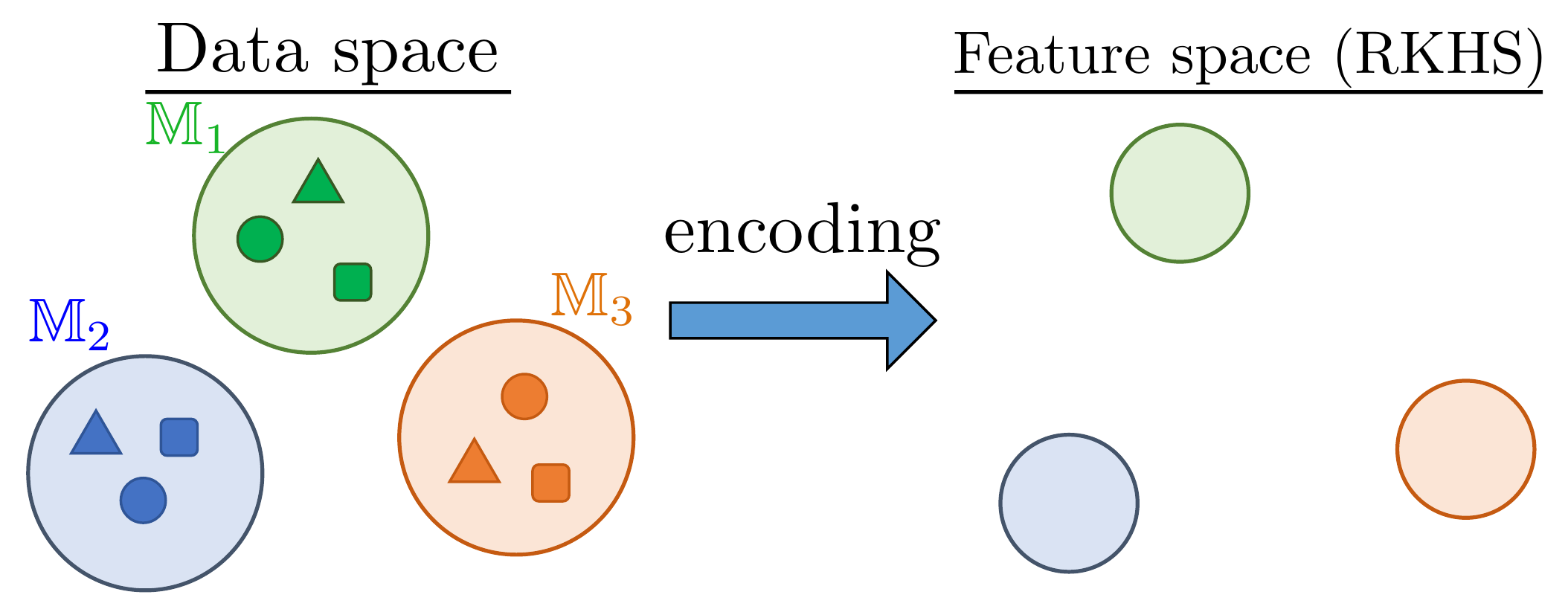}}
    \caption{An illustration of Theorem~\ref{thm:decomposition}. The clusters $\mathbb{M}_{1},\cdots,\mathbb{M}_{K}$ in the data space are mapped into the RKHS, where each cluster in the RKHS shrinks or expands (via $\mathfrak{a}(f)$) while maintaining the distance to other clusters (via $\mathfrak{c}(f)$).
    }
    \label{fig:kcl_paper_illustration}
\end{figure}
For the proof of Theorem~\ref{thm:decomposition}, see Appendix~\ref{appsubsec:proof of theorem 4.1}.
Note that the key point of the proof is the following usage of Assumption~\ref{assumption:mixture of clusters}: for any $x,x'\in\mathbb{M}_{i}$, the inequality $\textup{sim}(x,x';\lambda)\geq\delta$ implies the relation $w(x,x')\geq (\lambda+\delta)w(x)w(x')$.
Here we briefly explain each symbol in Theorem~\ref{thm:decomposition}.
The value $\mathfrak{a}(f)$ quantifies the concentration within each cluster consisting of the representations of augmented data in $\mathbb{M}_{i}$.
The quantity $\mathfrak{c}(f)$ measures how far the subsets $h(f(\mathbb{M}_{1})),\cdots,h(f(\mathbb{M}_{K}))$ are, since $\langle \mu_{i}(f),\mu_{j}(f)\rangle_{\mathcal{H}_{k}}=\int_{\mathbb{M}_{i}}\int_{\mathbb{M}_{j}}k(f(x),f(x'))P_{\mathbb{X}}(dx|\mathbb{M}_{i})P_{\mathbb{X}}(dx'|\mathbb{M}_{j})$ holds (see Lemma~\ref{lem:inner product} in Appendix~\ref{appsubsec:useful lemmas for theorem 4.1}).
These quantities indicate that representation learning by KCL can distinguish the subsets $\mathbb{M}_{1},\cdots,\mathbb{M}_{K}$ in the RKHS (see Figure~\ref{fig:kcl_paper_illustration} for illustration).
The function $R(\lambda)$ includes the term that represents the hardness of the pretraining task in the space of augmented data $\mathbb{X}$: if the overlaps between two different subsets $\mathbb{M}_{i}$ and $\mathbb{M}_{j}$ expand, then $R(\lambda)$ increases (the precise definition is given in Appendix~\ref{appsubsec:proof of theorem 4.1}). 

A key point of Theorem~\ref{thm:decomposition} is that $\delta$ and $\lambda$ can determine how learned representations distribute in the RKHS.
If $\delta>0$, then representations of augmented data in each $\mathbb{M}_{i}$ tend to align as controlling the trade-off between $\mathfrak{a}(f)$ and $\mathfrak{c}(f)$.
For $\delta\leq 0$, not just the means $\mu_{1}(f),\cdots,\mu_{K}(f)$ but also the representations tend to scatter.
We remark that $\delta$ depends on the fixed weight $\lambda$ due to the condition \textbf{(B)} in Assumption~\ref{assumption:mixture of clusters}.
Intuitively, larger $\lambda$ makes $\delta$ smaller and vice versa.

Here we should remark that under several assumptions, the equality holds in \eqref{eq:decomposition}, as shown below:
\begin{corollary}
\label{cor:equality holds}
Suppose that Assumption~\ref{assumption:kernel} and \ref{assumption:mixture of clusters} hold.
Take $\delta\in\mathbb{R}$, $K\in\mathbb{N}$, and $\mathbb{M}_{1},\cdots,\mathbb{M}_{K}$ such that the conditions \textbf{\textup{(A)}} and \textbf{\textup{(B)}} in Assumption~\ref{assumption:mixture of clusters} are satisfied.
Suppose that $\mathbb{M}_{1},\cdots,\mathbb{M}_{K}$ are disjoint, and for every pair $(i,j)\in[K]\times[K]$ such that $i\neq j$, every $(x,x')\in\mathbb{M}_{i}\times\mathbb{M}_{j}$ satisfies $w(x,x')=0$.
Suppose that for every $i\in[K]$, it holds that $\textup{sim}(x,x';\lambda)=\delta$ for any $x,x'\in\mathbb{M}_{i}$.
Then, for every $f\in\mathcal{F}$, the equality holds in \eqref{eq:decomposition}, i.e.,
\begin{align*}
    L_{\textup{KCL}}(f;\lambda)=\frac{\delta}{2}\cdot\mathfrak{a}(f)+\lambda\cdot\mathfrak{c}(f)-R(\lambda).
\end{align*}
\end{corollary}
The proof of Corollary~\ref{cor:equality holds} can be found in Appendix~\ref{appsubsec:proof of corollary equality holds}.
The above corollary means that, under these assumptions, the minimization of the kernel contrastive loss is equivalent to that of the objective function $(\delta/2)\cdot\mathfrak{a}(f)+\lambda\cdot\mathfrak{c}(f)$.
Note that Example~\ref{example:when the space meets the assumptions} satisfies all the assumptions enumerated in the statement.
In summary, Theorem~\ref{thm:decomposition} and Corollary~\ref{cor:equality holds} imply that contrastive learning by the KCL framework is characterized as representation learning with the similarity structure of augmented data space $\mathbb{X}$.

\subsubsection{Comparison to Related Work}
The quantity $\mathfrak{a}(f)$ is closely related to the property called \textit{alignment}~\citep{wang2020understanding} since the representations of similar data are learned to be close in the RKHS.
Also, the quantity $\mathfrak{c}(f)$ has some connection to \textit{divergence property}~\citep{huang2021towards} since it measures how far apart the means $\mu_{i}(f)$ and $\mu_{j}(f)$ are.
Although the relations between these properties and contrastive learning have been pointed out by \citet{wang2020understanding,huang2021towards}, we emphasize that our result gives a new characterization of the learned clusters.
Furthermore, this theorem also implies that the trade-off between $\mathfrak{a}(f)$ and $\mathfrak{c}(f)$ is determined with the threshold $\delta$ and the hyperparameter $\lambda$.
Therefore, Theorem~\ref{thm:decomposition} provides deeper insights into understanding the mechanism of contrastive learning.

\subsection{A New Upper Bound of the Classification Error}
\label{subsec:a surrogate upper bound of NN classification error}

Next, we show how minimization of the kernel contrastive loss guarantees good performance in the downstream classification task, according to our formulation of similarity. To this end, we prove that the properties of contrastive learning shown in Theorem~\ref{thm:decomposition} yield the linearly well-separable representations in the RKHS.
First, we quantify the linear separability as follows: 
following~\citet{haochen2021provable,saunshi2022understanding}, under Assumption~\ref{assumption:mixture of clusters}, for a model $f\in\mathcal{F}$, a linear weight $W:\mathcal{H}_{k}\to\mathbb{R}^{K}$, and a bias $\beta\in\mathbb{R}^{K}$, we define the downstream classification error as,
\begin{align*}
    L_{\textup{Err}}(f, W, \beta;y):=P_{\mathbb{X}}\left(g_{f,W,\beta}(x)\neq y(x)\right),
\end{align*}
where $\textstyle g_{f,W,\beta}(x):=\textup{arg~max}_{i\in[K]}\{\langle W_{i},h(f(x))\rangle_{\mathcal{H}_{k}}+\beta_{i}\}$ for $W_{i}\in\mathcal{H}_{k}$ and $\beta_{i}\in\mathbb{R}$.
Note that we let $\textup{arg~max}$ and $\textup{arg~min}$ break tie arbitrary as well as~\citet{haochen2021provable}.
Note that in our definition, after augmented data $x\in\mathbb{X}$ is encoded to $f(x)\in\mathbb{S}^{d-1}$, $f(x)$ is further mapped to $h(f(x))$ in the RKHS, and then linear classification is performed using $W$ and $\beta$.

To derive a generalization guarantee for KCL, we focus on the 1-Nearest Neighbor (NN) classifier in the RKHS $\mathcal{H}_{k}$ as a technical tool, which is a generalization of the 1-NN classifiers utilized in \citet{robinson2020contrastive,huang2021towards}.

\begin{definition}[1-NN classifier in $\mathcal{H}_{k}$]
\label{def:nearest neighbor classifier}
Suppose $\mathbb{C}_{1},\cdots,\mathbb{C}_{K}$ are subsets of $\mathbb{X}$.
For a model $f:\mathbb{X}\to\mathbb{R}^{d}$, the 1-NN classifier $g_{\textup{1-NN}}:\mathbb{X}\to[K]$ associated with the RKHS $\mathcal{H}_{k}$ is defined as
\begin{align*}
    g_{\textup{1-NN}}(x):=\textup{arg~min}_{i\in[K]}\|h(f(x))-\mu_{\mathbb{C}_{i}}(f)\|_{\mathcal{H}_{k}}.
\end{align*}
\end{definition}

\citet{huang2021towards} show that the 1-NN classifier they consider can be regarded as a mean classifier~\citep{pmlr-v97-saunshi19a,wang2022chaos} (see Appendix~E in~\citet{huang2021towards}).
This fact can also apply to our setup: indeed, under Assumption~\ref{assumption:kernel}, $g_{\textup{1-NN}}$ is equal to
\begin{align*}
g_{f,W_{\mu},\beta_{\mu}}(x)=\textup{arg~max}_{i\in[K]}\left\{\langle W_{\mu,i},h(f(x))\rangle_{\mathcal{H}_{k}}-\beta_{\mu,i}\right\},
\end{align*}
where $W_{\mu}:\mathcal{H}_{k}\to\mathbb{R}^{K}$ is defined as $W_{\mu}(\phi)_{i}:=\langle W_{\mu,i},\phi\rangle_{\mathcal{H}_{k}}=\langle \mu_{\mathbb{C}_{i}}(f),\phi\rangle_{\mathcal{H}_{k}}$ for each coordinate $i\in[K]$, and $\beta_{\mu,i}:=(\|\mu_{\mathbb{C}_{i}}(f)\|_{\mathcal{H}_{k}}^{2}+\psi(1))/2$ for $i\in[K]$.

Before presenting the result, we need the following notion:
\begin{definition}[Meaningful encoder]
\label{def:meaningful encoders}
An encoder $f\in\mathcal{F}$ is said to be meaningful if $\min_{i\neq j}\|\mu_{i}(f)-\mu_{j}(f)\|_{\mathcal{H}_{k}}^{2}>0$ holds.
\end{definition}
Note that a meaningful encoder $f\in\mathcal{F}$ avoids the \textit{complete collapse} of feature vectors~\citep{hua2021feature,jing2022understanding}, where many works on self-supervised representation learning~\citep{chen2020simple,grill2020bootstrap,chen2021exploring,haochen2021provable,li2021selfsupervised} introduce various architectures and algorithms to prevent it.
Now, the theoretical guarantee is presented:

\begin{theorem}
\label{thm:guarantee}
Suppose that Assumption~\ref{assumption:kernel} and \ref{assumption:mixture of clusters} hold.
Take $\delta\in\mathbb{R}$, $K\in\mathbb{N}$, $\mathbb{M}_{1},\cdots,\mathbb{M}_{K}$, and $y$ such that the conditions \textbf{\textup{(A)}}, \textbf{\textup{(B)}}, and \textbf{\textup{(C)}} in Assumption~\ref{assumption:mixture of clusters} are satisfied.
Then, for each meaningful encoder $f\in\mathcal{F}$, we have
\begin{align*}
    L_{\textup{Err}}(f, W_{\mu},\beta_{\mu};y)\leq 
    \frac{8(K-1)}{\Delta_{\textup{min}}(f)\cdot\min_{i\in[K]}P_{\mathbb{X}}(\mathbb{M}_{i})}\mathfrak{a}(f)
\end{align*}
where $\Delta_{\textup{min}}(f)=\min_{i\neq j}\|\mu_{i}(f)-\mu_{j}(f)\|_{\mathcal{H}_{k}}^{2}$.
\end{theorem}

The proof of Theorem~\ref{thm:guarantee} can be found in Appendix~\ref{appsubsec:proof of theorem nn classification error}. 
The upper bound in Theorem~\ref{thm:guarantee} becomes smaller if the representations of any two points $x,x'$ belonging to $\mathbb{M}_{i}$ are closer for each $i\in[K]$ and the closest centers $\mu_{i}(f)$ and $\mu_{j}(f)$ of different subsets $\mathbb{M}_{i}$ and $\mathbb{M}_{j}$ become distant from each other.
Since $\|\mu_{i}(f)-\mu_{j}(f)\|_{\mathcal{H}_{k}}^{2}=-2\langle \mu_{i}(f),\mu_{j}(f)\rangle_{\mathcal{H}_{k}}+\|\mu_{i}(f)\|_{\mathcal{H}_{k}}^{2}+\|\mu_{j}(f)\|_{\mathcal{H}_{k}}^{2}$, Theorem~\ref{thm:decomposition} and \ref{thm:guarantee} indicate that during the optimization for the kernel contrastive loss, the quantities $\mathfrak{a}(f)$ and $\mathfrak{c}(f)$ can contribute to making the learned representations linearly well-separable.
Thus, our theory is consistent with the empirical success of contrastive learning shown by a line of research~\citep{chen2020simple,chen2020improved,haochen2021provable,dwibedi2021little}.

\subsubsection{Comparison to Related Work}
We discuss Theorem~\ref{thm:guarantee}. 
1) Several works~\citep{robinson2020contrastive,huang2021towards} also show that the classification loss or error is upper bounded by the quantity related to the alignment of feature vectors within each cluster.
However, their results do not address the following conjecture:
\textit{Does the distance between the centers of each cluster consisting of feature vectors affect the linear separability?}
Theorem~\ref{thm:guarantee} indicates that the answer is yes via the quantity $\Delta_{\textup{min}}(f)$.
Note that Theorem~3.2 of \citet{zhao2023arcl} implies a similar answer, but their result is for the squared loss and requires several strong assumptions on the encoder functions. Meanwhile, our Theorem~\ref{thm:guarantee} for the classification error requires the meaningfulness of encoder functions (Definition~\ref{def:meaningful encoders}), which is more practical than those of~\citet{zhao2023arcl}.
2) Furthermore, our result is different from previous works~\citep{robinson2020contrastive,huang2021towards,zhao2023arcl} in the problem setup. Indeed, \citet{robinson2020contrastive} follow the setup of~\citet{pmlr-v97-saunshi19a}, and~\citet{huang2021towards,zhao2023arcl} formulate their setup by imposing the $(\sigma,\delta)$\textit{-augmentation} property to given latent class subsets.
Meanwhile, our formulation is mainly based on the statistical similarity~\eqref{def:similarity function}.
Furthermore, we note that our Theorem~\ref{thm:guarantee} can be extended to the case that $K\in\mathbb{N}$, $\mathbb{M}_{1},\cdots,\mathbb{M}_{K}\subset\mathbb{X}$, and $y$ are taken to satisfy the conditions \textbf{(A)} and \textbf{(C)} in Assumption~\ref{assumption:mixture of clusters} (see Theorem~\ref{thm:nn classification error} in Appendix~\ref{appsubsec:proof of theorem nn classification error}), implying that our result can apply to other problem setups of contrastive learning.
Due to space limitations, we present more detailed explanations in Appendix~\ref{appsubsec:comparison to robinson} and \ref{appsubsec:comparison to huang}.

\subsection{A Generalization Error Bound for KCL}
\label{subsec:rethinking generalization of contrastive learning}

Since in practice we minimize the empirical kernel contrastive loss, we derive a generalization error bound for KCL.
The empirical loss is defined as follows: 
denote $dP_{+}(x,x')=w(x,x')d\nu_{\mathbb{X}}^{\otimes 2}(x,x')$.
Let $(X_{1},X_{1}'),\cdots,(X_{n},X_{n}')$ be pairs of random variables drawn independently from $P_{+}$, where $X_{i}$ and $X_{j}'$ are assumed to be independent for each pair of distinct indices $i,j\in[n]$.
Following the standard setup that a pair of two augmented samples is obtained by randomly transforming the same raw sample, which is considered in many empirical works~\citep{chen2020simple,chen2020improved,dwibedi2021little}, for each $i\in[n]$, we consider the case that $X_{i},X_{i}'$ are not necessarily independent.
The empirical kernel contrastive loss is defined as,
\begin{align}
\label{def:empirical kernel contrastive loss}
    \widehat{L}_{\textup{KCL}}(f;\lambda)=-\frac{1}{n}\sum_{i=1}^{n}k(f(X_{i}),f(X_{i}'))+\frac{\lambda}{n(n-1)}\sum_{i\neq j}k(f(X_{i}),f(X_{j}')).
\end{align}
In the statement below, denote $\mathcal{Q}=\{f(\cdot)^{\top}f(\cdot):\mathbb{X}\times\mathbb{X}\to\mathbb{R}\;|\;f\in\mathcal{F}\}$.
Define the Rademacher complexity~\citep{mohri2018foundations} as, $\mathfrak{R}_{n}^{+}(\mathcal{Q}):=\mathbb{E}_{P_{+},\sigma_{1:n}}[\sup_{q\in\mathcal{Q}}n^{-1}\sum_{i=1}^{n}\sigma_{i}q(X_{i},X_{i}')]$, where $\sigma_{1},\cdots,\sigma_{n}$ are independent random variables taking $\pm 1$ with probability one half for each.
We also define the Rademacher compexltiy $\mathfrak{R}_{n/2}^{-}(\mathcal{Q};s^{*})$ with the optimal choice $s^{*}$ from the symmetric group $S_{n}$ of degree $n$: 
\begin{align*}
    \mathfrak{R}_{n/2}^{-}(\mathcal{Q};s^{*}):=\max_{s\in S_{n}}\mathbb{E}_{\substack{X,X'\\\sigma_{1:(n/2)}}}\left[\sup_{q\in\mathcal{Q}}\frac{2}{n}\sum_{i=1}^{n/2}\sigma_{i}q(X_{s(2i-1)},X_{s(2i)}')\right].
\end{align*}
Note that $\mathfrak{R}_{n/2}^{-}(\mathcal{Q};s^{*})$ is related to the \textit{average of "sums-of-i.i.d." blocks} technique for $U$-statistics explained in~\citet{clemencon2008ranking}; for more detail, see also Remark~\ref{remark:relation to clemencon} in Appendix~\ref{appsubsec:proof of theorem dependent ulln}.
The generalization error bound for KCL is presented below:
\begin{theorem}
\label{lem:ulln}
Suppose that Assumption~\ref{assumption:kernel} holds, and $n$ is even.
Furthermore, suppose that the minimizer $\widehat{f}\in\mathcal{F}$ of $\widehat{L}_{\textup{KCL}}(f;\lambda)$ exists. Then, with probability at least $1-2\varepsilon$ where $\varepsilon>0$, we have
\begin{align*}
    L_{\textup{KCL}}(\widehat{f};\lambda)\leq L_{\textup{KCL}}(f;\lambda)+2\cdot\textup{Gen}(n,\lambda,\varepsilon),
\end{align*}
where $\mathfrak{R}_{n}(\mathcal{Q}):=\mathfrak{R}_{n}^{+}(\mathcal{Q})+\lambda\mathfrak{R}_{n/2}^{-}(\mathcal{Q};s^{*})$, and
\begin{align*}
    \textup{Gen}(n,\lambda,\varepsilon)=O\left(\mathfrak{R}_{n}(\mathcal{Q})+(1+\lambda)\sqrt{\frac{\log(2/\varepsilon)}{n}}\right).
\end{align*}
\end{theorem}

\begin{remark}
In Appendix~\ref{appsubsec:proof of lemma chaining}, we show that under some conditions, $\textup{Gen}(n,\lambda,\varepsilon)\downarrow 0$ holds as $n\to\infty$.
\end{remark}

The proof of Theorem~\ref{lem:ulln} can be found in Appendix~\ref{appsubsec:proofs of lemma and theorem in section 4.3}.
Since $X_{i},X_{i}'$ are not necessarily independent for each $i\in[n]$, the standard techniques (e.g., Theorem~3.3 in~\citet{mohri2018foundations}) are not applicable in the proof.
We instead utilize the results by~\citet{zhang2019mcdiarmid} to overcome this difficulty, which is different from  the previous bounds for contrastive learning~\citep{pmlr-v97-saunshi19a,nozawa2020pac,haochen2021provable,ash2022investigating,zhang2022fmutual,zou2023generalization,lei2023generalization,wang2022spectral} (see Appendix~\ref{appsubsec:discussion of generalization bounds} for more detail).
Here, if $\mathfrak{R}_{n}(\mathcal{Q})\downarrow 0$ as $n\to\infty$, then by using our result, we can prove the consistency of the empirical contrastive loss to the population one for each $f\in\mathcal{F}$.

\subsection{Application of the Theoretical Results: A New Surrogate Bound}
\label{subsec:main result section 5.4}

Recent works~\citep{pmlr-v97-saunshi19a,nozawa2020pac,tosh2021contrastive,nozawa2021understanding,haochen2021provable,wang2022chaos,ash2022investigating,bao2022on,awasthi2022do,saunshi2022understanding,zou2023generalization,dufumier2022rethinking} show that some contrastive learning objectives can be regarded as surrogate losses of the supervised loss or error in downstream tasks.
Here, \citet{pmlr-v97-saunshi19a} show that, a contrastive loss $L_{\textup{CL}}$ \textit{surrogates} a supervised loss or error $L_{\textup{Sup}}$: for every $f\in\mathcal{F}$, $L_{\textup{sup}}(W\circ f)\lesssim L_{\textup{CL}}(f)+\alpha$ holds for some $\alpha\in\mathbb{R}$ and matrix $W$.
This type of inequality is also called \textit{surrogate bound}~\citep{bao2022on}.
\citet{pmlr-v97-saunshi19a} show that the inequality guarantees that $L_{\textup{sup}}(W^{*}\circ\widehat{f})\lesssim L_{\textup{CL}}(f)+\alpha$ holds with high probability, where $\widehat{f}$ is a minimizer of the empirical loss for $L_{\textup{CL}}$, $W^{*}$ is the optimal weight, and $\alpha$ is some term.
Motivated by these works, we show a surrogate bound for KCL.
\begin{theorem}
\label{cor:surrogate bound}
Suppose that Assumption~\ref{assumption:kernel} and \ref{assumption:mixture of clusters} hold, $n$ is even, and there exists a minimizer $\widehat{f}$ of $\widehat{L}_{\textup{KCL}}(f;\lambda)$ such that $\widehat{f}$ is meaningful.
Take $\delta\in\mathbb{R}$, $K\in\mathbb{N}$, $\mathbb{M}_{1},\cdots,\mathbb{M}_{K}$, and $y$ such that the conditions \textup{\textbf{(A)}}, \textup{\textbf{(B)}}, and \textup{\textbf{(C)}} in Assumption~\ref{assumption:mixture of clusters} are satisfied.
Then, for any $f\in\mathcal{F}$ and $\varepsilon>0$, with probability at least $1-2\varepsilon$,
\begin{align*}
    L_{\textup{Err}}(\widehat{f}, W^{*},\beta^{*};y)\lesssim L_{\textup{KCL}}(f;\lambda)+(1-\frac{\delta}{2})\mathfrak{a}(\widehat{f})-\lambda\mathfrak{c}(\widehat{f})+R(\lambda)+2\textup{Gen}(n,\lambda,\varepsilon),
\end{align*}
where $L_{\textup{Err}}(\widehat{f},W^{*},\beta^{*};y)=\inf_{W,\beta}L_{\textup{Err}}(\widehat{f},W,\beta;y)$, and $\lesssim$ omits the coefficient
$8(K-1)/(\Delta_{\textup{min}}(\widehat{f})\cdot\textup{min}_{i\in[K]}P_{\mathbb{X}}(\mathbb{M}_{i}))$.
\end{theorem}

The proof of Theorem~\ref{cor:surrogate bound} can be found in Appendix~\ref{appsec:surrogate bound}.
This theorem indicates that minimization of the kernel contrastive loss in $\mathcal{F}$ can reduce the infimum of the classification error with high probability.
Note that since larger $\lambda$ can make $\delta$ smaller due to the relation in condition \textbf{(B)} of Assumption~\ref{assumption:mixture of clusters}, larger $\lambda$ may result in enlarging $(1-\delta/2)\mathfrak{a}(\widehat{f})-\lambda\mathfrak{c}(\widehat{f})$ and loosening the upper bound if $\mathfrak{a}(\widehat{f})>0$ and $\mathfrak{c}(\widehat{f})<0$.
We empirically find that the KCL framework with larger $\lambda$ degrades its performance in the downstream classification task; see Appendix~\ref{appsubsec:larger lambda result in bad result}.

\subsubsection{Comparison to Related Work}
Several works also establish the surrogate bounds for some contrastive learning objectives~\citep{pmlr-v97-saunshi19a,nozawa2020pac,tosh2021contrastive,nozawa2021understanding,haochen2021provable,wang2022chaos,ash2022investigating,bao2022on,awasthi2022do,saunshi2022understanding,zou2023generalization,dufumier2022rethinking}.
The main differences between the previous works and Theorem~\ref{cor:surrogate bound} are summarized in three points:
1) Theorem~\ref{cor:surrogate bound} indicates that the kernel contrastive loss is a surrogate loss of the classification error, while the previous works deal with other contrastive learning objectives.
2) Recent works~\citep{wang2022chaos,bao2022on} prove that the InfoNCE loss is a surrogate loss of the cross-entropy loss.
However, since the theory of classification calibration losses (see e.g., \citet{zhang2004statistical}) indicates that the relation between the classification loss and the cross-entropy loss is complicated under the multi-class setting, the relation between the InfoNCE loss and the classification error is non-trivial from the previous results. 
On the other hand, combining Theorem~\ref{cor:surrogate bound} and \eqref{eq:lin vs nce}, we can show that the InfoNCE loss is also a surrogate loss of the classification error.
Note that Theorem~\ref{cor:surrogate bound} can apply to other contrastive learning objectives.
3) The bound in Theorem~\ref{cor:surrogate bound} is established by introducing the formulation presented in Section~\ref{subsec:assumptions and definitions}.
Especially our bound includes the geometric quantity $\delta$ and hyperparameter $\lambda$.

\section{Conclusion and Discussion}
\label{sec:discussions}

In this paper, we studied the characterization of the structure of the representations learned by contrastive learning.
By employing Kernel Contrastive Learning (KCL) as a unified framework, we showed that the formulation based on statistical similarity characterizes the clusters of learned representations and guarantees that the kernel contrastive loss minimization can yield good performance in the downstream classification task.
As a limitation of this paper, 
we point out that in practice, it is challenging to compute the true $\textup{sim}(x,x';\lambda)$ and $\delta$ for datasets.
However, we believe that our theory promotes future theoretical and empirical research to investigate the practical success of contrastive learning via the sets of augmented data defined by $\textup{sim}(x,x';\lambda)$ and $\delta$.
Note that as recent works~\citep{TsaiNeural2020,tsai2021rpc} tackle the estimation of the point-wise dependency by using neural networks, the estimation problem is an important future work.
As a future work, it is worth studying how the selection of kernels affects the quality of representations via our theory.
The investigation of transfer learning perspectives of KCL is also an interesting future work, as recent works~\citep{shen2022connect,haochen2022beyond,zhao2023arcl} also address the problem for some contrastive learning frameworks.

\subsubsection*{Ethical Statement}
Since this paper mainly studies theoretical analysis of contrastive learning, it will not be thought that there is a direct negative social impact.
However, revealing detailed properties of contrastive learning could promote an opportunity to misuse the knowledge.
We point out that such wrong usage is not straightforward with the proposed method, as the application is not discussed much in the paper.

\bibliographystyle{apalike}
% \bibliography{main}

\clearpage

\appendix

\section{Proof in Section~\ref{subsec:reproducing kernel}}
\label{appsubsec:proof of proposition kernel mean embedding}

First we prove Proposition~\ref{prop:kernel mean embedding}.

\begin{proof}[Proof of Proposition~\ref{prop:kernel mean embedding}]
The proof of the claim closely follows the proof of Lemma 3.1 in~\citet{muandet2017kernel} (see also~\citet{smola2007hilbert}), which shows that if $\mathbb{E}_{P}[\sqrt{k(x,x)}]<+\infty$ where $x\sim P$, then $\mathbb{E}_{P}[k(\cdot,x)]\in\mathcal{H}_{k}$.
For the sake of completeness, we provide the proof of Proposition~\ref{prop:kernel mean embedding} by modifying the proof of~\citet{muandet2017kernel} slightly.

Let $\mathbb{M}$ be a measurable set in $\mathbb{X}$.
Define $\mu_{\mathbb{M}}(f):=\mathbb{E}[h(f(x))|\mathbb{M}]$.
Our goal is to show that $\mu_{\mathbb{M}}(f)\in\mathcal{H}_{k}$ holds.
To this end, for $\phi\in\mathcal{H}_{k}$, we compute
\begin{align}
    \left|\mathbb{E}\left[\phi(f(x))|\mathbb{M}\right]\right|
    &=\left|\mathbb{E}\left[\langle \phi,k(\cdot,f(x))\rangle_{\mathcal{H}_{k}} |\mathbb{M}\right]\right|\nonumber\\
    &\leq \mathbb{E}\left[\left|\langle \phi,k(\cdot,f(x))\rangle_{\mathcal{H}_{k}}\right| |\mathbb{M}\right]\nonumber\\
    &\leq \mathbb{E}\left[\|\phi\|_{\mathcal{H}_{k}}\|k(\cdot,f(x))\|_{\mathcal{H}_{k}}|\mathbb{M}\right]\tag{Cauchy-Schwarz ineq.}\\
    &=\|\phi\|_{\mathcal{H}_{k}}\mathbb{E}\left[\sqrt{k(f(x),f(x))}|\mathbb{M}\right].\nonumber
\end{align}
Since $\sup_{z,z'\in\mathbb{S}^{d-1}}k(z,z')<\infty$ holds, we have $\mathbb{E}[\sqrt{k(f(x),f(x))}|\mathbb{M}]<+\infty$.
Hence, the map $\phi\mapsto \mathbb{E}[\phi(f(x))|\mathbb{M}]$ is a bounded linear functional on $\mathcal{H}_{k}$, and thus from Riesz's representation theorem, there exists some $\xi\in\mathcal{H}_{k}$ such that $\mathbb{E}[\phi(f(x))|\mathbb{M}]=\langle \xi,\phi\rangle_{\mathcal{H}_{k}}$.
However, let $\phi=k(\cdot,z)$, then $\xi(z)=\langle \xi,k(\cdot,z)\rangle_{\mathcal{H}_{k}}=\mathbb{E}[k(f(x),z)|\mathbb{M}]$.
This implies $\xi=\mathbb{E}[k(f(x),\cdot)|\mathbb{M}]\in\mathcal{H}_{k}$.
Since $k$ is symmetric, we have $\mu_{\mathbb{M}}(f)=\mathbb{E}[k(\cdot,f(x))|\mathbb{M}]\in\mathcal{H}_{k}$.
\end{proof}

\section{Proofs in Section~\ref{subsec:explaining what's going on contrastive learning}}
\label{appsec:proofs in section 4.1}

\subsection{Useful Lemmas for the Proof of Theorem~\ref{thm:decomposition}}
\label{appsubsec:useful lemmas for theorem 4.1}

Before showing Theorem~\ref{thm:decomposition}, we give several basic and useful lemmas that are used in the proof of the theorem. 
Since the definition of $\mu_{\mathbb{M}}(f)$, where $\mathbb{M}$ is a measurable subset of $\mathbb{X}$ and $f\in\mathcal{F}$, is slightly different from the kernel mean embedding of the usual form~\citep{berlinet2004reproducing,muandet2017kernel} due to the existence of the encoder function $f$, we provide the proof for each lemma for the sake of completeness.

\begin{lemma}
\label{lemma:reproducing property of kme}
Let $\{e_{j}\}$ be an orthonormal basis of $\mathcal{H}_{k}$, and let $\mathbb{M}$ be a measurable set.
Let $f\in\mathcal{F}$.
Then, the following identity holds for each $j$:
\begin{align*}
    \int_{\mathbb{M}}\langle h(f(x)),e_{j}\rangle_{\mathcal{H}_{k}}P_{\mathbb{X}}(dx|\mathbb{M})=\langle \mu_{\mathbb{M}}(f),e_{j}\rangle_{\mathcal{H}_{k}}.
\end{align*}
\end{lemma}

\begin{proof}
We calculate,
\begin{align*}
    \langle \mu_{\mathbb{M}}(f),e_{j}\rangle_{\mathcal{H}_{k}}
    &=\left\langle \int_{\mathbb{M}}h(f(x))P_{\mathbb{X}}(dx|\mathbb{M}),e_{j}\right\rangle_{\mathcal{H}_{k}}\\
    &=\left\langle \int_{\mathbb{M}}\sum_{j'}\langle h(f(x)),e_{j'}\rangle_{\mathcal{H}_{k}}e_{j'} P_{\mathbb{X}}(dx|\mathbb{M}),e_{j}\right\rangle_{\mathcal{H}_{k}}\\
    &=\left\langle \sum_{j'}e_{j'}\int_{\mathbb{M}}\langle h(f(x)),e_{j'}\rangle_{\mathcal{H}_{k}}P_{\mathbb{X}}(dx|\mathbb{M}),e_{j}\right\rangle_{\mathcal{H}_{k}}\\
    &=\int_{\mathbb{M}}\langle h(f(x)),e_{j}\rangle_{\mathcal{H}_{k}}P_{\mathbb{X}}(dx|\mathbb{M}),
\end{align*}
where in the third line, we use the Dominated Convergence Theorem for the Bochner integral (e.g., see Theorem~1.1.8 in~\citet{arendt2011vector}).
Hence, we obtain the claim.
\end{proof}

\begin{lemma}
\label{lem:inner product}
Let $\mathbb{M},\mathbb{M}'$ be measurable subsets of $\mathbb{X}$. 
Let $f\in\mathcal{F}$.
Then, we have
\begin{align*}
    \int_{\mathbb{M}}\int_{\mathbb{M}'}\langle h(f(x)),h(f(x'))\rangle_{\mathcal{H}_{k}}P_{\mathbb{X}}(dx|\mathbb{M})P_{\mathbb{X}}(dx'|\mathbb{M}')=\langle \mu_{\mathbb{M}}(f),\mu_{\mathbb{M}'}(f)\rangle_{\mathcal{H}_{k}}.
\end{align*}
\end{lemma}

\begin{proof}
Let $\{e_{j}\}$ be an orthonormal basis of $\mathcal{H}_{k}$.
Then we have,
\begin{align}
    &\;\;\;\;\int_{\mathbb{M}}\int_{\mathbb{M}'}\langle h(f(x)),h(f(x'))\rangle_{\mathcal{H}_{k}}P_{\mathbb{X}}(dx|\mathbb{M})P_{\mathbb{X}}(dx'|\mathbb{M}')\nonumber\\
    &=\int_{\mathbb{M}}\int_{\mathbb{M}'}\left\langle \sum_{j}\langle h(f(x)),e_{j}\rangle_{\mathcal{H}_{k}}e_{j},\sum_{j}\langle h(f(x')),e_{j}\rangle_{\mathcal{H}_{k}}e_{j}\right\rangle_{\mathcal{H}_{k}}P_{\mathbb{X}}(dx|\mathbb{M})P_{\mathbb{X}}(dx'|\mathbb{M}')\nonumber\\
    &=\int_{\mathbb{M}}\int_{\mathbb{M}'} \sum_{j}\langle h(f(x)),e_{j}\rangle_{\mathcal{H}_{k}}\langle h(f(x')),e_{j}\rangle_{\mathcal{H}_{k}}P_{\mathbb{X}}(dx|\mathbb{M})P_{\mathbb{X}}(dx'|\mathbb{M}')\nonumber\\
    \label{lem inner eq1}
    &=\sum_{j}\int_{\mathbb{M}}\int_{\mathbb{M}'} \langle h(f(x)),e_{j}\rangle_{\mathcal{H}_{k}}\langle h(f(x')),e_{j}\rangle_{\mathcal{H}_{k}}P_{\mathbb{X}}(dx|\mathbb{M})P_{\mathbb{X}}(dx'|\mathbb{M}')\\
    &=\sum_{j}\left(\int_{\mathbb{M}} \langle h(f(x)),e_{j}\rangle_{\mathcal{H}_{k}}P_{\mathbb{X}}(dx|\mathbb{M})\right)\left(\int_{\mathbb{M}'}\langle h(f(x')),e_{j}\rangle_{\mathcal{H}_{k}}P_{\mathbb{X}}(dx'|\mathbb{M}')\right)\nonumber\\
    &=\sum_{j}\langle\mu_{\mathbb{M}}(f),e_{j}\rangle_{\mathcal{H}_{k}}\langle \mu_{\mathbb{M}'}(f),e_{j}\rangle_{\mathcal{H}_{k}}\tag{Lemma~\ref{lemma:reproducing property of kme}}\\
    &=\left\langle\sum_{j}\langle \mu_{\mathbb{M}}(f),e_{j}\rangle_{\mathcal{H}_{k}}e_{j},\sum_{j}\langle \mu_{\mathbb{M}'}(f),e_{j}\rangle_{\mathcal{H}_{k}}e_{j}\right\rangle_{\mathcal{H}_{k}}\nonumber\\
    &=\langle \mu_{\mathbb{M}}(f),\mu_{\mathbb{M}'}(f)\rangle_{\mathcal{H}_{k}},\nonumber
\end{align}
where in \eqref{lem inner eq1} we use the Dominated Convergence Theorem.
Hence we obtain the claim.
\end{proof}

\subsection{Proof of Theorem~\ref{thm:decomposition}}
\label{appsubsec:proof of theorem 4.1}

The following notation is used in the proof of Theorem~\ref{thm:decomposition}.

\begin{definition}
\label{def:R(lambda)}
Denote $M_{k}=\sup_{z,z'\in\mathbb{S}^{d-1}}\|k(\cdot,z)-k(\cdot,z')\|_{\mathcal{H}_{k}}^{2}$.
We define,
\begin{align*}
    R(\lambda):=\frac{M_{k}}{2}\sum_{i\neq j}P_{+}((\mathbb{M}_{i}\cap\mathbb{M}_{j})\times(\mathbb{M}_{i}\cap\mathbb{M}_{j}))+\lambda\psi(1)\sum_{i=1}^{K}P_{\mathbb{X}}(\mathbb{M}_{i})\left(1-P_{\mathbb{X}}(\mathbb{M}_{i})\right)+(1-\lambda)\psi(1),
\end{align*}
where $dP_{+}(x,x')=w(x,x')d\nu_{\mathbb{X}}^{\otimes 2}(x,x')$.
\end{definition}

Note that under Assumption~\ref{assumption:kernel}, $k(z):=k(z,z)=\psi(z^{\top}z)=\psi(1)$ is a constant function on $\mathbb{S}^{d-1}$.
We are now ready to present the proof of Theorem~\ref{thm:decomposition}.

\begin{proof}[Proof of Theorem~\ref{thm:decomposition}]
It is convenient to analyze the following form instead of the kernel contrastive loss:
\begin{align}
\label{eq:thm 2 eq1}
    \widetilde{L}_{\textup{KCL}}(f;\lambda):=\underbrace{\mathbb{E}_{x,x^{+}}\left[\|h(f(x))-h(f(x^{+}))\|_{\mathcal{H}_{k}}^{2}\right]}_{\textup{the positive term}}-\lambda\underbrace{\mathbb{E}_{x,x^{-}}\left[\|h(f(x))-h(f(x^{-}))\|_{\mathcal{H}_{k}}^{2}\right]}_{\textup{the negative term}}.
\end{align}
Note that, $\widetilde{L}_{\textup{KCL}}(f;\lambda)=2(1-\lambda)\psi(1)+2L_{\textup{KCL}}(f;\lambda)$ holds since $f(x)\in\mathbb{S}^{d-1}$ for all $x\in\mathbb{X}$.
For the positive term of $\widetilde{L}_{\textup{KCL}}(f)$, we can evaluate that,

\begin{align}
&\;\;\;\;\mathbb{E}_{x,x^{+}}\left[\|h(f(x))-h(f(x^{+}))\|_{\mathcal{H}_{k}}^{2}\right]\nonumber\\
\label{eq:theorem 5.1 exeq1}
&\geq \int_{\bigcup_{i=1}^{K}\mathbb{M}_{i}\times\mathbb{M}_{i}}\|h(f(x))-h(f(x'))\|_{\mathcal{H}_{k}}^{2}w(x,x')d\nu_{\mathbb{X}}(x)d\nu_{\mathbb{X}}(x')\\
&\geq \sum_{i=1}^{K}\int_{\mathbb{M}_{i}\times\mathbb{M}_{i}}\|h(f(x))-h(f(x'))\|_{\mathcal{H}_{k}}^{2}w(x,x')d\nu_{\mathbb{X}}(x)d\nu_{\mathbb{X}}(x')\nonumber\\
\label{eq:theorem 5.1 exeq2}
&\;\;\;\;\;\;\;\;
-\sum_{j\neq i}\int_{(\mathbb{M}_{i}\cap\mathbb{M}_{j})\times(\mathbb{M}_{i}\cap\mathbb{M}_{j})}\|h(f(x))-h(f(x'))\|_{\mathcal{H}_{k}}^{2}w(x,x')d\nu_{\mathbb{X}}(x)d\nu_{\mathbb{X}}(x')\\
\label{eq:thm 2 eq2}
&\geq \sum_{i=1}^{K}\left(\int_{\mathbb{M}_{i}\times\mathbb{M}_{i}}\|h(f(x))-h(f(x'))\|_{\mathcal{H}_{k}}^{2}w(x,x')d\nu_{\mathbb{X}}(x)d\nu_{\mathbb{X}}(x')
-M_{k}\sum_{j\neq i}P_{+}((\mathbb{M}_{i}\cap\mathbb{M}_{j})\times(\mathbb{M}_{i}\cap\mathbb{M}_{j}))\right)
\end{align}
where in the second inequality we use the fact that 
\begin{align*}
Q(\bigcup_{i=1}^{K}\mathbb{M}_{i}\times\mathbb{M}_{i})\geq \sum_{i=1}^{K}Q(\mathbb{M}_{i}\times\mathbb{M}_{i})-\sum_{i\neq j}Q((\mathbb{M}_{i}\times\mathbb{M}_{i})\cap(\mathbb{M}_{j}\times\mathbb{M}_{j})),
\end{align*}
for any probability measure $Q$ in $\mathbb{X}\times\mathbb{X}$, and in the last inequality we use the definition $M_{k}=\sup_{z,z'\in\mathbb{S}^{d-1}}\|k(\cdot,z)-k(\cdot,z')\|_{\mathcal{H}_{k}}^{2}$.
The first term of the above lower bound can be bounded as
\begin{align}
    &\;\;\;\;\sum_{i=1}^{K}\int_{\mathbb{M}_{i}\times\mathbb{M}_{i}}\|h(f(x))-h(f(x'))\|_{\mathcal{H}_{k}}^{2}w(x,x')d\nu_{\mathbb{X}}(x)d\nu_{\mathbb{X}}(x')\nonumber\\
    \label{eq:thm 2 eq3}
    &\geq \sum_{i=1}^{K}\int_{\mathbb{M}_{i}\times\mathbb{M}_{i}}\|h(f(x))-h(f(x'))\|_{\mathcal{H}_{k}}^{2}\cdot(\lambda+\delta)w(x)w(x')d\nu_{\mathbb{X}}(x)d\nu_{\mathbb{X}}(x'),
\end{align}
where we utilize the definition of $\mathbb{M}_{i}$ for each $i\in\{1,\cdots,K\}$;
recall that due to the condition \textbf{(B)} in Assumption~\ref{assumption:mixture of clusters}, for every $x,x'\in\mathbb{M}_{i}$ we have $\textup{sim}(x,x';\lambda)\geq \delta$.

On the other hand, for the negative term we can compute as follows:
\begin{align}
    &\;\;\;\;-\mathbb{E}_{x,x^{-}}\left[\|h(f(x))-h(f(x^{-}))\|_{\mathcal{H}_{k}}^{2}\right]\nonumber\\
    &=-\int_{\bigcup_{i=1}^{K}\mathbb{M}_{i}\times\mathbb{M}_{i}}\|h(f(x))-h(f(x'))\|_{\mathcal{H}_{k}}^{2}w(x)w(x')d\nu_{\mathbb{X}}(x)d\nu_{\mathbb{X}}(x')\nonumber\\
    &\;\;\;\;\;\;\;\;-\int_{\mathbb{X}\times\mathbb{X}\setminus\left(\bigcup_{i=1}^{K}\mathbb{M}_{i}\times\mathbb{M}_{i}\right)}\|h(f(x))-h(f(x'))\|_{\mathcal{H}_{k}}^{2}w(x)w(x')d\nu_{\mathbb{X}}(x)d\nu_{\mathbb{X}}(x')\nonumber \\
    &\geq -\sum_{i=1}^{K}\int_{\mathbb{M}_{i}\times\mathbb{M}_{i}}\|h(f(x))-h(f(x'))\|_{\mathcal{H}_{k}}^{2}w(x)w(x')d\nu_{\mathbb{X}}(x)d\nu_{\mathbb{X}}(x')\nonumber\\
    \label{eq:theorem 5.1 exeq3}
    &\;\;\;\;\;\;\;\;-\int_{\mathbb{X}\times\mathbb{X}\setminus\left(\bigcup_{i=1}^{K}\mathbb{M}_{i}\times\mathbb{M}_{i}\right)}\|h(f(x))-h(f(x'))\|_{\mathcal{H}_{k}}^{2}w(x)w(x')d\nu_{\mathbb{X}}(x)d\nu_{\mathbb{X}}(x'),
\end{align}
where the last inequality is due to the union bound.
For the second term in the right hand side of the inequality above, we have
\begin{align}
    &\;\;\;\;-\int_{\mathbb{X}\times\mathbb{X}\setminus\left(\bigcup_{i=1}^{K}\mathbb{M}_{i}\times\mathbb{M}_{i}\right)}\|h(f(x))-h(f(x'))\|_{\mathcal{H}_{k}}^{2}w(x)w(x')d\nu_{\mathbb{X}}(x)d\nu_{\mathbb{X}}(x')\nonumber\\
    &\geq -\sum_{i\neq j}\int_{\mathbb{M}_{i}\times\mathbb{M}_{j}}\|h(f(x))-h(f(x'))\|_{\mathcal{H}_{k}}^{2}w(x)w(x')d\nu_{\mathbb{X}}(x)d\nu_{\mathbb{X}}(x')\nonumber\\
    \label{eq:theorem 5.1 exeq4}
    &\;\;\;\;\;\;\;\;-\int_{\mathbb{X}\times\mathbb{X}\setminus\left(\bigcup_{i,j=1}^{K}\mathbb{M}_{i}\times\mathbb{M}_{j}\right)}\|h(f(x))-h(f(x'))\|_{\mathcal{H}_{k}}^{2}w(x)w(x')d\nu_{\mathbb{X}}(x)d\nu_{\mathbb{X}}(x') \\
    \label{eq:negative term 1}
    &\geq -\sum_{i\neq j}\int_{\mathbb{M}_{i}\times\mathbb{M}_{j}}\|h(f(x))-h(f(x'))\|_{\mathcal{H}_{k}}^{2}w(x)w(x')d\nu_{\mathbb{X}}(x)d\nu_{\mathbb{X}}(x')-M_{k}P_{\mathbb{X}}^{\otimes 2}\left(\mathbb{X}\times\mathbb{X}\setminus\left(\bigcup_{i,j=1}^{K}\mathbb{M}_{i}\times\mathbb{M}_{j}\right)\right).
\end{align}
Here, the second term of \eqref{eq:negative term 1} vanishes since Assumption~\ref{assumption:mixture of clusters} implies $\mathbb{X}\times\mathbb{X}=\bigcup_{i,j=1}^{K}\mathbb{M}_{i}\times\mathbb{M}_{j}$.
The first term of \eqref{eq:negative term 1} is further lower bounded as,
\begin{align}
    &\;\;\;\;-\sum_{i\neq j}\int_{\mathbb{M}_{i}\times\mathbb{M}_{j}}\|h(f(x))-h(f(x'))\|_{\mathcal{H}_{k}}^{2}w(x)w(x')d\nu_{\mathbb{X}}(x)d\nu_{\mathbb{X}}(x')\nonumber\\
    &\geq 2\sum_{i\neq j}\left\{-\sup_{z\in\mathbb{S}^{d-1}}k(z,z)P_{\mathbb{X}}(\mathbb{M}_{i})P_{\mathbb{X}}(\mathbb{M}_{j})\right.\nonumber\\
    \label{eq:theorem 5.1 exeq5}
    &\;\;\;\;\;\;\;\;\;\;\;\;\;\;\;\;\left.+\int_{\mathbb{M}_{i}\times\mathbb{M}_{j}}\langle h(f(x)),h(f(x'))\rangle_{\mathcal{H}_{k}}w(x)w(x')d\nu_{\mathbb{X}}(x)d\nu_{\mathbb{X}}(x')\right\}\\
    &=2\sum_{i\neq j}\left\{-\sup_{z\in\mathbb{S}^{d-1}}k(z,z)P_{\mathbb{X}}(\mathbb{M}_{i})P_{\mathbb{X}}(\mathbb{M}_{j})+P_{\mathbb{X}}(\mathbb{M}_{i})P_{\mathbb{X}}(\mathbb{M}_{j})\langle \mu_{i}(f),\mu_{j}(f)\rangle_{\mathcal{H}_{k}}\right\}\tag{Lemma~\ref{lem:inner product}}\\
    &=-2\sup_{z\in\mathbb{S}^{d-1}}k(z,z)\sum_{i=1}^{K}P_{\mathbb{X}}(\mathbb{M}_{i})\left(1-P_{\mathbb{X}}(\mathbb{M}_{i})\right)+2\mathfrak{c}(f).\nonumber
\end{align}
Thus for the negative term we obtain the inequality,
\begin{align}
\label{eq:thm 2 eq4}
    &\;\;\;\;-\mathbb{E}_{x,x^{-}}\left[\|h(f(x))-h(f(x^{-}))\|_{\mathcal{H}_{k}}^{2}\right]\nonumber\\
    &\geq -\sum_{i=1}^{K}\int_{\mathbb{M}_{i}\times\mathbb{M}_{i}}\|h(f(x))-h(f(x'))\|_{\mathcal{H}_{k}}^{2}w(x)w(x')d\nu_{\mathbb{X}}(x)d\nu_{\mathbb{X}}(x')\nonumber\\
    &\;\;\;\;\;\;+2\mathfrak{c}(f)-2\sup_{z\in\mathbb{S}^{d-1}}k(z,z)\sum_{i=1}^{K}P_{\mathbb{X}}(\mathbb{M}_{i})\left(1-P_{\mathbb{X}}(\mathbb{M}_{i})\right).
\end{align}

Combining \eqref{eq:thm 2 eq1},\eqref{eq:thm 2 eq2},\eqref{eq:thm 2 eq3}, and \eqref{eq:thm 2 eq4}, we have
\begin{align}
\label{eq:thm 2 eq5}
    &\;\;\;\; L_{\textup{KCL}}(f;\lambda)+(1-\lambda)\psi(1)\nonumber\\
    &\geq \frac{\delta}{2}\sum_{i=1}^{K}\mathbb{E}_{x,x^{-}}\left[\|h(f(x))-h(f(x^{-}))\|_{\mathcal{H}_{k}}^{2};\mathbb{M}_{i}\times\mathbb{M}_{i}\right]+\lambda\mathfrak{c}(f)\nonumber\\
    &\;\;\;\;\;\;-\frac{M_{k}}{2}\sum_{i\neq j}P_{+}((\mathbb{M}_{i}\cap\mathbb{M}_{j})\times(\mathbb{M}_{i}\cap\mathbb{M}_{j}))-\lambda\psi(1)\sum_{i=1}^{K}P_{\mathbb{X}}(\mathbb{M}_{i})\left(1-P_{\mathbb{X}}(\mathbb{M}_{i})\right).
\end{align}
Therefore, we complete the proof.
\end{proof}

\subsection{Proof of Corollary~\ref{cor:equality holds}}
\label{appsubsec:proof of corollary equality holds}

\begin{proof}[Proof of Corollary~\ref{cor:equality holds}]
The proof of Corollary~\ref{cor:equality holds} is completed by checking whether the equality holds in each inequality that appears in the proof of Theorem~\ref{thm:decomposition}.
We list the detail of the checks below:
\begin{itemize}
    \item[\eqref{eq:theorem 5.1 exeq1}:] Since $w(x,x')=0$ for any $(x,x')\in\mathbb{M}_{i}\times\mathbb{M}_{j}$ ($i\neq j$), we have $\int_{\mathbb{M}_{i}\times\mathbb{M}_{j}}\|h(f(x))-h(f(x'))\|_{\mathcal{H}_{k}}^{2}w(x,x')d\nu_{\mathbb{X}}(x)d\nu_{\mathbb{X}}(x')=0$. Here, we have the decomposition $\mathbb{X}\times\mathbb{X}=(\bigcup_{i=1}^{K}\mathbb{M}_{i})\times(\bigcup_{j=1}^{K}\mathbb{M}_{j})=\bigcup_{i,j=1}^{K}\mathbb{M}_{i}\times\mathbb{M}_{j}$, where $(\mathbb{M}_{i}\times\mathbb{M}_{j})\cap(\mathbb{M}_{i'}\times\mathbb{M}_{j'})=\emptyset$ for any $(i,j,i',j')$ such that $i\neq i'$ or $j\neq j'$, from the assumption that $\mathbb{M}_{1},\cdots,\mathbb{M}_{K}$ are disjoint. Hence, using the additivity of a probability measure yields the equality.
    \item[\eqref{eq:theorem 5.1 exeq2}:] Since $\mathbb{M}_{i}\cap\mathbb{M}_{j}=\emptyset$ for $i,j\in[K]$ such that $i\neq j$, the first term in the right-hand-side of \eqref{eq:theorem 5.1 exeq2} is equal to the first term in the left-hand-side of \eqref{eq:theorem 5.1 exeq2}. On the other hand, the second term in the right-hand-side of \eqref{eq:theorem 5.1 exeq2} is equal to 0. Hence, the equality holds.
    \item[\eqref{eq:thm 2 eq2}:] Since the second term of the right-hand-side of \eqref{eq:thm 2 eq2} is 0 under the assumption that $\mathbb{M}_{i}\cap\mathbb{M}_{j}=\emptyset$ for $i,j$ $(i\neq j)$, the equality holds.
    \item[\eqref{eq:thm 2 eq3}:] The equality holds from the assumption that for any $x,x'\in\mathbb{M}_{i}$ ($i\in[K]$), $\textup{sim}(x,x';\lambda)=\delta$ holds. Indeed, this assumption implies that $w(x,x')=(\lambda+\delta)w(x)w(x')$ for any $x,x'\in\mathbb{M}_{i}$.
    \item[\eqref{eq:theorem 5.1 exeq3}:] Since $\mathbb{M}_{1}\times\mathbb{M}_{1},\cdots,\mathbb{M}_{K}\times\mathbb{M}_{K}$ are disjoint, the equality holds.
    \item[\eqref{eq:theorem 5.1 exeq4}:] Since $\bigcup_{i,j=1}^{K}\mathbb{M}_{i}\times\mathbb{M}_{j}=\mathbb{X}\times\mathbb{X}$, the second term of the right-hand-side of \eqref{eq:theorem 5.1 exeq4} is equal to 0. Thus, the equality holds.
    \item[\eqref{eq:negative term 1}:] The equality holds due to the same reason as \eqref{eq:theorem 5.1 exeq4} above.
    \item[\eqref{eq:theorem 5.1 exeq5}:] Since $\|h(f(x))\|_{\mathcal{H}_{k}}^{2}=k(f(x),f(x))=\psi(f(x)^{\top}f(x))=\psi(1)$ for any $x\in\mathbb{X}$ and $f\in\mathcal{F}$, the equality holds.
    \item[\eqref{eq:thm 2 eq4}:] Since \eqref{eq:thm 2 eq4} is the combination of \eqref{eq:theorem 5.1 exeq3}, \eqref{eq:theorem 5.1 exeq4}, \eqref{eq:negative term 1}, and \eqref{eq:theorem 5.1 exeq5}, the equality in \eqref{eq:thm 2 eq4} holds in this case.
    \item[\eqref{eq:thm 2 eq5}:] Since \eqref{eq:thm 2 eq5} is obtained by combining \eqref{eq:thm 2 eq2}, \eqref{eq:thm 2 eq3}, and \eqref{eq:thm 2 eq4}, the equality holds. 
\end{itemize}
Therefore, we obtain the result.
\end{proof}

\section{Proof in Section~\ref{subsec:a surrogate upper bound of NN classification error}}
\label{appsubsec:proof of theorem nn classification error}

We present the proof of a generalized version of Theorem~\ref{thm:guarantee}.
The generalized theorem is presented below.

\begin{theorem}[The generalization of Theorem~\ref{thm:guarantee}]
\label{thm:nn classification error}
Suppose that Assumption~\ref{assumption:kernel} and \ref{assumption:mixture of clusters} hold.
Take $K\in\mathbb{N}$ and $\mathbb{M}_{1},\cdots,\mathbb{M}_{K}$ such that the condition \textbf{\textup{(A)}} in Assumption~\ref{assumption:mixture of clusters} is satisfied.
Let $\widetilde{\mathbb{M}}_{1},\cdots\widetilde{\mathbb{M}}_{K}$ be a disjoint partition of $\mathbb{X}$ satisfying $\widetilde{\mathbb{M}}_{i}\subset\mathbb{M}_{i}$ for each $i\in[K]$.
Define $\widetilde{y}:\mathbb{X}\to[K]$ as $\widetilde{y}(x)=i$ for every $x\in\widetilde{\mathbb{M}}_{i}$.
Then, for each meaningful encoder $f\in\mathcal{F}$, we have
\begin{align*}
    L_{\textup{Err}}(f, W_{\mu},\beta_{\mu};\widetilde{y})\leq 
    \frac{8(K-1)}{\Delta_{\textup{min}}(f)\cdot\min_{i\in[K]}P_{\mathbb{X}}(\mathbb{M}_{i})}\mathfrak{a}(f)
\end{align*}
where $\Delta_{\textup{min}}(f)=\min_{i\neq j}\|\mu_{i}(f)-\mu_{j}(f)\|_{\mathcal{H}_{k}}^{2}$.
\end{theorem}

\begin{proof}[Proof of Theorem~\ref{thm:nn classification error}]
From the definition, we have $\widetilde{\mathbb{M}}_{i}\subset\mathbb{M}_{i}$ and $\widetilde{\mathbb{M}}_{i}\cap\widetilde{\mathbb{M}}_{j}=\emptyset$ for all the pairs of distinct indices $i,j\in[K]$.
Let us recall the definition of $L_{\textup{Err}}(f,W_{\mu},\beta_{\mu};\widetilde{y})$:
\begin{align*}
L_{\textup{err}}(f,W_{\mu},\beta_{\mu};\widetilde{y})=P_{\mathbb{X}}\left(g_{f,W_{\mu},\beta_{\mu}}(x)\neq \widetilde{y}(x)\right).    
\end{align*}
Here recall that we let $\textup{arg~max}, \textup{arg~min}$ also breaks tie arbitrary.
For instance, if there are distinct integers $i_{1},\cdots,i_{j}\in[K]$ such that $g_{f,W_{\mu},\beta_{\mu}}(x)=\{i_{1},\cdots,i_{j}\}$, then we define $g_{f,W_{\mu},\beta_{\mu}}(x)=\widetilde{y}(x)$ if $\widetilde{y}(x)\in\{i_{1},\cdots,i_{j}\}$, and $g_{f,W_{\mu},\beta_{\mu}}(x)=i_{1}$ if $\widetilde{y}(x)\notin\{i_{1},\cdots,i_{j}\}$.
The event $\mathbb{A}:=\{x\;|\;g_{f,W_{\mu},\beta_{\mu}}(x)\neq \widetilde{y}(x)\}=\{x\;|\;g_{f,W_{\mu},\beta_{\mu}}(x)\neq \widetilde{y}(x)\}\cap\bigcup_{i=1}^{K}\widetilde{\mathbb{M}}_{i}\subset\mathbb{X}$ is a subset of the event $\mathbb{D}:=\bigcup_{i=1}^{K}\bigcup_{j\neq i}\{x\;|\;\|h(f(x))-\mu_{i}(f)\|_{\mathcal{H}_{k}}\geq \|h(f(x))-\mu_{j}(f)\|_{\mathcal{H}_{k}}\}\cap\widetilde{\mathbb{M}}_{i}$, since
\begin{align}
    &\qquad\quad x\in\mathbb{A}\nonumber\\
    &\iff \textup{arg~min}_{i\in[K]}\|h(f(x))-\mu_{i}(f)\|_{\mathcal{H}_{k}}\neq \widetilde{y}(x)\quad\textup{and}\quad x\in\bigcup_{i=1}^{K}\widetilde{\mathbb{M}}_{i}\tag{def. of $g_{f,W_{\mu},\beta_{\mu}}$ and $g_{\textup{1-NN}}$}\\
    &\iff x\in \bigcup_{j\neq \widetilde{y}(x)} \{x'\;|\;\|h(f(x'))-\mu_{\widetilde{y}(x)}(h_{f})\|_{\mathcal{H}_{k}}\geq\|h(f(x'))-\mu_{j}(f)\|_{\mathcal{H}_{k}}\}\nonumber\\
    &\qquad\qquad\qquad\qquad\qquad\qquad\qquad\qquad\qquad\qquad\qquad\qquad\qquad\quad\textup{and}\quad x\in\bigcup_{i=1}^{K}\widetilde{\mathbb{M}}_{i}\nonumber\\
    &\iff x\in\bigcup_{j\neq \widetilde{y}(x)}\{x'\;|\;\|h(f(x'))-\mu_{\widetilde{y}(x)}(h_{f})\|_{\mathcal{H}_{k}}\geq \|h(f(x'))-\mu_{j}(f)\|_{\mathcal{H}_{k}}\}\cap\widetilde{\mathbb{M}}_{\widetilde{y}(x)}\nonumber\\
    &\;\;\Longrightarrow x\in\bigcup_{i=1}^{K}\bigcup_{j\neq i}\{x'\;|\;\|h(f(x'))-\mu_{i}(f)\|_{\mathcal{H}_{k}}\geq \|h(f(x'))-\mu_{j}(f)\|_{\mathcal{H}_{k}}\}\cap\widetilde{\mathbb{M}}_{i}=\mathbb{D}.\nonumber
\end{align}

Define $\mathbb{L}_{ij}:=\{c(\mu_{j}(f)-\mu_{i}(f))\;|\;c\in\mathbb{R}\}\subset\mathcal{H}_{k}$ for every $i,j\in[K], i\neq j$. 
Since each $\mathbb{L}_{ij}$ is a closed subspace of $\mathcal{H}_{k}$, for every $z\in\mathcal{H}_{k}$ there exists some $z_{1}\in\mathbb{L}_{ij}$ and $z_{2}\in\mathbb{L}_{ij}^{\perp}$ (where $\mathbb{L}_{ij}^{\perp}$ is the orthogonal complement space of $\mathbb{L}_{ij}$) such that $z$ admits the unique decomposition $z=z_{1}+z_{2}$.
Here define the projection $\widetilde{\pi}_{ij}:\mathcal{H}_{k}\to\mathbb{L}_{ij}$ as $\widetilde{\pi}_{ij}(z)=z_{1}$, and define the shifted projection $\pi_{ij}$ as
$
    \pi_{ij}:\mathcal{H}_{k}\to\mathcal{H}_{k}, \pi_{ij}(z):=\widetilde{\pi}_{ij}(z-\mu_{i}(f))+\mu_{i}(f).
$
From the definition, we have that $\|\pi_{ij}(z)-\mu_{i}(f)\|_{\mathcal{H}_{k}}\leq\|z-\mu_{i}(f)\|_{\mathcal{H}_{k}}$ and $\|\pi_{ij}(z)-\mu_{j}(f)\|_{\mathcal{H}_{k}}\leq\|z-\mu_{j}(f)\|_{\mathcal{H}_{k}}$.

Hereafter, we use the abbreviation $\Delta_{ij}:=\Delta_{ij}(f)=\|\mu_{i}(f)-\mu_{j}(f)\|_{\mathcal{H}_{k}}^{2}$ for the sake of convenience.
Using $\pi_{ij}$, $i,j\in[K], i\neq j$, the event $\mathbb{D}$ can be decomposed into,
\begin{align*}
    \mathbb{D}&=\underbrace{\left(\mathbb{D}\cap\left(\bigcup_{i=1}^{K}\bigcup_{j\neq i}\left\{x\;|\;\|\pi_{ij}(h(f(x)))-\mu_{j}(f)\|_{\mathcal{H}_{k}}\leq \frac{1}{2}\Delta_{ij}^{\frac{1}{2}}\right\}\cap\widetilde{\mathbb{M}}_{i}\right)\right)}_{=\;\mathbb{D}_{1}}\\
    &\;\;\;\;\;\;\cup
    \underbrace{\left(
    \mathbb{D}\cap\left(\bigcup_{i=1}^{K}\bigcup_{j\neq i}\left\{x\;|\;\|\pi_{ij}(h(f(x)))-\mu_{j}(f)\|_{\mathcal{H}_{k}}\leq \frac{1}{2}\Delta_{ij}^{\frac{1}{2}}\right\}\cap\widetilde{\mathbb{M}}_{i}\right)^{c}
    \right)}_{=\;\mathbb{D}_{2}}.
\end{align*}
For $\mathbb{D}_{1}$, we have
\begin{align}
    &\;\;\;\;P_{\mathbb{X}}(\mathbb{D}_{1})\nonumber\\
    &\leq
    P_{\mathbb{X}}\left(\bigcup_{i=1}^{K}\bigcup_{j\neq i}\left\{x\;|\;\|\pi_{ij}(h(f(x)))-\mu_{j}(f)\|_{\mathcal{H}_{k}}\leq \frac{1}{2}\Delta_{ij}^{\frac{1}{2}}\right\}\cap\widetilde{\mathbb{M}}_{i}\right)\nonumber\\
    &\leq \sum_{i=1}^{K}\sum_{j\neq i}P_{\mathbb{X}}\left(\left\{x\;|\;\|\pi_{ij}(h(f(x)))-\mu_{j}(f)\|_{\mathcal{H}_{k}}\leq \frac{1}{2}\Delta_{ij}^{\frac{1}{2}}\right\}\cap\widetilde{\mathbb{M}}_{i}\right)\tag{the union bound}\\
    &\leq \sum_{i=1}^{K}\sum_{j\neq i}P_{\mathbb{X}}\left(\left\{x\;|\;-\|\pi_{ij}(h(f(x)))-\mu_{i}(f)\|_{\mathcal{H}_{k}}+\|\mu_{i}(f)-\mu_{j}(f)\|_{\mathcal{H}_{k}}\leq \frac{1}{2}\Delta_{ij}^{\frac{1}{2}}\right\}\cap\widetilde{\mathbb{M}}_{i}\right)\tag{triangle ineq.}\\
    &=\sum_{i=1}^{K}\sum_{j\neq i}P_{\mathbb{X}}\left(\left\{x\;|\;\|\pi_{ij}(h(f(x)))-\mu_{i}(f)\|_{\mathcal{H}_{k}}\geq \frac{1}{2}\Delta_{ij}^{\frac{1}{2}}\right\}\cap\widetilde{\mathbb{M}}_{i}\right)\nonumber\\
    &\leq \sum_{i=1}^{K}\sum_{j\neq i}\frac{4}{\Delta_{ij}}\mathbb{E}\left[\|\pi_{ij}(h(f(x)))-\mu_{i}(f)\|_{\mathcal{H}_{k}}^{2};\widetilde{\mathbb{M}}_{i}\right]\tag{Markov's ineq.}\\
    &\leq \sum_{i=1}^{K}\sum_{j\neq i}\frac{4}{\Delta_{ij}}\mathbb{E}\left[\|h(f(x))-\mu_{i}(f)\|_{\mathcal{H}_{k}}^{2};\widetilde{\mathbb{M}}_{i}\right]\tag{def. of $\pi_{ij}$}\\
    &\leq \sum_{i=1}^{K}\sum_{j\neq i}\frac{4}{\Delta_{ij}}\mathbb{E}\left[\|h(f(x))-\mu_{i}(f)\|_{\mathcal{H}_{k}}^{2};\mathbb{M}_{i}\right].\tag{def. of $\widetilde{\mathbb{M}_{i}}$}
\end{align}

For $\mathbb{D}_{2}$, we note that we can rewrite as,
\begin{align}
    &\;\;\;\;P_{\mathbb{X}}(\mathbb{D}_{2})\nonumber\\
    &= P_{\mathbb{X}}\left(\mathbb{D}\cap\left(\bigcup_{i=1}^{K}\bigcup_{j\neq i}\left\{x\;|\;\|\pi_{ij}(h(f(x)))-\mu_{j}(f)\|_{\mathcal{H}_{k}}\leq \frac{1}{2}\Delta_{ij}^{\frac{1}{2}}\right\}\cap\widetilde{\mathbb{M}}_{i}\right)^{c}\right)\nonumber\\
    &= P_{\mathbb{X}}\left(\left(\bigcup_{i=1}^{K}\bigcup_{j\neq i}\{x\;|\;\|h(f(x))-\mu_{i}(f)\|_{2}\geq \|h(f(x))-\mu_{j}(f)\|_{2}\}\cap\widetilde{\mathbb{M}}_{i}\right)\right.
    \cap\nonumber\\
    &\;\;\;\;\;\;\;\;\;\;\;\;\;\;\;\;\;\;\;\;\;\;\;\;\;\;\;\;\;\;
    \left.\left(\bigcap_{i=1}^{K}\bigcap_{j\neq i}\left\{x\;|\;\|\pi_{ij}(h(f(x)))-\mu_{j}(f)\|_{\mathcal{H}_{k}}> \frac{1}{2}\Delta_{ij}^{\frac{1}{2}}\right\}\cup\widetilde{\mathbb{M}}_{i}^{c}
    \right)\right)\nonumber\\
    &= P_{\mathbb{X}}\left(\bigcup_{i=1}^{K}\bigcup_{j\neq i}\bigcap_{i'=1}^{K}\bigcap_{j'\neq i'}\left(
    \{x\;|\;\|h(f(x))-\mu_{i}(f)\|_{2}\geq \|h(f(x))-\mu_{j}(f)\|_{2}\}\cap\widetilde{\mathbb{M}}_{i}\cap\right.\right.\nonumber\\
    &\;\;\;\;\;\;\;\;\;\;\;\;\;\;\;\;\;\;\;\;\;\;\;\;\;\;\;\;\;\;\;
    \left.
    \left(
    \left\{x\;|\;\|\pi_{i'j'}(h(f(x)))-\mu_{j'}(h_{f})\|_{\mathcal{H}_{k}}> \frac{1}{2}\Delta_{i'j'}^{\frac{1}{2}}\right\}\cup\widetilde{\mathbb{M}}_{i'}^{c}
    \right)\right)\Bigg)\nonumber
\end{align}
By using above, we have
\begin{align}
    &\;\;\;\;P_{\mathbb{X}}(\mathbb{D}_{2})\nonumber\\
    \label{eq:theorem 2 eq1}
    &\leq P_{\mathbb{X}}\left(\bigcup_{i=1}^{K}\bigcup_{j\neq i}
    \left(
    \{x\;|\;\|h(f(x))-\mu_{i}(f)\|_{\mathcal{H}_{k}}\geq \|h(f(x))-\mu_{j}(f)\|_{\mathcal{H}_{k}}\right.\right.\nonumber\\
    &\;\;\;\;\;\;\;\;\;\;\;\;\;\;\;\;\;\;\;\;\;\;\;\;\;\;\;\;\textup{ and }\|\pi_{ij}(h(f(x)))-\mu_{j}(f)\|_{\mathcal{H}_{k}}>\frac{1}{2}\Delta_{ij}^{\frac{1}{2}}\}\cap\widetilde{\mathbb{M}}_{i}
    )
    \Bigg)\\
    &\leq P_{\mathbb{X}}\left(\bigcup_{i=1}^{K}\bigcup_{j\neq i}
    \left(
    \{x\;|\;\|h(f(x))-\mu_{i}(f)\|_{\mathcal{H}_{k}}\geq \frac{1}{2}\Delta_{ij}^{\frac{1}{2}}\}\cap\widetilde{\mathbb{M}}_{i}
    \right)
    \right)\tag{def. of $\pi_{ij}$}\\
    &\leq \sum_{i=1}^{K}\sum_{j\neq i}P_{\mathbb{X}}\left(
    \{x\;|\;\|h(f(x))-\mu_{i}(f)\|_{\mathcal{H}_{k}}\geq \frac{1}{2}\Delta_{ij}^{\frac{1}{2}}\}\cap\widetilde{\mathbb{M}}_{i}
    \right)\tag{the union bound}\\
    &\leq \sum_{i=1}^{K}\sum_{j\neq i}\frac{4}{\Delta_{ij}}\mathbb{E}\left[\|\pi_{ij}(h(f(x)))-\mu_{i}(f)\|_{\mathcal{H}_{k}}^{2};\widetilde{\mathbb{M}}_{i}\right]\tag{Markov's ineq.}\\
    &\leq \sum_{i=1}^{K}\sum_{j\neq i}\frac{4}{\Delta_{ij}}\mathbb{E}\left[\|h(f(x))-\mu_{i}(f)\|_{\mathcal{H}_{k}}^{2};\widetilde{\mathbb{M}}_{i}\right]\tag{def. of $\pi_{ij}$}\\
    &\leq \sum_{i=1}^{K}\sum_{j\neq i}\frac{4}{\Delta_{ij}}\mathbb{E}\left[\|h(f(x))-\mu_{i}(f)\|_{\mathcal{H}_{k}}^{2};\mathbb{M}_{i}\right]. \tag{def. of $\widetilde{\mathbb{M}_{i}}$}
\end{align}
Here, let us show \eqref{eq:theorem 2 eq1}.
First let us fix $i,j\in[K]$, where $i\neq j$.
For $i',j'\in[K]$ satisfying $i'\neq j'$, we consider the following two cases.
\begin{itemize}
    \item If $i'=i$ and $j'=j$, then $\widetilde{\mathbb{M}}_{i}\cap \widetilde{\mathbb{M}}_{i}^{c}=\emptyset$, which implies
    \begin{align*}
    &\hskip -1em
    \left\{x\;|\;\|h(f(x))-\mu_{i}(f)\|_{2}\geq \|h(f(x))-\mu_{j}(f)\|_{2}\right\}\cap\widetilde{\mathbb{M}}_{i}\cap
    \left(
    \left\{x\;|\;\|\pi_{i'j'}(h(f(x)))-\mu_{j'}(h_{f})\|_{\mathcal{H}_{k}}> \frac{1}{2}\Delta_{i'j'}^{\frac{1}{2}}\right\}\cup\widetilde{\mathbb{M}}_{i'}^{c}
    \right)\\
    &=\left\{x\;|\;\|h(f(x))-\mu_{i}(f)\|_{\mathcal{H}_{k}}\geq \|h(f(x))-\mu_{j}(f)\|_{\mathcal{H}_{k}}\textup{ and }\|\pi_{ij}(h(f(x)))-\mu_{j}(f)\|_{\mathcal{H}_{k}}>\frac{1}{2}\Delta_{ij}^{\frac{1}{2}}\right\}\cap\widetilde{\mathbb{M}}_{i}.
    \end{align*}
    \item if $i'\neq i$ or $j'\neq j$, then
    \begin{align*}
        &\hskip -1em
        \{x\;|\;\|h(f(x))-\mu_{i}(f)\|_{2}\geq \|h(f(x))-\mu_{j}(f)\|_{2}\}\cap\widetilde{\mathbb{M}}_{i}\cap
    \left(
    \left\{x\;|\;\|\pi_{i'j'}(h(f(x)))-\mu_{j'}(h_{f})\|_{\mathcal{H}_{k}}> \frac{1}{2}\Delta_{i'j'}^{\frac{1}{2}}\right\}\cup\widetilde{\mathbb{M}}_{i'}^{c}
    \right)\\
    &\subset\widetilde{\mathbb{M}}_{i}.
    \end{align*}
\end{itemize}
Thus,
\begin{align*}
    &\;\;\;\;\bigcap_{i'=1}^{K}\bigcap_{j'\neq i'}\bigg(
    \{x\;|\;\|h(f(x))-\mu_{i}(f)\|_{2}\geq \|h(f(x))-\mu_{j}(f)\|_{2}\}\\
    &\;\;\;\;\;\;\;\;\;\;\;\;\;\;\;\;\;\;\;\;\;\;\left.\cap\widetilde{\mathbb{M}}_{i}\cap
    \left(
    \left\{x\;|\;\|\pi_{i'j'}(h(f(x)))-\mu_{j'}(h_{f})\|_{\mathcal{H}_{k}}> \frac{1}{2}\Delta_{i'j'}^{\frac{1}{2}}\right\}\cup\widetilde{\mathbb{M}}_{i'}^{c}
    \right)\right)\\
    &\subset \{x\;|\;\|h(f(x))-\mu_{i}(f)\|_{\mathcal{H}_{k}}\geq \|h(f(x))-\mu_{j}(f)\|_{\mathcal{H}_{k}}\textup{ and }\|\pi_{ij}(h(f(x)))-\mu_{j}(f)\|_{\mathcal{H}_{k}}>\frac{1}{2}\Delta_{ij}^{\frac{1}{2}}\}\cap\widetilde{\mathbb{M}}_{i}.
\end{align*}

By combining all the results, we obtain
\begin{align}
    P_{\mathbb{X}}(\mathbb{A})&\leq P_{\mathbb{X}}(\mathbb{D})\nonumber\\
    &\leq P_{\mathbb{X}}(\mathbb{D}_{1})+P_{\mathbb{X}}(\mathbb{D}_{2})\nonumber\\
    &\leq \sum_{i=1}^{K}\sum_{j\neq i}\frac{8}{\Delta_{ij}}\mathbb{E}\left[\|h(f(x))-\mu_{i}(f)\|_{\mathcal{H}_{k}}^{2};\mathbb{M}_{i}\right]\nonumber\\
    &\leq \frac{8(K-1)}{\Delta_{\textup{min}}(f)}\sum_{i=1}^{K}\mathbb{E}[\|h(f(x))-\mu_{i}(f)\|_{\mathcal{H}_{k}}^{2};\mathbb{M}_{i}]\nonumber\\
    &\leq \frac{8(K-1)}{\Delta_{\textup{min}}(f)}\sum_{i=1}^{K}\frac{1}{P_{\mathbb{X}}(\mathbb{M}_{i})}\mathbb{E}_{x,x^{-}}[\|h(f(x))-h(f(x^{-}))\|_{\mathcal{H}_{k}}^{2};\mathbb{M}_{i}\times\mathbb{M}_{i}]\tag{Jensen's inequality}\\
    &\leq \frac{8(K-1)}{\Delta_{\textup{min}}(f)\cdot\min_{i\in[K]}P_{\mathbb{X}}(\mathbb{M}_{i})}\mathfrak{a}(f),\nonumber
\end{align}
and we complete the proof.
\end{proof}

\begin{proof}[Proof of Theorem~\ref{thm:guarantee}]
From the definition of $y$, it is guaranteed that the sets $\{x\in\mathbb{X}\;|\;y(x)=i\}$ for $i=1,\cdots,K$ are disjoint and satisfy the relation $\{x\in\mathbb{X}\;|\;y(x)=i\}\subseteq \mathbb{M}_{i}$ for every $i\in[K]$.
Thus, Theorem~\ref{thm:nn classification error} can apply to this case, and we obtain the result.
\end{proof}

\section{Proofs in Section~\ref{subsec:rethinking generalization of contrastive learning}}
\label{appsec:proofs in section 4.3}

\subsection{Proof of Theorem~\ref{lem:ulln}}
\label{appsubsec:proofs of lemma and theorem in section 4.3}

First, we prove Theorem~\ref{lem:ulln}.
Before that, we present the following theorem, which is a part of the proof of Theorem~\ref{lem:ulln}.

\begin{theorem}
\label{thm:dependent ulln}
Let $(X_{1},X_{1}'),\cdots,(X_{n},X_{n}')$ be random variables introduced in Section~\ref{subsec:rethinking generalization of contrastive learning}.
Suppose that Assumption~\ref{assumption:kernel} holds, and suppose that $n$ is even.
Then, with probability at least $1-\varepsilon$, the following inequality holds:
  \begin{align*}
      \sup_{f\in\mathcal{F}}\left(-\frac{1}{n(n-1)}\sum_{i\neq j}k(f(X_{i}),f(X_{j}'))+\mathbb{E}_{X,X^{-}}\left[k(f(X),f(X^{-}))\right]\right)\leq 2\rho\mathfrak{R}_{n/2}^{-}(\mathcal{Q};s^{*})+\sqrt{\frac{10b^{2}\log\left(1/\varepsilon\right)}{n}},
  \end{align*}
where we define $\mathfrak{R}_{n/2}^{-}(\mathcal{Q};s^{*})$ with the symmetric group $S_{n}$ of degree $n$: 
\begin{align*}
    \mathfrak{R}_{n/2}^{-}(\mathcal{Q};s^{*}):=\max_{s\in S_{n}}\mathbb{E}_{\substack{X,X'\\\sigma_{1:(n/2)}}}\left[\sup_{f\in\mathcal{F}}\frac{2}{n}\sum_{i=1}^{n/2}\sigma_{i}f(X_{s(2i-1)})^{\top}f(X_{s(2i)}')\right].
\end{align*}
\end{theorem}

We remark that in Theorem~\ref{thm:dependent ulln}, we need to deal with more delicate technical matters compared to the typical generalization error bounds (e.g., Theorem~3.3 of \citet{mohri2018foundations}), since in our setup $X_{1},X_{1}',\cdots,X_{n},X_{n}'$ are not necessarily independent to each other.
We give the proof of Theorem~\ref{thm:dependent ulln} in Appendix~\ref{appsubsec:proof of theorem dependent ulln}.

Now, we can show Theorem~\ref{lem:ulln}.

\begin{proof}[Proof of Theorem~\ref{lem:ulln}]
First observe that,
\begin{align*}
    &\;\;\;\;\sup_{f\in\mathcal{F}}\left(-\widehat{L}_{\textup{KCL}}(f;\lambda)+L_{\textup{KCL}}(f;\lambda)\right)\\
    &=
    \sup_{f\in\mathcal{F}}\left(\frac{1}{n}\sum_{i=1}^{n}k(f(X_{i}),f(X_{i}^{'}))-\frac{\lambda}{n(n-1)}\sum_{i\neq j}k(f(X_{i}),f(X_{j}'))\right.\\
    &\qquad\qquad\qquad\qquad -\mathbb{E}_{X,X^{+}}\left[k(f(X),f(X^{+}))\right]+\lambda\mathbb{E}_{X,X^{-}}\left[k(f(X),f(X^{-}))\right]\Bigg)\\
    &\leq \underbrace{\sup_{f\in\mathcal{F}}\left(\frac{1}{n}\sum_{i=1}^{n}k(f(X_{i}),f(X_{i}^{'}))-\mathbb{E}_{X,X^{+}}\left[k(f(X),f(X^{+}))\right]\right)}_{\textup{(i)}}\\
    &\qquad\quad +\lambda\underbrace{ \sup_{f\in\mathcal{F}}\left(-\frac{1}{n(n-1)}\sum_{i\neq j}k(f(X_{i}),f(X_{j}'))+\mathbb{E}_{X,X^{-}}\left[k(f(X),f(X^{-}))\right]\right)}_{\textup{(ii)}}.\\
\end{align*}

Let us define the function space $\mathcal{K}:=\{k(f(\cdot),f(\cdot)):\mathbb{X}\times\mathbb{X}\to\mathbb{R}\;|\;f\in\mathcal{F}\}$.
Then $\mathcal{K}$ is uniformly bounded with constant $b=\sup_{z,z\in\mathbb{S}^{d-1}}|k(z,z')|$.
Here we note that $b<+\infty$ holds since $k$ is continuous and $\mathbb{S}^{d-1}$ is compact; see Section~\ref{subsec:problem setup}.
From the ULLNs~(Theorem~3.3 in~\citet{mohri2018foundations}), with probability at least $1-\varepsilon/2$, we have
\begin{align*}
    \textup{(i)}\leq 2\mathfrak{R}_{n}^{+}(\mathcal{K})+\sqrt{\frac{2b^{2}\log\left(2/\varepsilon\right)}{n}}.
\end{align*}
Since $k$ is represented by $k(x,x')=\psi(x^{\top}x')$ for some $\rho$-Lipshitz function $\psi$ from Assumption~\ref{assumption:kernel}, by applying Talagrand's lemma (Lemma~26.9 in~\citet{shalev2014understanding}) we have $\mathfrak{R}_{n}^{+}(\mathcal{K})\leq \rho\mathfrak{R}_{n}^{+}(\mathcal{Q})$. Hence, with probability at least $1-\varepsilon/2$, we have
\begin{align*}
    \textup{(i)}\leq 2\rho\mathfrak{R}_{n}^{+}(\mathcal{Q})+\sqrt{\frac{2b^{2}\log\left(2/\varepsilon\right)}{n}}.
\end{align*}

For (ii), from Theorem~\ref{thm:dependent ulln}, with probability at least $1-\varepsilon/2$ we have
\begin{align*}
    \textup{(ii)}\leq 2\rho\mathfrak{R}_{n/2}^{-}(\mathcal{Q};s^{*})+\sqrt{\frac{10b^{2}\log\left(2/\varepsilon\right)}{n}}.
\end{align*}
Therefore, with probability at least $1-\varepsilon$ we have,
\begin{align}
\label{eq:gen err bound 1}
    \sup_{f\in\mathcal{F}}\left(-\widehat{L}_{\textup{KCL}}(f;\lambda)+L_{\textup{KCL}}(f;\lambda)\right)\leq 2\rho\mathfrak{R}_{n}(\mathcal{Q})+\sqrt{\frac{2b^{2}\log\left(2/\varepsilon\right)}{n}}+\lambda\sqrt{\frac{10b^{2}\log\left(2/\varepsilon\right)}{n}},
\end{align}
where $\mathfrak{R}_{n}(\mathcal{Q}):=\mathfrak{R}_{n}^{+}(\mathcal{Q})+\lambda\mathfrak{R}_{n/2}^{-}(\mathcal{Q};s^{*})$.

Note that in the same way as the proof of the above probability bound, we have the following inequality: with probability at least $1-\varepsilon$,
\begin{align}
\label{eq:gen err bound 2}
    \sup_{f\in\mathcal{F}}\left(\widehat{L}_{\textup{KCL}}(f;\lambda)-L_{\textup{KCL}}(f;\lambda)\right)\leq 2\rho\mathfrak{R}_{n}(\mathcal{Q})+\sqrt{\frac{2b^{2}\log\left(2/\varepsilon\right)}{n}}+\lambda\sqrt{\frac{10b^{2}\log\left(2/\varepsilon\right)}{n}}.
\end{align}
Hence, let $\widehat{f}$ be the minimizer of $\widehat{L}_{\textup{KCL}}(f;\lambda)$, then from \eqref{eq:gen err bound 1} and \eqref{eq:gen err bound 2}, with probability at least $1-2\varepsilon$ we have
\begin{align*}
    L_{\textup{KCL}}(\widehat{f};\lambda)\leq L_{\textup{KCL}}(f;\lambda)+ 4\rho\mathfrak{R}_{n}(\mathcal{Q})+2\sqrt{\frac{2b^{2}\log\left(2/\varepsilon\right)}{n}}+2\lambda\sqrt{\frac{10b^{2}\log\left(2/\varepsilon\right)}{n}},
\end{align*}
where we note that $\widehat{L}_{\textup{KCL}}(\widehat{f};\lambda)\leq \widehat{L}_{\textup{KCL}}(f;\lambda)$ from the definition of $\widehat{f}$.
Therefore, we complete the proof.
\end{proof}

\subsection{An Upper Bound of the Rademacher Complexity}
\label{appsubsec:proof of lemma chaining}

In this section for the sake of simplicity, we consider the case in which for every $f\in\mathcal{F}$, there exists the unique function $f_{0}\in\mathcal{F}_{0}$ such that $f(x)=f_{0}(x)/\|f_{0}(x)\|_{2}$ for every $x\in\mathbb{X}$.
First let us recall the definition of a sub-Gaussian process:

\begin{definition}[Quoted from Definition 5.16 in~\citet{wainwright2019high}]
A collection of zero-mean random variables $\{X_{\theta},\theta\in\mathbb{T}\}$ is a sub-Gaussian process with respect to a metric $\rho_{X}$ on $\mathbb{T}$ if
\begin{align*}
    \mathbb{E}\left[e^{\lambda(X_{\theta}-X_{\widetilde{\theta}})}\right]\leq e^{\frac{\lambda^{2}\rho_{X}^{2}(\theta,\widetilde{\theta})}{2}}\quad\textup{for all}\;\theta,\widetilde{\theta}\in\mathbb{T}\textup{, and }\lambda\in\mathbb{R}.
\end{align*}
\end{definition}

We next upper bound the Rademacher complexity via the chaining technique~(Theorem 5.22 in~\citet{wainwright2019high}).

\begin{proposition}
\label{lem:chaining}
Suppose $n$ is even.
For $\mathfrak{R}_{n}(\mathcal{Q})$, we have the upper bound,
\begin{align*}
    \mathfrak{R}_{n}(\mathcal{Q})\leq \frac{64(1+\sqrt{2}\lambda)}{\mathfrak{m}(\mathcal{F}_{0})\sqrt{n}}\int_{0}^{Cd}\sqrt{\log \mathfrak{C}(u;\mathcal{F}_{0},\|\cdot\|_{\infty})}du,
\end{align*}
where $\|f_{0}\|_{\infty}:=\sup_{x\in\mathbb{X}}\|f_{0}(x)\|_{2}$ for $f_{0}\in\mathcal{F}_{0}$, $\mathfrak{m}(\mathcal{F}_{0})$ is defined in Section~\ref{subsec:problem setup}, $C$ is a constant independent of $d,n,\lambda$, and $\mathfrak{C}(u;\mathcal{F}_{0},\|\cdot\|_{\infty})$ is the $u$-covering number of $(\mathcal{F}_{0},\|\cdot\|_{\infty})$ (for the definition of covering number, see e.g., Definition 5.1 in~\citet{wainwright2019high}).
\end{proposition}

\begin{proof}[Proof of Proposition~\ref{lem:chaining}]
In this proof, we follow the proof idea of~\citet{Tu2019theoretical} (see Lemma~5 in~\citet{Tu2019theoretical}).
Since our setup is different from~\citet{Tu2019theoretical}, we need to modify the proof and add several new techniques.
Define
\begin{align*}
    Z_{f_{0}}:=\frac{\mathfrak{m}(\mathcal{F}_{0})}{2\sqrt{n}}\sum_{i=1}^{n}\sigma_{i}q(f_{0}(X_{i}),f_{0}(X_{i}')),
\end{align*}
where $\sigma_{1},\cdots,\sigma_{n}$ are Rademacher random variables that are independent to each other and to each $(X_{i},X_{i}')$, $i\in[n]$, $f_{0}\in\mathcal{F}_{0}$, $(X_{1},X_{1}),\cdots,(X_{n},X_{n}')$ are the random vectors defined in Section~\ref{subsec:rethinking generalization of contrastive learning}, and
\begin{align*}
    q(z,z')=\frac{z^{\top}z'}{\|z\|_{2}\cdot\|z'\|_{2}}\;\;\;\;z,z'\in\mathbb{S}^{d-1}.
\end{align*}
Also, let us recall the assumption for $\mathcal{F}_{0}$ introduced in Section~\ref{subsec:problem setup}: for every $f\in\mathcal{F}$, there exists the unique function $f_{0}\in\mathcal{F}_{0}$ such that $f(x)=f_{0}(x)/\|f_{0}(x)\|_{2}$ for every $x\in\mathbb{X}$.
We show that $\{Z_{f_{0}}\}_{f_{0}\in\mathcal{F}_{0}}$ is a sub-Gaussian process as follows: note that, for every $f_{1,0},f_{2,0}\in\mathcal{F}_{0}$,
\begin{align}
    &\;\;\;\;\frac{\mathfrak{m}(\mathcal{F}_{0})}{2\sqrt{n}}\left|\sigma_{i}(q(f_{1,0}(X_{i}),f_{1,0}(X_{i}'))-q(f_{2,0}(X_{i}),f_{2,0}(X_{i}')))\right|\nonumber\\
    &\leq \frac{\mathfrak{m}(\mathcal{F}_{0})}{2\sqrt{n}}\left|f_{1}(X_{i})^{\top}f_{1}(X_{i}')-f_{2}(X_{i})^{\top}f_{2}(X_{i}')\right|\tag{def. of $\sigma_{i}$}\\
    &\leq \frac{\mathfrak{m}(\mathcal{F}_{0})}{2\sqrt{n}}\left(\left|f_{1}(X_{i})^{\top}f_{1}(X_{i}')-f_{1}(X_{i})^{\top}f_{2}(X_{i}')\right|+\left|f_{1}(X_{i})^{\top}f_{2}(X_{i}')-f_{2}(X_{i})^{\top}f_{2}(X_{i}')\right|\right)\tag{triangle ineq.}\\
    &\leq \frac{\mathfrak{m}(\mathcal{F}_{0})}{2\sqrt{n}}\left(\|f_{1}(X_{i})\|_{2}\|f_{1}(X_{i}')-f_{2}(X_{i}')\|_{2}+\|f_{1}(X_{i})-f_{2}(X_{i})\|_{2}\|f_{2}(X_{i}')\|_{2}\right)\tag{Cauchy-Schwarz ineq.}\\
    &\leq \frac{\mathfrak{m}(\mathcal{F}_{0})}{2\sqrt{n}}\left(\|f_{1}(X_{i}')-f_{2}(X_{i}')\|_{2}+\|f_{1}(X_{i})-f_{2}(X_{i})\|_{2}\right)\tag{def. of $f_{1},f_{2}$}\\
    &\leq \frac{\mathfrak{m}(\mathcal{F}_{0})}{\sqrt{n}}\sup_{x\in\mathbb{X}}\|f_{1}(x)-f_{2}(x)\|_{2}\nonumber\\
    &= \frac{\mathfrak{m}(\mathcal{F}_{0})}{\sqrt{n}}\sup_{x\in\mathbb{X}}\left\|\frac{f_{1,0}(x)}{\|f_{1,0}(x)\|_{2}}-\frac{f_{2,0}(x)}{\|f_{2,0}(x)\|_{2}}\right\|_{2}\tag{the uniqueness of $f_{1,0},f_{2,0}$}\nonumber\\
    &\leq \frac{1}{\sqrt{n}}\|f_{1,0}-f_{2,0}\|_{\infty}.\tag{def. of $\mathfrak{m}(\mathcal{F}_{0})$}
\end{align}
Hence, we have
\begin{align*}
    \mathbb{E}_{X_{1:n},X_{1:n}',\sigma_{1:n}}\left[\exp\left(t(Z_{f_{1,0}}-Z_{f_{2,0}})\right)\right]
    \leq \exp\left(\frac{t^{2}}{2n}\|f_{1,0}-f_{2,0}\|_{\infty}^{2}\right)^{n}=\exp\left(\frac{t^{2}}{2}\|f_{1,0}-f_{2,0}\|_{\infty}^{2}\right).
\end{align*}
This indicates that $\{Z_{f_{0}}\}_{f\in\mathcal{F}_{0}}$ is a sub-Gaussian process with the norm $\|\cdot\|_{\infty}$.
Here note that $\sup_{f_{1,0},f_{2,0}\in\mathcal{F}_{0}}\|f_{1,0}-f_{2,0}\|_{\infty}\leq C\sqrt{d}$ for some constant $C\in\mathbb{R}$ that is independent of $d$, since $\mathcal{F}_{0}$ is uniformly bounded.
By using the chaining theorem~(Theorem~5.22 in~\citet{wainwright2019high}), we have
\begin{align*}
    \mathfrak{R}_{n}^{+}(\mathcal{Q})\leq \frac{64}{\mathfrak{m}(\mathcal{F}_{0})\sqrt{n}}\int_{0}^{C\sqrt{d}}\sqrt{\log\mathfrak{C}(u;\mathcal{F}_{0},\|\cdot\|_{\infty})}du.
\end{align*}
For $\mathfrak{R}_{n/2}^{-}(\mathcal{Q})$, in a similar way we obtain,
\begin{align*}
    \mathfrak{R}_{n/2}^{-}(\mathcal{Q})\leq \frac{64\sqrt{2}}{\mathfrak{m}(\mathcal{F}_{0})\sqrt{n}}\int_{0}^{C\sqrt{d}}\sqrt{\log\mathfrak{C}(u;\mathcal{F}_{0},\|\cdot\|_{\infty})}du.
\end{align*}
Thus, we have
\begin{align*}
    \mathfrak{R}_{n}(\mathcal{Q})\leq  \frac{64(1+\sqrt{2}\lambda)}{\mathfrak{m}(\mathcal{F}_{0})\sqrt{n}}\int_{0}^{C\sqrt{d}}\sqrt{\log\mathfrak{C}(u;\mathcal{F}_{0},\|\cdot\|_{\infty})}du,
\end{align*}
and complete the proof.
\end{proof}

The integral in the above upper bound is often called Dudley entropy integral~\citep{wainwright2019high}.
Proposition~\ref{lem:chaining} makes it easier to derive a generalization bound via chaining, since it is enough to evaluate the Dudley entropy integral for the function space $\mathcal{F}_{0}$ instead of the space of critic functions $\mathcal{Q}$.

Here, denote by $\mathfrak{D}(\mathcal{F}_{0},\|\cdot\|_{\infty})$, the Dudley entropy integral w.r.t. $(\mathcal{F}_{0},\|\cdot\|_{\infty})$, i.e., 
\begin{align*}
    \mathfrak{D}(\mathcal{F}_{0},\|\cdot\|_{\infty})=\int_{0}^{C\sqrt{d}}\sqrt{\log\mathfrak{C}(u;\mathcal{F}_{0},\|\cdot\|_{\infty})}du.
\end{align*}
It is shown by~\citet{Tu2019theoretical} that if $\mathcal{F}_{0}$ is a function space of feedforward (deep) neural networks, where each neural networks have weight matrices whose norms are bounded by some universal constant, and Lipschitz activation functions that vanish at the origin, then $\mathfrak{D}(\mathcal{F}_{0},\|\cdot\|_{\infty})<+\infty$ holds.
Based on this fact, we introduce:

\begin{assumption}
\label{assumption:function space}
The Dudley entropy integral $\mathfrak{D}(\mathcal{F}_{0},\|\cdot\|_{\infty})$ is finite, and $\mathfrak{R}_{n}(\mathcal{Q})\leq O((1+\lambda)/\sqrt{n})$ holds.
\end{assumption}

Consequently, we obtain the generalization error bound. 

\begin{corollary}
\label{thm:generalization bound}
Suppose that Assumption~\ref{assumption:kernel}, \ref{assumption:function space} hold, and $n$ is even.
Then, with probability at least $1-\varepsilon$ where $\varepsilon>0$, we have
\begin{align*}
    L_{\textup{KCL}}(f;\lambda)\leq \widehat{L}_{\textup{KCL}}(f;\lambda)+O\left(\frac{(1+\lambda)\left(1+\sqrt{\log\left(2/\varepsilon\right)}\right)}{\sqrt{n}}\right).
\end{align*}
\end{corollary}
\begin{proof}
Due to Theorem~\ref{lem:ulln} and Assumption~\ref{assumption:function space}.
\end{proof}

\subsection{Useful Results on McDiarmid's Inequality for Dependent Random Variables}
\label{subsubsec:useful results of mcdiarmid's inequality for dependent random variabels}

Before showing Theorem~\ref{thm:dependent ulln}, we need to prepare several definitions and an existing result.
The following three definitions are quoted from~\citet{zhang2019mcdiarmid}.

\begin{definition}[Dependency Graph, quoted from Definition~3.1 in~\citet{zhang2019mcdiarmid}]
An undirected graph $G$ is called a dependency graph of a random vector ${\mathbf{X}}=(X_{1},\cdots,X_{n})$ if
\begin{enumerate}
    \item $V(G)=[n]$
    \item if $I,J\subset [n]$ are non-adjacent in $G$, then $\{X_{i}\}_{i\in I}$ and $\{X_{j}\}_{j\in J}$ are independent.
\end{enumerate}
\end{definition}

\begin{definition}[Forest Approximation, quoted from Definition~3.4 in~\citet{zhang2019mcdiarmid}]
Given a graph $G$, a forest $F$, and a mapping $\phi: V(G)\to V(F)$, if $\phi(u)=\phi(v)$ or $\langle \phi(u),\phi(v)\rangle \in E(F)$ for any $\langle u,v\rangle \in E(G)$, we say that $(\phi, F)$ is a forest approximation of $G$.
Let $\Phi(G)$ denote the set of forest approximations of $G$.
\end{definition}

\begin{definition}[Forest Complexity, quoted from Definition~3.5 in~\citet{zhang2019mcdiarmid}]
Given a graph $G$ and any forest approximation $(\phi, F)\in\Phi(G)$ with $F$ consisting of trees $\{T_{i}\}_{i\in[k]}$, let
\begin{align*}
    \lambda_{(\phi,F)}=\sum_{\langle u,v\rangle \in E(F)}\left(|\phi^{-1}(u)|+|\phi^{-1}(v)|\right)^{2}+\sum_{i=1}^{k}\min_{u\in V(T_{i})}|\phi^{-1}(u)|^{2}.
\end{align*}
We call
\begin{align*}
    \Lambda(G)=\min_{(\phi,F)\in\Phi(G)}\lambda_{(\phi, F)}
\end{align*}
the forest complexity of the graph $G$.
\end{definition}

\citet{zhang2019mcdiarmid} have shown the following result, which is an extension of McDiarmid's inequality~\citep{mcdiarmid1989method} for dependent random variables.

\begin{theorem}[Quoted from Theorem~3.6 in~\citet{zhang2019mcdiarmid}]
\label{thm:ulln for dependent rvs}
Suppose that $f:\mathbf{\Omega}\to\mathbb{R}$ is a $\mathbf{c}$-Lipschitz function and $G$ is a dependency graph of a random vector ${\mathbf{X}}$ that takes values in $\mathbf{\Omega}$.
For any $t>0$, the following inequality holds:
\begin{align*}
    \mathbf{\textup{Pr}}(f(\mathbf{X})-\mathbf{E}[f(\mathbf{X})]\geq t)\leq \exp\left(-\frac{2t^{2}}{\Lambda(G)\|\mathbf{c}\|_{\infty}^{2}}\right).
\end{align*}
\end{theorem}
Note that, in the above theorem $f:\mathbf{\Omega}\to\mathbb{R}$ is said to be $\mathbf{c}$-Lipschitz if $|f(\mathbf{x})-f(\mathbf{x}')|\leq \sum_{i=1}^{p}\mathbf{c}_{i}\mathbbm{1}_{\{\mathbf{x}_{i}\neq \mathbf{x}_{i}'\}}$ for every $\mathbf{x},\mathbf{x}'\in\mathbf{\Omega}$, where $\mathbf{\Omega}\subset\mathbb{R}^{p}$ for some $p\in\mathbb{N}$.

\subsection{Proof of Theorem~\ref{thm:dependent ulln}}
\label{appsubsec:proof of theorem dependent ulln}

We show Theorem~\ref{thm:dependent ulln} by utilizing the contents in Appendix~\ref{subsubsec:useful results of mcdiarmid's inequality for dependent random variabels}.
Recall the definition of the random variables introduced in Section~\ref{subsec:rethinking generalization of contrastive learning}:
$(X_{1},X_{1}'),\cdots,(X_{n},X_{n}')$ are pairs of random variables sampled independently according to the joint probability distribution with density $w(x,x')$, where $X_{i}$ and $X_{j}'$ are independent for each pair of distinct indices $i,j\in[K]$.
From the definition, the following claim holds.

\begin{lemma}
\label{lem:lambda g n}
Let $G_{n}$ be a dependency graph that is defined with a random vector $(X_{1},X_{1}',\cdots,X_{n},X_{n}')$, where the edges in $G_{n}$ are defined as follows: for any $i,j\in[n]$, $X_{i}$ and $X_{j}$ are not connected, and $X_{i}$ and $X_{j}'$ are connected by an edge if and only if $i=j$.
Then, we have $\Lambda(G_{n})\leq 5n$.
\end{lemma}

\begin{proof}
Let $\phi:G_{n}\to G_{n}$ be the identity map.
From the definition, $G_{n}$ can be decomposed into trees $\{T_{i}\}_{i\in[n]}$ where $V(T_{i})=\{X_{i},X_{i}'\}$ for each $i\in[n]$.
Let $F$ be the forest consisting of the trees $\{T_{i}\}_{i\in[n]}$.
Then, we have $\lambda_{(\phi,F)}= 5n$, which implies $\Lambda (G_{n})\leq \lambda_{(\phi,F)}\leq 5n$.
\end{proof}

\begin{proof}[Proof of Theorem~\ref{thm:dependent ulln}]
The goal of this proof is to upper bound the following quantity with high probability:
  \begin{align*}
    \sup_{f\in\mathcal{F}}\left(-\frac{1}{n(n-1)}\sum_{i\neq j}k(f(X_{i}),f(X_{j}'))+\mathbb{E}_{X,X^{-}}\left[k(f(X_{i}),f(X^{-}))\right]\right).
  \end{align*}
  However, as explained before, the standard argument (see e.g., Theorem~3.3 in~\citet{mohri2018foundations}) cannot apply to this case since $k(f(X_{i}),f(X_{j}'))$, $i,j\in[n], i\neq j$ are not necessarily independent to each other from our problem setup.
  We instead utilize the McDiarmid's inequality for dependent random variables, which is shown by~\citet{zhang2019mcdiarmid}, to avoid this problem.
  Our proof below is mainly based on Theorem~3.3 in~\cite{mohri2018foundations}, but it includes some modification due to the application of the results by~\citet{zhang2019mcdiarmid}.
  Let $\widetilde{X}_{1},\widetilde{X}_{1}',\cdots,\widetilde{X}_{n},\widetilde{X}_{n}'$ be i.i.d. random variables to the original random variables $X_{1},X_{1}'.\cdots,X_{n},X_{n}'$.
  Define the measurable function $F(f):=F(f)(x_{1},x_{1}',\cdots,x_{n},x_{n}')$ on $\mathbb{X}^{2n}$ as
  \begin{align*}
    F(f):=\frac{1}{n(n-1)}\sum_{i\neq j}k(f(x_{i}),f(x_{j}'))-\mathbb{E}_{X,X^{-}}\left[k(f(X),f(X^{-}))\right].
  \end{align*}
  For simplicity, denote
  \begin{align*}
      F(f)_{x_{\ell}}&:=\frac{1}{n(n-1)}\left(\sum_{j:j\neq \ell}k(f(\widetilde{x}_{\ell}),f(x_{j}'))+\sum_{\substack{i,j:i\neq j\\i\neq \ell}}k(f(x_{i}),f(x_{j}'))\right)-\mathbb{E}_{X,X^{-}}\left[k(f(X),f(X^{-}))\right].
  \end{align*}
  In a similar way, we also use the notation $F(f)_{x_{\ell}'}$.
  Let $j\in[n]$.
  Then, for every $f\in\mathcal{F}$, we have
  \begin{align*}
    F(f)-\sup_{f\in\mathcal{F}}F(f)_{X_{j}}
    \leq F(f)-F(f)_{X_{j}}
    &\leq \left|F(f)-F(f)_{X_{j}}\right|\\
    &\leq \left|\frac{1}{n(n-1)}\sum_{i\in[n],i\neq j}\left(k(f(X_{j}),f(X_{i}'))-k(f(\widetilde{X}_{j}),f(X_{i}'))\right)\right|\\
    &\leq \frac{1}{n(n-1)}\cdot 2(n-1)b= \frac{2b}{n},
  \end{align*}
  where $b:=\sup_{z,z'\in\mathbb{S}^{d-1}}|k(z,z')|$.
  Hence, $\sup_{f\in\mathcal{F}}F(f)-\sup_{f\in\mathcal{F}}F(f)_{X_{j}}\leq \frac{2b}{n}$.
  By applying the same argument several times, $\sup_{f\in\mathcal{F}}F(f)$ satisfies the assumption of Theorem~\ref{thm:ulln for dependent rvs}.
  Therefore, from Theorem~\ref{thm:ulln for dependent rvs} (i.e., Theorem 3.6 in~\citet{zhang2019mcdiarmid}) and Lemma~\ref{lem:lambda g n}, with probability at least $1-\varepsilon$ we have
  \begin{align}
    \label{lemma3ineq0.9}
    \sup_{f\in\mathcal{F}}F(f)\leq \mathbb{E}\left[\sup_{f\in\mathcal{F}}F(f)\right]+\sqrt{\frac{10b^{2}\log\left(1/\varepsilon\right)}{n}}.
  \end{align}
  Let $\sigma_{1:n/2}:=(\sigma_{1},\cdots,\sigma_{n/2})$ be a random vector that consists of a Rademacher random variable (i.e., a random variable taking $\pm$ 1 with probability $1/2$ each) for each entry, and let $\widetilde{X}_{1:n},\widetilde{X}_{1:n}'$ be i.i.d. copies of the random vectors $X_{1:n},X_{1:n}'$, respectively.
  Denote $m=n/2\in\mathbb{N}$.
  Then,
  \begin{align}
    \label{lemma3ineq0.901}
    &\;\;\;\;\mathbb{E}\left[\sup_{f\in\mathcal{F}}F(f)\right]\nonumber\\
    &=
    \mathbb{E}\left[\sup_{f\in\mathcal{F}}\left(\frac{1}{n(n-1)}\sum_{i\neq j}k(f(X_{i}),f(X_{j}'))-\mathbb{E}_{X,X^{-}}\left[k(f(X),f(X^{-}))\right]\right)\right]\nonumber\\
    &=
    \mathbb{E}\left[\sup_{f\in\mathcal{F}}\left(\frac{1}{n(n-1)}\sum_{i\neq j}k(f(X_{i}),f(X_{j}'))-\mathbb{E}_{\widetilde{X}_{1:n},\widetilde{X}_{1:n}'}\left[\frac{1}{n(n-1)}\sum_{i\neq j}k(f(\widetilde{X}_{i}),f(\widetilde{X}_{j}'))\right]\right)\right]\nonumber\\
    &=\mathbb{E}\left[\sup_{f\in\mathcal{F}}\left(\frac{1}{n!m}\sum_{s\in S_{n}}\left(\sum_{i=1}^{m}k(f(X_{s(2i-1)}),f(X_{s(2i)}'))-\mathbb{E}_{\widetilde{X}_{1:n}',\widetilde{X}_{1:n}'}\left[\sum_{i=1}^{m}k(f(\widetilde{X}_{s(2i-1)}),f(\widetilde{X}_{s(2i)}'))\right]\right)\right)\right]\\
    &\leq \frac{1}{n!}\sum_{s\in S_{n}}\mathbb{E}\left[\sup_{f\in\mathcal{F}}\left(\frac{1}{m}\sum_{i=1}^{m}k(f(X_{s(2i-1)}),f(X_{s(2i)}'))-\mathbb{E}_{\widetilde{X}_{1:n},\widetilde{X}_{1:n}'}\left[\frac{1}{m}\sum_{i=1}^{m}k(f(\widetilde{X}_{s(2i-1)}),f(\widetilde{X}_{s(2i)}'))\right]\right)\right]\nonumber\\
    &\leq \frac{1}{n!}\sum_{s\in S_{n}}\mathbb{E}_{\substack{X_{1:n},X_{1:n}'\\\widetilde{X}_{1:n},\widetilde{X}_{1:n}'}}\left[\sup_{f\in\mathcal{F}}\left(\frac{1}{m}\sum_{i=1}^{m}\left(k(f(X_{s(2i-1)}),f(X_{s(2i)}'))-k(f(\widetilde{X}_{s(2i-1)}),f(\widetilde{X}_{s(2i)}'))\right)\right)\right]\nonumber\\
    \label{lemma3ineq0.902}
    &= \frac{1}{n!}\sum_{s\in S_{n}}\mathbb{E}_{\substack{X_{1:n},X_{1:n}'\\\widetilde{X}_{1:n},\widetilde{X}_{1:n}'\\\sigma_{1:m}}}\left[\sup_{f\in\mathcal{F}}\left(\frac{1}{m}\sum_{i=1}^{m}\sigma_{i}\left(k(f(X_{s(2i-1)}),f(X_{s(2i)}'))-k(f(\widetilde{X}_{s(2i-1)}),f(\widetilde{X}_{s(2i)}'))\right)\right)\right]\\
    &\leq \frac{2}{n!}\sum_{s\in S_{n}}\mathbb{E}_{\substack{X_{1:n},X_{1:n}'\\\sigma_{1:m}}}\left[\sup_{f\in\mathcal{F}}\frac{1}{m}\sum_{i=1}^{m}\sigma_{i}k(f(X_{s(2i-1)}),f(X_{s(2i)}'))\right]\nonumber\\
    \label{lemma3ineq0.903}
    &\leq \frac{2\rho}{n!}\sum_{s\in S_{n}}\mathbb{E}_{\substack{X_{1:n},X_{1:n}'\\\sigma_{1:m}}}\left[\sup_{f\in\mathcal{F}}\frac{1}{m}\sum_{i=1}^{m}\sigma_{i}f(X_{s(2i-1)})^{\top}f(X_{s(2i)}')\right]\\
    &\leq 2\rho\max_{s\in S_{n}}\mathbb{E}_{\substack{X_{1:n},X_{1:n}'\\\sigma_{1:m}}}\left[\sup_{f\in\mathcal{F}}\frac{1}{m}\sum_{i=1}^{m}\sigma_{i}f(X_{s(2i-1)})^{\top}f(X_{s(2i)}')\right]\nonumber\\
    &=2\rho\mathfrak{R}_{m}^{-}(\mathcal{Q};s^{*})\nonumber,
  \end{align}
  where in \eqref{lemma3ineq0.901} we define $S_{n}$ as the symmetric group of degree $n$ (see Remark~\ref{remark:relation to clemencon} for the relation to the \textit{average of "sums-of-i.i.d." blocks} technique for $U$-statistics which is explained in~\citet{clemencon2008ranking}).
  Besides in \eqref{lemma3ineq0.902}, for every $s\in S_{n}$ the random vectors $(X_{s(2i-1)},X_{s(2i)})$, $(\widetilde{X}_{s(2i-1)},\widetilde{X}_{s(2i)})$ for $i=1,\cdots,m$ are independent and identically distributed, which implies that the standard symmetrization argument~(Theorem~4.10 of~\citet{wainwright2019high}) is applicable.
  Finally, in \eqref{lemma3ineq0.903}, under Assumption~\ref{assumption:kernel}, we apply Talagrand's lemma~(Lemma~26.9 in~\citet{shalev2014understanding}).
  Therefore, we obtain with probability at least $1-\varepsilon$,
  \begin{align*}
      \sup_{f\in\mathcal{F}}F(f)\leq 2\rho\mathfrak{R}_{n/2}^{-}(\mathcal{Q};s^{*})+\sqrt{\frac{10b^{2}\log\left(1/\varepsilon\right)}{n}}.
  \end{align*}
  Thus, we obtain the claim.
\end{proof}

\begin{remark}
\label{remark:relation to clemencon}
In \eqref{lemma3ineq0.901} of the proof of Theorem~\ref{thm:dependent ulln}, we use the identity,
\begin{align*}
    \frac{1}{n(n-1)}\sum_{i\neq j}k(f(X_{i}),f(X_{j}'))=\frac{1}{n!m}\sum_{s\in S_{n}}\sum_{i=1}^{m}k(f(X_{s(2i-1)},f(X_{s(2i)}')).
\end{align*}
We notice that the above identity is closely related to the \textit{average of "sum-of-i.i.d." blocks} technique explained in Appendix~A of~\citet{clemencon2008ranking}.
As well as the technique presented in~\citet{clemencon2008ranking}, in \eqref{lemma3ineq0.901} of our paper we also decompose the sum $\sum_{i\neq j}k(f(X_{i}),f(X_{j}'))$ into the sums of the i.i.d. random variables.
However, we remark that the definition of the sum $\sum_{i\neq j}k(f(X_{i}),f(X_{j}'))$ is different from that presented in~\citet{clemencon2008ranking}: indeed, in our case, the random variables $f(X_{1}),f(X_{1}'),\cdots,f(X_{n}),f(X_{n}')$ are not necessarily independent of each other.
To address this problem, we decompose our sum in \eqref{lemma3ineq0.901} as follows: for $2n$ random variables $X_{1},X_{1}',\cdots,X_{n},X_{n}'$, we create the tuples $(X_{s(1)},X_{s(2)}',\cdots,X_{s(n-1)},X_{s(n)}')$ where $s\in S_{n}$, then sum up all the components $\{\sum_{i=1}^{n}k(f(X_{s(2i-1)}),f(X_{s(2i)}'))\}_{s\in S_{n}}$.
\end{remark}

\section{Proof in Section~\ref{subsec:main result section 5.4}}
\label{appsec:surrogate bound}

\begin{proof}[Proof of Theorem~\ref{cor:surrogate bound}]
First applying Theorem~\ref{thm:guarantee} to the empirical loss minimizer $\widehat{f}$, we have
\begin{align}
\label{eq:cor eq1}
    L_{\textup{Err}}(\widehat{f},W_{\mu},\beta_{\mu};y)\leq\frac{8(K-1)}{\Delta_{\textup{min}}(\widehat{f})\cdot\min_{i\in[K]}P_{\mathbb{X}}(\mathbb{M}_{i})}\mathfrak{a}(\widehat{f}).
\end{align}
Using Theorem~\ref{thm:decomposition}, we have the inequality,
\begin{align}
\label{eq:cor eq2}
    \mathfrak{a}(\widehat{f})\leq L_{\textup{KCL}}(\widehat{f};\lambda)+(1-\frac{\delta}{2})\mathfrak{a}(\widehat{f})-\lambda\mathfrak{c}(\widehat{f})+R(\lambda).
\end{align}
Combining \eqref{eq:cor eq1} and \eqref{eq:cor eq2}, we obtain
\begin{align}
\label{eq:cor eq3}
    L_{\textup{Err}}(\widehat{f},W_{\mu},\beta_{\mu};y)\leq\frac{8(K-1)}{\Delta_{\textup{min}}(\widehat{f})\cdot\min_{i\in[K]}P_{\mathbb{X}}(\mathbb{M}_{i})}\left(L_{\textup{KCL}}(\widehat{f};\lambda)+(1-\frac{\delta}{2})\mathfrak{a}(\widehat{f})-\lambda\mathfrak{c}(\widehat{f})+R(\lambda)\right).
\end{align}
Here, using the standard technique for upper bounding the optimal classification loss or error~\citep{pmlr-v97-saunshi19a,ash2022investigating}, the classification error $L_{\textup{Err}}(\widehat{f},W_{\mu},\beta_{\mu};y)$ is lower bounded as
\begin{align}
\label{eq:cor eq4}
    L_{\textup{Err}}(\widehat{f},W^{*}\beta^{*};y)=\inf_{W,\beta}L_{\textup{Err}}(\widehat{f},W,\beta;y)\leq L_{\textup{Err}}(\widehat{f},W_{\mu},\beta_{\mu};y).
\end{align}
From \eqref{eq:cor eq3} and \eqref{eq:cor eq4}, 
\begin{align}
    \label{eq:cor eq5}
    L_{\textup{Err}}(\widehat{f},W^{*},\beta^{*};y)\leq\frac{8(K-1)}{\Delta_{\textup{min}}(\widehat{f})\cdot\min_{i\in[K]}P_{\mathbb{X}}(\mathbb{M}_{i})}\left(L_{\textup{KCL}}(\widehat{f};\lambda)+(1-\frac{\delta}{2})\mathfrak{a}(\widehat{f})-\lambda\mathfrak{c}(\widehat{f})+R(\lambda)\right).
\end{align}
Applying Theorem~\ref{lem:ulln} to \eqref{eq:cor eq5}, we obtain: with probability at least $1-2\varepsilon$,
\begin{align*}
    L_{\textup{Err}}(\widehat{f},W_{\mu},\beta_{\mu};y)\lesssim L_{\textup{KCL}}(f;\lambda)+(1-\frac{\delta}{2})\mathfrak{a}(\widehat{f})-\lambda\mathfrak{c}(\widehat{f})+R(\lambda)+2\textup{Gen}(n,\lambda,\varepsilon),
\end{align*}
where $\lesssim$ omits the coefficient $\frac{8(K-1)}{\Delta_{\textup{min}}(\widehat{f})\cdot\min_{i\in[K]}P_{\mathbb{X}}(\mathbb{M}_{i})}$.
Therefore, we obtain the result.
\end{proof}

\section{Additional Information, Results, and Discussion}
\label{appsec:additional results and disccusions}

\subsection{Examples Satisfying Assumption~\ref{assumption:mixture of clusters}}
\label{appsec:proofs in examples}

\subsubsection{Proofs in Example~\ref{example:when the space meets the assumptions}}
\label{appsubsec:proofs in example}

We show the several claims that appear in Example~\ref{example:when the space meets the assumptions} as a proposition.

\begin{proposition}
Let $r>0$, $K\in\mathbb{N}$, and $v_{1},\cdots,v_{K}\in\mathbb{R}^{p}$.
For each $i\in[K]$, let $\mathbb{B}_{i}\subset\mathbb{R}^{p}$ be the open ball of radius $r$ centered at a point $v_{i}$.
Suppose $\mathbb{B}_{1},\cdots,\mathbb{B}_{K}$ are disjoint to each other. Define $\overline{\mathbb{X}}=\bigcup_{i=1}^{K}\mathbb{B}_{i}$, $\mathbb{X}=\overline{\mathbb{X}}$, and the conditional probability $a(x|\overline{x})=\textup{vol}(\mathbb{B}_{1})^{-1}\sum_{i=1}^{K}\mathbbm{1}_{\mathbb{B}_{i}\times\mathbb{B}_{i}}(x,\overline{x})$, where $\textup{vol}(\mathbb{B}_{1})$ be the volume of $\mathbb{B}_{i}$ in $\mathbb{R}^{p}$.
Let $p_{\overline{\mathbb{X}}}(\overline{x}):=(K\textup{vol}(\mathbb{B}_{1}))^{-1}$ be a probability density function of $P_{\overline{\mathbb{X}}}$.
Define $y:\mathbb{X}\to [K]$ as $y(x)=i$ if $x\in\mathbb{B}_{i}$.
Then, we have the following properties:
\begin{enumerate}
    \item $w(x)>0$ for every $x\in\mathbb{X}$.
    \item $\textup{sim}(x,x';\lambda)=K\mathbbm{1}_{\bigcup_{i\in [K]}\mathbb{B}_{i}\times\mathbb{B}_{i}}(x,x')-\lambda$ for every $x,x'\in\mathbb{X}$.
    \item Let $\delta\in(-\lambda,K-\lambda]$. Then, $\delta$, $K$, $\mathbb{B}_{1},\cdots,\mathbb{B}_{K}$, and y satisfy Assumption~\ref{assumption:mixture of clusters}.
\end{enumerate}
\end{proposition}

\begin{proof}
We first show the claim 1.
From the definition of $w(x)$, for every $x\in\mathbb{B}_{1}$ we have
\begin{align*}
    w(x)=\int_{\overline{\mathbb{X}}}a(x|\overline{x})p_{\overline{\mathbb{X}}}(\overline{x})d\overline{x}=\int_{\mathbb{B}_{1}} \frac{1}{K\left(\textup{vol}(\mathbb{B}_{1})\right)^{2}}d\overline{x}=\frac{1}{K\textup{vol}(\mathbb{B}_{1})}.
\end{align*}
Similarly, for each $i\in[K]$ we obtain $w(x)=(K\textup{vol}(\mathbb{B}_{1}))^{-1}$ for every $x\in\mathbb{B}_{i}$.
Since $\mathbb{X}=\overline{\mathbb{X}}=\bigcup_{i=1}^{K}\mathbb{B}_{i}$, we have that $w(x)>0$ for every $x\in\mathbb{X}$.

Next, let us show the claim 2.
From the claim 1, the function $\textup{sim}(x,x';\lambda)$ is well-defined.
To compute $\textup{sim}(x,x';\lambda)$, we need to know the function $w(x,x')$.
The computation of $w(x,x')$ is done as follows:
\begin{align*}
    w(x,x')&=\int_{\overline{\mathbb{X}}}a(x|\overline{x})a(x'|\overline{x})p_{\overline{\mathbb{X}}}(\overline{x})d\overline{x}\\
    &=\begin{cases}
      \int_{\mathbb{B}_{i}} \frac{1}{K\left(\textup{vol}(\mathbb{B}_{1})\right)^{3}}d\overline{x}\quad\textup{if }x,x'\in\mathbb{B}_{i}\textup{ for some }i\in[K]\\
      0 \quad\quad\quad\quad\quad\quad\quad\;\,\textup{if }x\in\mathbb{B}_{i}\textup{ and }x'\in\mathbb{B}_{j}\textup{ for some }i\neq j
    \end{cases}
    \\
    &=\begin{cases}
      \frac{1}{K\left(\textup{vol}(\mathbb{B}_{1})\right)^{2}}\quad\quad\;\;\quad\textup{if }x,x'\in\mathbb{B}_{i}\textup{ for some }i\in[K]\\
      0 \quad\quad\quad\quad\quad\quad\quad\;\,\textup{if }x\in\mathbb{B}_{i}\textup{ and }x'\in\mathbb{B}_{j}\textup{ for some }i\neq j.
    \end{cases}
\end{align*}
Hence, it is obvious that the claim 2 holds.

Finally, let us prove the claim 3.
However, from the claim 2 we see that $\textup{sim}(x,x';\lambda)\geq \delta$ if and only if $x,x'\in\mathbb{B}_{i}$ for some $i\in[K]$.
Furthermore, $y$ is well-defined and the set $\{x\in\mathbb{X}\;|\;y(x)=i\}=\mathbb{B}_{i}$ is measurable for every $i\in[K]$.
Thus, the claim 3 is also true, and we end the proof.
\end{proof}

\subsubsection{An Example When Clusters Overlap}
Here, we also deal with an example where the clusters in $\mathbb{X}$ have some overlap.
In the following proposition, for the sake of simplicity, we consider the case that there are two clusters in $\mathbb{X}$.

\begin{proposition}
Let $r>0$, and $v_{1},v_{2}\in\mathbb{R}^{p}$.
For each $i\in\{1,2\}$, let $\mathbb{B}(v_{i};r)\subset\mathbb{R}^{p}$ be the open ball of radius $r$ centered at point $v_{i}$.
Suppose that $\|v_{1}-v_{2}\|_{2}=3r$.
Define $\overline{\mathbb{X}}=\mathbb{B}(v_{1};r)\cup\mathbb{B}(v_{2};r)$, $\mathbb{X}=\mathbb{B}(v_{1};2r)\cup\mathbb{B}(v_{2};2r)$, and $a(x|\overline{x})=\textup{vol}(\mathbb{B}(v_{1};2r))^{-1}\sum_{i=1}^{2}\mathbbm{1}_{\mathbb{B}(v_{i};2r)\times\mathbb{B}(v_{i};r)}(x,\overline{x})$.
Let $p_{\overline{\mathbb{X}}}(\overline{x}):=(2\cdot\textup{vol}(\mathbb{B}(v_{1};r)))^{-1}$ be a probability density function of $P_{\overline{\mathbb{X}}}$.
Define $y:\mathbb{X}\to \{1,2\}$ as $y(x)=1$ if $x\in\mathbb{B}(v_{1};2r)$ and $y(x)=2$ if $x\in\mathbb{B}(v_{2};2r)\setminus\mathbb{B}(v_{1};2r)$.
Then, we have the following results:
\begin{enumerate}
    \item $w(x)>0$ for every $x\in\mathbb{X}$.
    \item $\textup{sim}(x,x';\lambda)=2-\lambda$ if $x,x'\in\mathbb{B}(v_{1};2r)\setminus\mathbb{B}(v_{2};2r)$ or $x,x'\in\mathbb{B}(v_{2};2r)\setminus\mathbb{B}(v_{1};2r)$, $\textup{sim}(x,x';\lambda)=-\lambda$ if $(x,x')\in(\mathbb{B}(v_{1};2r)\setminus\mathbb{B}(v_{2};2r))\times(\mathbb{B}(v_{2};2r)\setminus\mathbb{B}(v_{1};2r))$ or $(x,x')\in(\mathbb{B}(v_{2};2r)\setminus\mathbb{B}(v_{1};2r))\times(\mathbb{B}(v_{1};2r)\setminus\mathbb{B}(v_{2};2r))$, and $\textup{sim}(x,x';\lambda)=1-\lambda$ otherwise.
    \item Let $\delta\in(-\lambda,1-\lambda]$. Then, $\delta$, $K$, $\mathbb{B}_{1},\cdots,\mathbb{B}_{K}$, and y satisfy Assumption~\ref{assumption:mixture of clusters}.
\end{enumerate}
\end{proposition}

\begin{proof}
Let $\overline{\mathbb{X}}_1$ (resp. $\overline{\mathbb{X}}_{2}$) denote $\mathbb{B}(v_{1};r)$, (resp. $\mathbb{B}(v_{2};r)$). Then, 
\begin{align*}
        w(x)&=\mathbb{E}\left[a(x|\overline{x})\right]\\
        &=\int_{\overline{\mathbb{X}}}a(x|\overline{x})p_{\overline{\mathbb{X}}} (\overline{x}) d\overline{x}\\
        &=\int_{\overline{\mathbb{X}}_{1}}a(x|\overline{x})p_{\overline{\mathbb{X}}} (\overline{x}) d\overline{x}+\int_{\overline{\mathbb{X}}_{2}}a(x|\overline{x})p_{\overline{\mathbb{X}}} (\overline{x}) d\overline{x}\\
        &=p_{\overline{\mathbb{X}}} (\overline{x})\left\{  \int_{\overline{\mathbb{X}}_{1}}a(x|\overline{x})d\overline{x}+\int_{\overline{\mathbb{X}}_{2}}a(x|\overline{x}) d\overline{x}   \right\}\tag{$p_{\overline{\mathbb{X}}}$ is a constant function}
\end{align*}
Here, we consider Case~1 and Case~2. Firstly, Case~1 is when either $x \in \mathbb{B}(v_1;2r) \setminus \mathbb{B}(v_2;2r)$ or $x \in \mathbb{B}(v_2;2r) \setminus \mathbb{B}(v_1;2r)$ holds. Since in this case, it is sufficient to prove for the case that $x \in \mathbb{B}(v_1;2r) \setminus \mathbb{B}(v_2;2r)$ holds, we may assume this condition. Then, $\int_{\overline{\mathbb{X}}_{1}}a(x|\overline{x})d\overline{x} = \text{vol}(\mathbb{B}(v_1;r))\text{vol}(\mathbb{B}(v_1;2r))^{-1}$  and $\int_{\overline{\mathbb{X}}_2}a(x|\overline{x})d\overline{x} = 0$. Thus, $w(x) = 1/(2\text{vol}(\mathbb{B}(v_1;2r)))$. Secondly, Case~2 is when $x \in \mathbb{B}(v_1;2r) \cap \mathbb{B}(v_2;2r)$. Then, $\int_{\overline{\mathbb{X}}_1}a(x|\overline{x})d\overline{x} = \int_{\overline{\mathbb{X}}_2}a(x|\overline{x})d\overline{x} = \text{vol}(\mathbb{B}(v_1;r))\text{vol}(\mathbb{B}(v_1;2r))^{-1}$. Thus, $w(x) = 1/\text{vol}(\mathbb{B}(v_1;2r))$. Since $r>0$, it implies $\text{vol}(\mathbb{B}(v_1;2r))>0$. Thus $w(x)>0$ for both cases.

Next, we compute
\begin{align*}
        w(x, x')&=\mathbb{E}_{\overline{x}}\left[a(x|\overline{x})a(x'|\overline{x})\right]\\
        &=\int_{\overline{\mathbb{X}}}a(x|\overline{x})a(x'|\overline{x})p_{\overline{\mathbb{X}}} (\overline{x}) d\overline{x}\\
        &=\int_{\overline{\mathbb{X}}_1}a(x|\overline{x})a(x'|\overline{x})p_{\overline{\mathbb{X}}} (\overline{x}) d\overline{x}\ + \int_{\overline{\mathbb{X}}_2}a(x|\overline{x})a(x'|\overline{x})p_{\overline{\mathbb{X}}} (\overline{x}) d\overline{x}\\
        &=p_{\overline{\mathbb{X}}} (\overline{x})\left\{  \int_{\overline{\mathbb{X}}_1}a(x|\overline{x}) a(x'|\overline{x})d\overline{x}+\int_{\overline{\mathbb{X}}_2}a(x|\overline{x})a(x'|\overline{x}) d\overline{x}   \right\}\tag{$p_{\overline{\mathbb{X}}}$ is a constant function}.
\end{align*}
Here, we consider Case~A, Case~B, Case~C, and Case~D. Firstly Case~A is that both $x$ and $x^\prime$ belong to $\mathbb{B}(v_1;2r)\setminus \mathbb{B}(v_2;2r)$ (note that the computation for the case that both $x$ and $x^\prime$ belong to $\mathbb{B}(v_2;2r)\setminus \mathbb{B}(v_1;2r)$ is the same). Then, $\int_{\overline{\mathbb{X}}_1}a(x|\overline{x}) a(x'|\overline{x})d\overline{x} = \text{vol}(\mathbb{B}(v_1;r))/\text{vol}(\mathbb{B}(v_1;2r))^2$ and $\int_{\overline{\mathbb{X}}_2}a(x|\overline{x}) a(x'|\overline{x})d\overline{x} = 0$. Hence, 
$w(x, x') = \{2(\text{vol}(\mathbb{B}(v_1;2r)))^2\}^{-1}$. Here recall that $w(x)=w(x')=1/(2\text{vol}(\mathbb{B}(v_1;2r)))$, then we have $\text{sim}(x, x';\lambda) = 2 - \lambda$. Secondly Case~B is that $x \in \mathbb{B}(v_1;2r)\setminus\mathbb{B}(v_{2};2r)$ and $x' \in \mathbb{B}(v_2;2r) \setminus \mathbb{B}(v_1;2r)$ (the calculation for the case that $x \in \mathbb{B}(v_2;2r)\setminus\mathbb{B}(v_{1};2r)$ and $x' \in \mathbb{B}(v_1;2r) \setminus \mathbb{B}(v_2;2r)$ is the same). Then, $\int_{\overline{\mathbb{X}}_1}a(x|\overline{x}) a(x'|\overline{x})d\overline{x}=\int_{\overline{\mathbb{X}}_2}a(x|\overline{x}) a(x'|\overline{x})d\overline{x}=0$. Therefore, $\text{sim}(x, x';\lambda) = 0 - \lambda = -\lambda$. Thirdly Case~C is that both $x$ and $x^\prime$ belong to $\mathbb{B}(v_1;2r) \cap \mathbb{B}(v_2;2r)$. Then, $\int_{\overline{\mathbb{X}}_1}a(x|\overline{x}) a(x'|\overline{x})d\overline{x} =\int_{\overline{\mathbb{X}}_2}a(x|\overline{x}) a(x'|\overline{x})d\overline{x} = \text{vol}(\mathbb{B}(v_1;r))/\text{vol}(\mathbb{B}(v_1;2r))^2$. Since $w(x) = w(x') = 1/\text{vol}(\mathbb{B}(v_1;2r))$, $\text{sim}(x, x';\lambda) = 1 - \lambda$.
Finally in Case~D, consider the complementary of the union of the other cases.
From the setting, we may assume that $x$ belongs to $\mathbb{B}(v_{1};2r)\cap\mathbb{B}(v_{2};2r)$ and $x'$ to $\mathbb{B}(v_{1};2r)\setminus\mathbb{B}(v_{2};2r)$.
Then, $\int_{\overline{\mathbb{X}}_{1}}a(x|\overline{x})a(x'|\overline{x})d\overline{x}=\textup{vol}(\mathbb{B}(v_{1};r))/\textup{vol}(\mathbb{B}(v_{1};2r))^{2}$ and $\int_{\overline{\mathbb{X}}_{2}}a(x|\overline{x})a(x'|\overline{x})d\overline{x}=0$.
Since $w(x)=1/\textup{vol}(\mathbb{B}(v_{1};2r))$ and $w(x')=1/(2\textup{vol}(\mathbb{B}(v_{1};2r)))$, we have $\textup{sim}(x,x';\lambda)=1-\lambda$.
As a result,
\begin{equation*}
    \text{sim}(x,x';\lambda)=\left\{ 
  \begin{array}{ c l }
    2 - \lambda,& \text{if Case~A holds}, \\
    -\lambda,   & \text{if Case~B holds}, \\
    1 - \lambda, & \text{if Case~C holds}, \\
    1 - \lambda, & \text{if Case~D holds}.
  \end{array}
  \right.
\end{equation*}

Finally, take $\delta\in(-\lambda,1-\lambda]$. Then, from the computation for $\textup{sim}(x,x';\lambda)$ above, the conditions in Assumption~\ref{assumption:mixture of clusters} are satisfied. 
\end{proof}

\subsection{SSL-HSIC Revisit}
\label{appsubsubsec:relations to ssl-hsic}
\citet{li2021selfsupervised} propose the framework termed SSL-HSIC, which is defined using the notion Hilbert-Schmidt Independence Criterion (HSIC, \citep{smola2007hilbert}).
They show that under some conditions, for a random variable $Z$ (resp. $Y$) that represents the feature vector (resp. the label), one obtains
\begin{align*}
    \textup{HSIC}(Z,Y)=c\left( \mathbb{E}_{x,x^{+}}\left[k(f(x),f(x^{+}))\right]-\mathbb{E}_{x,x^{-}}\left[k(f(x),f(x^{-}))\right]\right),
\end{align*}
where $c>0$.
\citet{li2021selfsupervised} define the loss of SSL-HSIC as,
\begin{align*}
    L_{\textup{SSL-HSIC}}(f;\kappa)=-\textup{HSIC}(Z,Y)+\kappa \sqrt{\textup{HSIC}(Z,Z)},
\end{align*}
where $\kappa\in\mathbb{R}$.

In the case that $\kappa>0$, we have
\begin{align*}
    L_{\textup{KCL}}(f;1)\lesssim L_{\textup{SSL-HSIC}}(f;\kappa).
\end{align*}

\subsection{Supplementary Information of Section~\ref{subsec:introduction to kernel contrastive loss}}
\label{appsubsec:supplementary information of section 3.2}

\subsubsection{Relations to Variants of InfoNCE}
\label{appsubsubsec:lin vs variant nce}

We first define variants of InfoNCE~\citep{oord2018representation,chen2020simple}:
\begin{itemize}
    \item Decoupled InfoNCE loss, which is a variant of the decoupled NT-Xent loss of~\citet{chen2021intriguing}:
    \begin{align*}
        \widetilde{L}_{\textup{NCE}}(f;\tau,\lambda)=
        -\mathbb{E}_{x,x^{+}}\left[\frac{f(x)^{\top}f(x^{+})}{\tau}\right]+\lambda
        \mathbb{E}_{\substack{x,x^{+}\\\{x_{i}^{-}\}}}\left[\log\left(e^{\frac{f(x)^{\top}f(x^{+})}{\tau}}+\sum_{i=1}^{M}e^{\frac{f(x)^{\top}f(x_{i}^{-})}{\tau}}\right)\right].
    \end{align*}
    \item Asymptotic of contrastive loss~\citep{wang2020understanding} decoupled by following the way of~\citet{chen2021intriguing}:
    \begin{align*}
        \widetilde{L}_{\infty\textup{-NCE}}(f;\tau,\lambda)=-\mathbb{E}_{x,x^{+}}\left[\frac{f(x)^{\top}f(x^{+})}{\tau}\right]+\lambda\mathbb{E}_{x}\left[\log\mathbb{E}_{x'}\left[e^{\frac{f(x)^{\top}f(x')}{\tau}}\right]\right],
    \end{align*}
    \item InfoNCE loss as a variant of decoupled contrastive learning loss~\citep{yeh2021decoupled}:
    \begin{align*}
        \widetilde{L}_{\textup{NCE}}(f;\tau,1)=
        -\mathbb{E}_{x,x^{+}}\left[\frac{f(x)^{\top}f(x^{+})}{\tau}\right]+
        \mathbb{E}_{x,\{x_{i}^{-}\}}\left[\log\left(\sum_{i=1}^{M}e^{\frac{f(x)^{\top}f(x_{i}^{-})}{\tau}}\right)\right].
    \end{align*}
    \item InfoNCE loss as a variant of decoupled contrastive learning loss with additional weight parameter, following~\citet{chen2021intriguing}:
    \begin{align*}
        \widetilde{L}_{\textup{NCE}}(f;\tau,\lambda)=
        -\mathbb{E}_{x,x^{+}}\left[\frac{f(x)^{\top}f(x^{+})}{\tau}\right]+\lambda
        \mathbb{E}_{x,\{x_{i}^{-}\}}\left[\log\left(\sum_{i=1}^{M}e^{\frac{f(x)^{\top}f(x_{i}^{-})}{\tau}}\right)\right].
    \end{align*}
\end{itemize}

Note that $L_{\textup{NCE}}(f;\tau)$ and $L_{\infty\textup{-NCE}}(f;\tau)$ in Section~\ref{sec:kcl} coincide with $\widetilde{L}_{\textup{NCE}}(f;\tau,1)$ and $\widetilde{L}_{\infty\textup{-NCE}}(f;\tau,1)$ in this subsection, respectively.
We show the following facts:

\begin{proposition}
\label{prop:relations to nce variants}
The following relations hold:
\begin{align}
\label{eq:lkcl vs dnce}
    \tau^{-1}L_{\textup{LinKCL}}(f;\lambda)&\leq \widetilde{L}_{\textup{NCE}}(f;\tau,\lambda)+\lambda\log M^{-1},\\
    \label{eq:lkcl vs dance}
    \tau^{-1}L_{\textup{LinKCL}}(f;\lambda)&\leq \widetilde{L}_{\infty\textup{-NCE}}(f;\tau,\lambda),\\
    \label{eq:lkcl vs ddcllnce}
    \tau^{-1}L_{\textup{LinKCL}}(f;\lambda)&\leq \widetilde{L}_{\textup{NCE}}(f;\tau,\lambda)+\lambda\log M^{-1}.
\end{align}
\end{proposition}

\begin{proof}
From the definition of $L_{\textup{NCE}}(f;\tau,\lambda)$, we have
\begin{align*}
    &\;\;\;\;\widetilde{L}_{\textup{NCE}}(f;\tau,\lambda)+\lambda\log \frac{1}{M}\\
    &= -\tau^{-1}\mathbb{E}_{x,x^{+}}\left[f(x)^{\top}f(x^{+})\right]+\lambda\mathbb{E}_{x,x^{+},\{x_{i}^{-}\}}\left[\log \left(\frac{1}{M}e^{f(x)^{\top}f(x^{+})/\tau}+\frac{1}{M}\sum_{i=1}^{M}e^{f(x)^{\top}f(x_{i}^{-})/\tau}\right)\right]\\
    &\geq -\tau^{-1}\mathbb{E}_{x,x^{+}}\left[f(x)^{\top}f(x^{+})\right]+\lambda\mathbb{E}_{x,\{x_{i}^{-}\}}\left[\log \left(\frac{1}{M}\sum_{i=1}^{M}e^{f(x)^{\top}f(x_{i}^{-})/\tau}\right)\right]\\
    &\geq -\tau^{-1}\mathbb{E}_{x,x^{+}}\left[f(x)^{\top}f(x^{+})\right]+\tau^{-1}\lambda\mathbb{E}_{x,\{x_{i}^{-}\}}\left[\frac{1}{M}\sum_{i=1}^{M}f(x)^{\top}f(x_{i}^{-})\right]\\
    &=\tau^{-1}L_{\textup{LinKCL}}(f;\lambda),
\end{align*}
where in the first inequality we use the fact that $M^{-1}e^{f(x)^{\top}f(x^{+})/\tau}\geq 0$ for any $x,x^{+}\in\mathbb{X}$, and in the second inequality we use Jensen's inequality.
Note that when $\lambda=1$, we obtain \eqref{eq:lin vs nce}.

The proofs of \eqref{eq:lkcl vs ddcllnce} are almost the same as the proof of \eqref{eq:lkcl vs dnce}.
The equation \eqref{eq:lkcl vs dance} is obtained by applying Jensen's inequality.
\end{proof}

\subsubsection{Relations to SCL}
\label{appsubsubsec:relations to scl}
Let us define the quadratic kernel contrastive loss as:
\begin{align*}
    L_{\textup{QKCL}}(f;\lambda)=-\mathbb{E}_{x,x^{+}}\left[\left(f(x)^{\top}f(x^{+})\right)^{2}\right]+\lambda\mathbb{E}_{x,x^{-}}\left[\left(f(x)^{\top}f(x^{-})\right)^{2}\right].
\end{align*}
The spectral contrastive loss $L_{\textup{SCL}}(f)$~\citep{haochen2021provable} is defined as,
\begin{align}
\label{def:scl}
    L_{\textup{SCL}}(f)=-2\mathbb{E}_{x,x^{+}}[f(x)^{\top}f(x^{+})]+\mathbb{E}_{x,x^{-}}[(f(x)^{\top}f(x^{-}))^{2}].
\end{align}
The following proposition is an elementary result.

\begin{proposition}
\label{prop:pkcl and scl}
We have,
\begin{align*}
L_{\textup{QKCL}}(f;2^{-1})\leq \frac{1}{2}L_{\textup{SCL}}(f)+\frac{1}{4}.
\end{align*}
\end{proposition}

\begin{proof}
Since $t^{2}+1/4\geq t$ for every $t\in\mathbb{R}$, we obtain the claim.
\end{proof}

\subsection{Comparison of Assumption~\ref{assumption:mixture of clusters} of our work to Assumption~3 in~\citet{haochen2022theoretical}}
\label{appsubsec:relation to haochen ma}

Let $\mathbb{M}$ be a measurable subset of $\mathbb{X}$, and let $g:\mathbb{X}\to\mathbb{R}$ be a function.
\citet{haochen2022theoretical} introduce the following notion that quantifies the inner-connectivity of clusters (see (4) in~\citet{haochen2022theoretical}):
\begin{align*}
    Q_{\mathbb{M}}(g):=\frac{\mathbb{E}_{x,x^{+}}[(g(x)-g(x^{+}))^{2}|\mathbb{M}\times\mathbb{M}]}{\mathbb{E}_{x,x^{-}}[(g(x)-g(x^{-}))^{2}|\mathbb{M}\times\mathbb{M}]}.
\end{align*}
Here, the expectations above are defined as
\begin{align*}
    \mathbb{E}_{x,x^{+}}[(g(x)-g(x^{+}))^{2}|\mathbb{M}\times\mathbb{M}]&=\int_{\mathbb{X}\times\mathbb{X}}(g(x)-g(x'))^{2}P_{+}(dx,dx'|\mathbb{M}\times\mathbb{M}),\\
    \mathbb{E}_{x,x^{-}}[(g(x)-g(x^{-}))^{2}|\mathbb{M}\times\mathbb{M}]&=\int_{\mathbb{X}}\int_{\mathbb{X}}(g(x)-g(x'))^{2}P_{\mathbb{X}}(dx|\mathbb{M})P_{\mathbb{X}}(dx'|\mathbb{M}),
\end{align*}
where we use the notation $dP_{+}=w(x,x')d\nu_{\mathbb{X}}^{\otimes 2}$.
We focus on the following notion, where~\citet{haochen2022theoretical} denote their subsets by $\{S_{1},\cdots,S_{m}\}$.
\begin{assumption}[$\mathcal{F}$-implementable inner-cluster connection larger than $\beta$, quoted from Assumption~3 in~\citet{haochen2022theoretical}]
For any function $f\in\mathcal{F}$ and any linear head $w\in\mathbb{R}^{k}$, let function $g(x)=w^{\top}f(x)$.
For any $i\in[m]$ we have that:
\begin{align*}
    Q_{S_{i}}(g)\geq \beta.
\end{align*}
\end{assumption}
In summary, the relation between Assumption~3 of~\citet{haochen2022theoretical} and Assumption~\ref{assumption:mixture of clusters} of our work is given below:
\begin{proposition}
\label{prop:implementable}
Suppose that Assumption~\ref{assumption:mixture of clusters} holds.
Take $\delta\in\mathbb{R}$, $K\in\mathbb{N}$, and $\mathbb{M}_{1},\cdots,\mathbb{M}_{K}$ such that the conditions \textup{\textbf{(A)}} and \textup{\textbf{(B)}} are satisfied.
Suppose also that: 1) there exists some $c>0$ such that for every $i\in[K]$, $c\cdot P_{+}(\mathbb{M}_{i}\times\mathbb{M}_{i})\leq P_{\mathbb{X}}(\mathbb{M}_{i})^{2}$ holds; 2) $\delta+\lambda\geq 0$ holds.
Then, the function class $\widetilde{\mathcal{F}}$ including all the maps from $\mathbb{X}$ to $\mathbb{S}^{d-1}$ satisfies Assumption~3 in~\citet{haochen2022theoretical}.
\end{proposition}
\begin{proof}
Take an arbitrary $f\in\widetilde{\mathcal{F}}$ and $w\in\mathbb{R}^{d}$.
For each $i\in[K]$, for any $x,x'\in\mathbb{M}_{i}$ we have that $w(x,x')\geq (\delta+\lambda)w(x)w(x')$.
Since $\delta+\lambda\geq 0$,
\begin{align*}
    \int_{\mathbb{M}_{i}\times\mathbb{M}_{i}}(g(x)-g(x'))^{2}w(x,x')\nu_{\mathbb{X}}^{\otimes 2}(dx,dx')\geq (\delta+\lambda)\int_{\mathbb{M}_{i}\times\mathbb{M}_{i}}(g(x)-g(x'))^{2}w(x)w(x')\nu_{\mathbb{X}}^{\otimes 2}(dx,dx').
\end{align*}
Here, using $c\cdot P_{+}(\mathbb{M}_{i}\times\mathbb{M}_{i})\leq P_{\mathbb{X}}(\mathbb{M}_{i})^{2}$, we have
\begin{align*}
    &\;\;\;\;\frac{1}{P_{+}(\mathbb{M}_{i}\times\mathbb{M}_{i})}\int_{\mathbb{M}_{i}\times\mathbb{M}_{i}}(g(x)-g(x'))^{2}w(x,x')\nu_{\mathbb{X}}^{\otimes 2}(dx,dx')\\
    &\geq c(\delta+\lambda)\frac{1}{P_{\mathbb{X}}(\mathbb{M}_{i})^{2}}\int_{\mathbb{M}_{i}\times\mathbb{M}_{i}}(g(x)-g(x'))^{2}w(x)w(x')\nu_{\mathbb{X}}^{\otimes 2}(dx,dx').
\end{align*}
The above inequality means that,
\begin{align*}
    \int_{\mathbb{X}\times\mathbb{X}}(g(x)-g(x'))^{2}P_{+}(dx,dx'|\mathbb{M}_{i}\times\mathbb{M}_{i})\geq c(\delta+\lambda)\int_{\mathbb{X}}\int_{\mathbb{X}}(g(x)-g(x'))^{2}P_{\mathbb{X}}(dx|\mathbb{M}_{i})P_{\mathbb{X}}(dx'|\mathbb{M}_{i}).
\end{align*}
Thus, we obtain $Q_{\mathbb{M}_{i}}(g)\geq c(\delta+\lambda)$, where $g(x)=w^{\top}f(x)$.
\end{proof}
The above proposition indicates that the inner-connectivity $Q_{\mathbb{M}_{i}}(g)$ for any $i\in[K]$ and $g(x)=w^{\top}f(x)$ with $f\in\widetilde{\mathcal{F}}$ is lower bounded by $c(\delta+\lambda)$ under the assumptions.
Therefore, Assumption~\ref{assumption:mixture of clusters} of our work is a sufficient condition of Assumption~3 in~\citet{haochen2022theoretical} if $c\cdot P_{+}(\mathbb{M}_{i}\times\mathbb{M}_{i})\leq P_{\mathbb{X}}(\mathbb{M}_{i})^{2}$ (for every $i\in[K]$) and $\delta+\lambda\geq 0$ hold.
\begin{remark}
We can construct a positive value $c$ in the above statement explicitly.
In this remark, we show a simple way to do so.
Let $X,Y$ be random variables on a probability space $(\Omega,P_{\Omega})$ with the joint probability distribution $P_{+}$ and the marginal distribution $P_{\mathbb{X}}$.
Denote $p_{1}=P_{\Omega}(X\in\mathbb{M}_{i})$ and $p=P_{\Omega}(X\in\mathbb{M}_{i},Y\in\mathbb{M}_{i})$.
Here, let $V$ be the covariance matrix of the random variables $\mathbbm{1}_{\{X\in\mathbb{M}_{i}\}}$ and $\mathbbm{1}_{\{Y\in\mathbb{M}_{i}\}}$.
The positive semi-definiteness of $V$ implies the inequality $(p_{1}-p_{1}^{2})^{2}-(p-p_{1}^{2})^{2}\geq 0$.
This inequality is valid when $2p_{1}^{2}-p_{1}\leq p\leq p_{1}$.
Combining this fact with the property $p\geq 0$, we obtain
\begin{align*}
    \max\{2p_{1}^{2}-p_{1},0\}\leq p\leq p_{1},
\end{align*}
which implies,
\begin{align*}
    P_{\mathbb{X}}(\mathbb{M}_{i})\cdot\max\{2P_{\mathbb{X}}(\mathbb{M}_{i})-1,0\}\cdot P_{+}(\mathbb{M}_{i}\times\mathbb{M}_{i})\leq P_{\mathbb{X}}(\mathbb{M}_{i})^{2}.
\end{align*}
Thus, if $P_{\mathbb{X}}(\mathbb{M}_{i})>1/2$ holds for every $i\in [K]$, then we can take 
\begin{align*}
    c=\min_{i\in [K]}\{P_{\mathbb{X}}(\mathbb{M}_{i})(2P_{\mathbb{X}}(\mathbb{M}_{i})-1)\}.
\end{align*}
\end{remark}

\subsection{More Discussion about Generalization Bounds in Section~\ref{subsec:rethinking generalization of contrastive learning}}
\label{appsubsec:discussion of generalization bounds}

In this section, we discuss the differences between our generalization error bound and the results presented by other works on contrastive learning.
We summarize the differences below.

\begin{itemize}
    \item \citet{pmlr-v97-saunshi19a,ash2022investigating,lei2023generalization,zou2023generalization} consider the case that for a pair $(x,x^{+})$, $M$-tuple samples $(x_{1}^{-},\cdots,x_{M}^{-})$ independent from other random variables are available. Thus, our problem setup is different from them.
    Especially, in our analysis, it is also necessary to tackle the cases in which $X_{i},X_{i}'$, $i\in[n]$, are not necessarily independent and the standard techniques (e.g., see~\citet{mohri2018foundations}) cannot be applied.
    We instead utilized the results of McDiarmid's inequality for dependent random variables shown by~\citet{zhang2019mcdiarmid}.
    \item The empirical loss considered in~\citet{zhang2022fmutual} is defined in a different way from our empirical kernel contrastive loss. 
    Also, our proof technique is different from~\citet{zhang2022fmutual}.
    \item \citet{haochen2021provable,nozawa2020pac} consider the case in which the augmented samples are not necessarily independent. \citet{nozawa2020pac} utilize the theory on PAC-Bayes bounds~\citep{guedj2019primer}, and~\citet{haochen2021provable} provide the high probability bound. Our analysis is different from~\citet{nozawa2020pac} since our analysis is based on several concentration inequalities.
    \citet{haochen2021provable} consider the empirical spectral contrastive loss that is defined by raw samples and expressed in the expectation w.r.t. the augmented samples that are drawn according to the conditional distribution given the raw samples (see Section~4.1 in~\citet{haochen2021provable}).
    On the other hand, we derive a generalization error bound for the empirical kernel contrastive loss defined by only augmented samples.
    \item \citet{wang2022spectral} establish the generalization error bound for the spectral contrastive loss~\citep{haochen2021provable}, where their analysis improves the convergence rate of~\citet{haochen2021provable}. In their analysis, they decompose the second term of the spectral contrastive loss in a different way from us (for the detail of their decomposition, see the proof of Proposition~D.1 of~\citet{wang2022spectral}).
    Also, they utilize the concentration inequality shown by~\citet{clemencon2008ranking} (see equation (51) in~\citet{wang2022spectral}), while we use the results proved by~\citet{zhang2019mcdiarmid}. Thus, the techniques we use in the proof of Theorem~\ref{thm:dependent ulln} are different from those of~\citet{wang2022spectral}.
\end{itemize}

\subsection{Detailed Comparison to~\citet{robinson2020contrastive}}
\label{appsubsec:comparison to robinson}

\citet{robinson2020contrastive} tackle the hard negative sampling problem in contrastive learning from both the theoretical and empirical perspectives.
They also establish generalization bounds for their hard negative objectives by introducing the 1-NN classifier (for their definition of the 1-NN classifier, see the statement of Theorem~5 in~\citet{robinson2020contrastive}).
We give a detailed comparison between our results and the theoretical analysis by~\citet{robinson2020contrastive}.
The main differences are listed below:
\begin{itemize}
    \item The problem setup of the theoretical results by~\citet{robinson2020contrastive} is based on that of~\citet{pmlr-v97-saunshi19a}, i.e., they rely on the conditional independence assumption. On the other hand, we does not rely on it, where we utilize the similarity function $\textup{sim}(\cdot,\cdot;\lambda)$ instead.
    \item In the proof of Theorem~5 of~\citet{robinson2020contrastive} (see also Theorem~8 and the proof of their work), the supervised loss is upper-bounded by the term $\mathbb{E}_{c}\mathbb{E}_{x,x^{+}\sim_{iid} p(\cdot|c)}\|f(x)-f(x^{+})\|^{2}$. 
    In summary, the differences between Theorem 5 of~\citet{robinson2020contrastive} and Theorem~\ref{thm:guarantee} of our work are: (i) In the numerator of the upper bound in the proof of Theorem~8 of~\citet{robinson2020contrastive}, the term mentioned above appears. On the other hand, in Theorem~\ref{thm:guarantee} of our work, the quantity $\mathfrak{a}(f)$ appears. 
    (ii) Our upper bound includes the quantity $\Delta_{\textup{min}}(f)$.
    (iii) We also note that the proof techniques used in Theorem~\ref{thm:guarantee} in our work are different from~\citet{robinson2020contrastive}.
    \item Note that the label employed in the analysis by~\citet{robinson2020contrastive} is a random variable, while our analysis employ the deterministic labeling function.
\end{itemize}

\subsection{Detailed Comparison to~\citet{huang2021towards,zhao2023arcl}}
\label{appsubsec:comparison to huang}

\citet{huang2021towards} present the generalization bounds that utilizes the 1-NN classifier (for the definition of the 1-NN classifier introduced in~\citet{huang2021towards}, see Section~2 in their paper).
Besides, \citet{zhao2023arcl} extend the results of~\citet{huang2021towards}.
Thus, it is worth discussing the differences between the results by~\citet{huang2021towards,zhao2023arcl} and our Theorem~\ref{thm:guarantee}.
We summarize the differences below:

\begin{itemize}
    \item \citet{huang2021towards} show that if the centers of clusters in the feature space are sufficiently apart from each other (note that they call it \emph{divergence}), then their supervised error function is upper bounded by the \emph{alignment} term up to several constants and parameters.
    Hence, their results do not show that the \textit{divergence} relates directly to the supervised error, i.e., the \textit{divergence} term does not appear in their upper bounds of the supervised error.
    On the other hand, we show that the quantities related to the \textit{divergence} in the RKHS can also contribute to upper-bounding the supervised error (see Theorem~\ref{thm:guarantee}).
    \item In~\citet{huang2021towards,zhao2023arcl}, it is little investigated to what range of encoder models their results can apply. On the other hand, our Theorem~\ref{thm:guarantee} requires only the meaningfulness (Definition~\ref{def:meaningful encoders}) of encoders belonging to $\mathcal{F}$. Especially, suppose $k$ is the linear kernel, then Theorem~\ref{thm:guarantee} in our study refines Theorem~1 of~\citet{huang2021towards} in this sense.
    \item \citet{huang2021towards,zhao2023arcl} utilize the notion termed $(\sigma,\delta)$\textit{-augmentation}, while our analysis utilizes Assumption~\ref{assumption:mixture of clusters} based on the definition of the similarity function $\textup{sim}(\cdot,\cdot;\lambda)$.
    \item \citet{huang2021towards,zhao2023arcl} often use the assumption that the encoder $f$ is a Lipschitz function. Meanwhile, our main result does not require that $f$ should be a Lipschitz function.
    \item \citet{zhao2023arcl} consider the squared loss for the downstream classification task (see Theorem~3.2 in their paper). On the other hand, we consider the classification error.
\end{itemize}

\section{Connections between KCL and Normalized Cut}
\label{appsec:connections between kcl and normalized cut}

In this section, we present supplementary information of Section~\ref{subsubsec:key ingredient}.
Throughout this section, we assume that $\inf_{x\in\mathbb{X}}w(x)<\infty$ and $\sup_{\overline{x}\in\overline{\mathbb{X}}}\sup_{x\in\mathbb{X}}a(x|\overline{x})<\infty$ hold.
Note that the assumption $\sup_{\overline{x}\in\overline{\mathbb{X}}}\sup_{x\in\mathbb{X}}a(x|\overline{x})<\infty$ implies $\sup_{x\in\mathbb{X}}w(x)<\infty$.

\subsection{The Problem Setup of Normalized Cut}
\label{appsubsec: review of ncut}
In this section, we first explain the population-level normalized cut problem based on~\citet{shi2000normalized,von2007tutorial,terada2019kernel}.
Suppose that there are total $K$ clusters in $\mathbb{X}$.
Following~\citet{terada2019kernel}, the optimization problem of the population-level normalized cut is given as:
\begin{align}
 \label{def:ncut}
 \min_{\mathbb{V}_{1},\cdots,\mathbb{V}_{K}}\sum_{i=1}^{K}\frac{W(\mathbb{V}_{i},\mathbb{V}_{i}^{c})}{\textup{vol}(\mathbb{V}_{i})}
\end{align}
where the minimum in the above problem is taken over all the possible combinations of $K$ disjoint non-empty measurable subsets $\mathbb{V}_{1},\cdots,\mathbb{V}_{K}$ satisfying $\bigcup_{i=1}^{K}\mathbb{V}_{i}=\mathbb{X}$, and $W$ and $\textup{vol}(\cdot)$ are defined as,
\begin{align}
  \label{def:cut capacity}
  W(\mathbb{V}_{i},\mathbb{V}_{i}^{c})&=\int_{(x,x')\in\mathbb{V}_{i}\times\mathbb{V}_{i}^{c}}\text{sim}(x,x';\lambda)w(x)w(x')d\nu_{\mathbb{X}}^{\otimes 2}(x,x'),\\
  \label{def:volume}
  \textup{vol}(\mathbb{V}_{i})&=\int_{\mathbb{V}_{i}}w(x)d\nu_{\mathbb{X}}(x),
\end{align}
where $\nu_{\mathbb{X}}^{\otimes 2}:=\nu_{\mathbb{X}}\otimes \nu_{\mathbb{X}}$ is the product measure.
Here also note that $\textup{vol}(\cdot)$ is the volume of a set $\mathbb{V}_{i}$.
\citet{terada2019kernel} consider the case that a reproducing kernel is used as similarity measurement: see Theorem 7 in~\citet{terada2019kernel}.
Note that some existing work deals with the measurable partition problems such as the ratio cut and Cheeger cut~\citep{trillos2016consistency}.

Denote by $L^{2}(\mathbb{X},P_{\mathbb{X}})$, the Hilbert space over the field $\mathbb{R}$ consisting of real-valued and squared-integrable function defined on $\mathbb{X}$ for $P_{\mathbb{X}}$-a.e., with its inner product $\langle f,g\rangle_{L^{2}(\mathbb{X},P_{\mathbb{X}})}=\int f(x)g(x)dP_{\mathbb{X}}(x)$.
Let $U:\mathbb{R}^{K}\to L^{2}(\mathbb{X},P_{\mathbb{X}})$ be a linear operator defined as,
\begin{align}
\label{def:assignment matrix}
 (Uz)(\cdot)=\sum_{i=1}^{K}\frac{\mathbbm{1}_{\mathbb{V}_{i}}(\cdot)}{\sqrt{\textup{vol}(\mathbb{V}_{i})}}z_{i},\;\;z=(z_1,\cdots,z_{K})^{\top},
\end{align}
where $\mathbbm{1}_{\mathbb{V}_{i}}(x)=1$ if $x\in\mathbb{V}_{i}$ and 0 if $x\notin \mathbb{V}_{i}$, and $z_{i}:=\langle z,e_{i}\rangle_{\mathbb{R}^{K}}$ for each $i\in[K]$ with an orthonormal basis $\{e_{i}\}_{i=1}^{K}$ of $\mathbb{R}^{K}$.
Note that under the setting that $|\mathbb{X}|<\infty$, the linear operator $U$ is equal to $(\mathbbm{1}_{\mathbb{V}_{j}}(x_{i})/\sqrt{\textup{vol}(\mathbb{V}_{j})})_{ij}$.
Therefore, the definition of $U$ matches that of the classical theory of normalized cut~\citep{shi2000normalized}.
Moreover, every augmented data $x\in \mathbb{X}$ belongs to one of the subsets $\mathbb{V}_{1},\cdots,\mathbb{V}_{K}$.
Since $\mathbb{V}_{1},\cdots,\mathbb{V}_{k}$ are disjoint, linear operator $U$ is bounded, and the adjoint operator $U^{\dag}$ exists uniquely.
Here, \citet{von2007tutorial} explain that the objective function of the normalized cut problem can be rewritten as a combinatorial optimization problem.
Applying the arguments presented by~\citet{von2007tutorial} to our setup, we have
\begin{align}
\label{eq:normalized cut rearranged}
  \sum_{i=1}^{K}\frac{W(\mathbb{V}_{i},\mathbb{V}_{i}^{c})}{\text{vol}(\mathbb{V}_{i})}&=-\textup{Tr}(U^{\dag}AU)+(1-\lambda)K,
\end{align}
where $A:L^{2}(\mathbb{X},P_{\mathbb{X}})\to L^{2}(\mathbb{X},P_{\mathbb{X}})$ is a Hilbert-Schmidt integral operator defined as
\begin{equation}
  \label{def: adjacency matrix}
  A\psi(\cdot)=\int \text{sim}(\cdot,x;\lambda)\psi(x)w(x)d\nu_{\mathbb{X}}(x)\quad \psi\in L^{2}(\mathbb{X},P_{\mathbb{X}}),
\end{equation}
and $-\textup{Tr}(U^{\dag}AU)=-\sum_{i=1}^{K}\langle U^{\dag}AUe_{i},e_{i}\rangle_{\mathbb{R}^{K}}$.
The proof of \eqref{eq:normalized cut rearranged} closely follows that of~\citet{von2007tutorial}; in Appendix~\ref{appsubsec:proof of (7)}, we present the proof of an extended version.
Here the following proposition shows the well-definedeness of $A$.

\begin{proposition}
\label{prop:A is well-defined}
Suppose the setting described in Section~\ref{subsec:problem setup} holds, $\inf_{x\in\mathbb{X}}w(x)>0$, and $\sup_{\overline{x}\in\overline{\mathbb{X}}}\sup_{x\in\mathbb{X}}a(x|\overline{x})<\infty$ holds. Then, the integral operator $A$ is well-defined.
\end{proposition}

\begin{proof}
We can evaluate
\begin{align*}
    \left|\int\textup{sim}^{2}(x,x;\lambda)w(x)d\nu_{\mathbb{X}}(x)\right|
    &=\left|\int\left(\frac{w(x,x)}{w(x)w(x)}-\lambda\right)^{2}w(x)d\nu_{\mathbb{X}}(x)\right|\\
    &\leq \int\left(\left(\frac{w(x,x)}{w(x)w(x)}\right)^{2}+2\lambda\frac{w(x,x)}{w(x)w(x)}+\lambda^{2}\right)w(x)d\nu_{\mathbb{X}}(x)\\
    &<+\infty,
\end{align*}
where we use the assumptions that $\inf_{x\in\mathbb{X}}w(x)>0$, $\sup_{\overline{x}\in\overline{\mathbb{X}}}\sup_{x\in\mathbb{X}}a(x|\overline{x})<\infty$.
\end{proof}

Suppose the dimension of the RKHS $\mathcal{H}_{k}$ associated with the kernel function $k$ is greater than or equal to $K$.
In the following section, it is convenient to redefine \eqref{def:ncut} as,
\begin{align*}
    \min_{\mathbb{V}_{1},\cdots,\mathbb{V}_{K}}\sum_{i=1}^{\infty}\frac{W(\mathbb{V}_{i},\mathbb{V}_{i}^{c})}{\textup{vol}(\mathbb{V}_{i})},
\end{align*}
where we define $\mathbb{V}_{j}=\emptyset$ for every $j>K$.
Also, let us redefine \eqref{def:assignment matrix} as the linear operator $U:\mathcal{H}_{k}\to L^{2}(\mathbb{X},P_{\mathbb{X}})$,
\begin{align*}
    (U\psi)(\cdot)=\sum_{i=1}^{\infty}\frac{\mathbbm{1}_{\mathbb{V}_{i}}(\cdot)}{\sqrt{\textup{vol}(\mathbb{V}_{i})}}\langle \psi,e_{i}\rangle_{\mathcal{H}_{k}},
\end{align*}
where $\{e_{j}\}_{j=1}^{\infty}$ is an orthonormal basis of $\mathcal{H}_{k}$ (if $\mathcal{H}_{k}$ is finite dimensional, then we understand that $\{e_{j}\}_{j=1}^{\infty}$ consists of finitely many non-zero elements), and we define $\mathbbm{1}_{\mathbb{V}_{i}}(\cdot)/\sqrt{\textup{vol}(\mathbb{V}_{i})}=0$ for every $i>K$ as notations.
Then, we have the following identity that is analogous of \eqref{eq:normalized cut rearranged}:
\begin{align}
    \label{eq:normalized cut rkhs}
    \sum_{i=1}^{\infty}\frac{W(\mathbb{V}_{i},\mathbb{V}_{i}^{c})}{\textup{vol}(\mathbb{V}_{i})}=-\textup{Tr}(U^{\dag}AU)+(1-\lambda)K.
\end{align}
For the sake of completeness, we provide the proof of the identity \eqref{eq:normalized cut rkhs} in Appendix~\ref{appsubsec:proof of (7)}.

\subsection{Connecting KCL and Normalized Cut via RKHS}
\label{appsubsec:generalized explanation}

Let $k:\mathbb{S}^{d-1}\times\mathbb{S}^{d-1}$ be a continuous, symmetric, and positive-definite kernel function whose RKHS $\mathcal{H}_{k}$ is $K$-dimensional Hilbert space ($K$ is either finite or $\infty$); For the theory of reproducing kernels, see e.g., \citet{aronszajn1950theory,berlinet2004reproducing,steinwart2008support}.
Many kernel functions satisfy these conditions, e.g. the Gaussian kernel, the polynomial kernel, and the linear kernel.
Since $\mathbb{S}^{d-1}$ is separable, the RKHS $\mathcal{H}_{k}$ has an orthonormal basis that is at most countable (e.g., see~\citet{berlinet2004reproducing}).
Let $\{e_{j}\}_{j=1}^{\infty}$ be a countable orthonormal basis of $\mathcal{H}_{k}$, where $\{e_{j}\}_{j=1}^{\infty}$ includes only finitely many non-zero elements if $\mathcal{H}_{k}$ is finite-dimensional.
Note that our construction is valid regardless of the choice of $\{e_{j}\}_{j=1}^{\infty}$.
Recall the problem setup presented in Section~\ref{subsec:problem setup}.
Then, the linear operator $H:\mathcal{H}_{k}\to L^{2}(\mathbb{X},P_{\mathbb{X}})$ is defined as
\begin{align}
 \label{eq:H cond}
 (H\varphi)(\cdot)=\langle h(f(\cdot)),\varphi\rangle_{\mathcal{H}_{k}},
\end{align}
for $\varphi\in\mathcal{H}_{k}$.
Let $\|\cdot\|_{\mathcal{H}_{k}}$ be the norm of the RKHS $\mathcal{H}_{k}$.
Then the following holds for the linear operator $H$ defined in \eqref{eq:H cond}:
\begin{proposition}
  \label{proposition: boundedness}
  The linear operator $H$ is well-defined, i.e., $H\varphi\in L^{2}(\mathbb{X},P_{\mathbb{X}})$ for every $\varphi\in\mathcal{H}_{k}$. Also, $H$ is continuous, i.e., for a sequence $\varphi_{j}$ converging strongly to $\varphi$ in $\mathcal{H}_{k}$, we have that $H\varphi_{j}$ is convergent to $H\varphi$.
\end{proposition}

\begin{proof}
  For any $\varphi\in\mathcal{H}_{k}$, we have
  \begin{align*}
    \int |(H\varphi)(x)|^{2}w(x)d\nu_{\mathbb{X}}(dx)
    &=\int|\langle h(f(x)),\varphi\rangle_{\mathcal{H}_{k}}|^{2}w(x)d\nu_{\mathbb{X}}(x)\\
    &\leq \int \|h(f(x))\|_{\mathcal{H}_{k}}^{2}\|\varphi\|_{\mathcal{H}_{k}}^{2}w(x)d\nu_{\mathbb{X}}(x)\\
    &\leq \|\varphi\|_{\mathcal{H}_{k}}^{2}\mathbb{E}_{x}\left[k(f(x),f(x))\right]<+\infty.
  \end{align*}
  Here in the second inequality we use the Cauchy-Schwarz inequality, and in the last equality we use the fact that $\mathbb{S}^{d-1}$ is compact and $k$ is continuous.
  Furthermore, if $\varphi_{j}\to\varphi$ in the sense of strongly convergence in $\mathcal{H}_{k}$, then we have
  \begin{align*}
    \|\langle h(f(\cdot)),\varphi_{j}\rangle_{\mathcal{H}_{k}}-\langle h(f(\cdot)),\varphi\rangle_{\mathcal{H}_{k}}\|_{L^{2}(\mathbb{X},P_{\mathbb{X}})}^{2}
    &=\int |\langle h(f(x)),\varphi_{j}-\varphi\rangle_{\mathcal{H}_{k}}|^{2}w(x)d\nu_{\mathbb{X}}(x)\\
    &\leq \|\varphi_{j}-\varphi\|_{\mathcal{H}_{k}}^{2}\mathbb{E}_{x}\left[k(f(x),f(x))\right]\\
    &\longrightarrow 0\quad(j\to\infty).
  \end{align*}
  Thus $H\varphi_{j}$ converges to $H\varphi$, and we end the proof.
\end{proof}
Proposition \ref{proposition: boundedness} implies that $H$ is bounded. Therefore, the adjoint operator $H^{\dag}:L^{2}(\mathbb{X},P_{\mathbb{X}})\to\mathcal{H}_{k}$ exists uniquely. 

Now let us recall the definition of the similarity function $\text{sim}(\cdot, \cdot;\lambda)$ with the fixed $\lambda$ in \eqref{def:similarity function}, and we consider to relax the combinatorial problem \eqref{def:ncut} using the linear operator $H$ defined in \eqref{eq:H cond} as follows:
we replace the linear operator $U$ in \eqref{eq:normalized cut rearranged} with $H$, which results in the objective function $-\text{Tr}(H^{\dag}AH)$.
Then, the following proposition holds.
\begin{proposition}
 \label{prop:operator ncut to expectation based ncut}
 We have
 \begin{align*}
 -\textup{Tr}(H^{\dag}AH)=-\mathbb{E}_{x,x^{+}}\left[k(f(x),f(x^{+}))\right]+\lambda \mathbb{E}_{x,x^{-}}\left[k(f(x),f(x^{-}))\right].
 \end{align*}
\end{proposition}

\begin{proof}
From the definition of $\textup{sim}(x,x';\lambda)$,
\begin{align*}
    (A\psi)(x)&=\int \textup{sim}(x,x';\lambda)\psi(x')w(x')d\nu_{\mathbb{X}}(x')\\
    &=\int \left(\frac{w(x,x')}{w(x)w(x')}-\lambda\right)\psi(x')w(x')d\nu_{\mathbb{X}}(x')\\
    &=\underbrace{\int \frac{w(x,x')}{w(x)w(x')}\psi(x')w(x')d\nu_{\mathbb{X}}(x')}_{:=(A_{\textup{pos}}\psi)(x)}-\lambda\underbrace{\int \psi(x')w(x')d\nu_{\mathbb{X}}(x')}_{:=(A_{\textup{neg}}\psi)(x)}.
\end{align*}
  Firstly, let us proof the identity
  \begin{align*}
    \textup{Tr}(H^{\dag}A_{\textup{pos}}H)=\mathbb{E}_{x,x^{+}}\left[k(f(x),f(x^{+}))\right].
  \end{align*}
  The proof is described as follows:
  From the definition of $H$,
  \begin{equation*}
    (He_{i})(x)=\langle h(f(x)),e_{i}\rangle_{\mathcal{H}_{k}}\quad x\in\mathbb{X}.
  \end{equation*}
  Then we have
  \begin{align*}
    (A_{\textup{pos}}He_{i})(x)
    &=\int \frac{w(x,x')}{w(x)w(x')}w(x')\langle h(f(x')),e_{i}\rangle_{\mathcal{H}_{k}}d\nu_{\mathbb{X}}(x')\\
    &=\int\frac{w(x,x')}{w(x)}\langle h(f(x')),e_{i}\rangle_{\mathcal{H}_{k}}d\nu_{\mathbb{X}}(x').
  \end{align*}
  Here, the adjoint operator $H^{\dag}$ satisfies the following identity; For $\psi\in L^{2}(\mathbb{X},P_{\mathbb{X}})$,
  \begin{equation*}
  \langle He_{i},\psi\rangle_{L^{2}(\mathbb{X},P_{\mathbb{X}})}=\langle e_{i},H^{\dag}\psi\rangle_{\mathcal{H}_{k}}.
  \end{equation*}
  Utilizing this relation yields the following representation:
  \begin{align*}
    &\;\;\;\;H^{\dag}A_{\textup{pos}}He_{j}\\
    &=\sum_{i=1}^{\infty}\langle H^{\dag}A_{\textup{pos}} H e_{j},e_{i}\rangle_{\mathcal{H}_{k}}e_{i}\\
    &=\sum_{i=1}^{\infty}\langle A_{\textup{pos}}He_{j}, He_{i}\rangle_{L^{2}(\mathbb{X},P_{\mathbb{X}})}e_{i}\\
    &=\sum_{i=1}^{\infty} \left(\int w(x)\langle h(f(x)),e_{i}\rangle_{\mathcal{H}_{k}}\int \frac{w(x,x')}{w(x)}\langle h(f(x')),e_{j}\rangle_{\mathcal{H}_{k}}d\nu_{\mathbb{X}}(x')d\nu_{\mathbb{X}}(x)\right)e_{i}\\
    &=\sum_{i=1}^{\infty}\left(\int\int w(x,x')\langle h(f(x)),e_{i}\rangle_{\mathcal{H}_{k}}\langle h(f(x')),e_{j}\rangle_{\mathcal{H}_{k}}d\nu_{\mathbb{X}}(x)d\nu_{\mathbb{X}}(x')\right)e_{i}.
  \end{align*}
  Therefore,
  \begin{align*}
    \textup{Tr}(H^{\dag}A_{\textup{pos}}H)&=\sum_{j=1}^{\infty}\langle H^{\dag}A_{\textup{pos}}He_{j},e_{j}\rangle_{\mathcal{H}_{k}}\\
    &=\sum_{j=1}^{\infty}\int\int w(x,x')\langle h(f(x)),e_{j}\rangle_{\mathcal{H}_{k}}\langle h(f(x')),e_{j}\rangle_{\mathcal{H}_{k}}d\nu_{\mathbb{X}}(x)d\nu_{\mathbb{X}}(x')\\
    &=\int\int w(x,x')\sum_{j=1}^{\infty}\langle h(f(x)),e_{j}\rangle_{\mathcal{H}_{k}}\langle h(f(x')),e_{j}\rangle_{\mathcal{H}_{k}}d\nu_{\mathbb{X}}(x)d\nu_{\mathbb{X}}(x')\\
    &=\int\int w(x,x')\langle h(f(x)),h(f(x'))\rangle_{\mathcal{H}_{k}}d\nu_{\mathbb{X}}(x)d\nu_{\mathbb{X}}(x')\\
    &=\mathbb{E}_{x,x^{+}}\left[ k(f(x),f(x^{+}))\right].
  \end{align*}
  Note that the third equality above is due to the Dominated Convergence Theorem.
  Indeed, the sum $\sum_{j=1}^{n}\langle h(f(x)),e_{j}\rangle_{\mathcal{H}_{k}}\langle h(f(x')),e_{j}\rangle_{\mathcal{H}_{k}}$ converges pointwisely to $\langle h(f(x)),h(f(x))\rangle_{\mathcal{H}_{k}}$ on $\mathbb{X}\times\mathbb{X}$, and
  \begin{align*}
    \left|\sum_{j=1}^{n}\langle h(f(x)),e_{j}\rangle_{\mathcal{H}_{k}}\langle h(f(x')),e_{j}\rangle_{\mathcal{H}_{k}}\right|&\leq \sum_{j=1}^{n}\left|\langle h(f(x)),e_{j}\rangle_{\mathcal{H}_{k}}\langle h(f(x')),e_{j}\rangle_{\mathcal{H}_{k}}\right|\\
    &\leq \left(\sum_{j=1}^{n}\langle h(f(x)),e_{j}\rangle_{\mathcal{H}_{k}}^{2}\right)^{1/2}\left(\sum_{j=1}^{n}\langle h(f(x')),e_{j}\rangle_{\mathcal{H}_{k}}^{2}\right)^{1/2}\\
    &\leq \left(\sum_{j=1}^{\infty}\langle h(f(x)),e_{j}\rangle_{\mathcal{H}_{k}}^{2}\right)^{1/2}\left(\sum_{j=1}^{\infty}\langle h(f(x')),e_{j}\rangle_{\mathcal{H}_{k}}^{2}\right)^{1/2}\\
    &=\|h(f(x))\|_{\mathcal{H}_{k}}\|h(f(x'))\|_{\mathcal{H}_{k}}\\
    &\leq \sup_{x\in\mathbb{X}}k(f(x),f(x))<+\infty.
  \end{align*}
  On the other hand, it is obvious that,
  \begin{align*}
      \langle H^{\dag}A_{\textup{neg}}He_{j},e_{j}\rangle_{\mathcal{H}_{k}}&=\langle A_{\textup{neg}}He_{j},He_{j}\rangle_{L^{2}(\mathbb{X},P_{\mathbb{X}})}\\
      &=\int w(x)\langle h(f(x)),e_{j}\rangle_{\mathcal{H}_{k}}\int w(x')\langle h(f(x')),e_{j}\rangle_{\mathcal{H}_{k}}d\nu_{\mathbb{X}}(x')d\nu_{\mathbb{X}}(x)\\
      &=\int\int \langle h(f(x)),e_{j}\rangle_{\mathcal{H}_{k}}\langle h(f(x')),e_{j}\rangle_{\mathcal{H}_{k}}w(x)w(x')d\nu_{\mathbb{X}}(x)d\nu_{\mathbb{X}}(x').
  \end{align*}
  Hence we obtain,
  \begin{align*}
      \textup{Tr}(H^{\dag}A_{\textup{neg}}H)
      &=\sum_{j=1}^{\infty}\langle H^{\dag}A_{\textup{neg}}He_{j},e_{j}\rangle_{\mathcal{H}_{k}}\\
      &=\int\int \langle h(f(x)),h(f(x'))\rangle_{\mathcal{H}_{k}}w(x)w(x')d\nu_{\mathbb{X}}(x)d\nu_{\mathbb{X}}(x')\\
      &=\mathbb{E}_{x,x^{-}}\left[k(f(x),f(x^{-}))\right].
  \end{align*}
  Hence, we obtain the desired results and end the proof.
\end{proof}

\subsubsection{Comparison with Related Work from the Graph Cut Viewpoint}
\label{appsubsubsec:comparison with haochen from the graph cut viewpoint}

\citet{haochen2021provable} has already investigated links between the population-level spectral clustering and contrastive learning.
However, our integral kernel \eqref{def:similarity function} introduced in Section~\ref{subsubsec:key ingredient} is slightly different from that of~\citet{haochen2021provable}, since 1) we divide $w(x,x')$ by $w(x)w(x')$ in the first term rather than by $\sqrt{w(x)w(x')}$ (see Appendix~F in~\citet{haochen2021provable}), 2) we also incorporate the hyperparemeter $\lambda$.

Note that~\citet{tian2022understanding} introduces a unified framework termed $\alpha$-CL, which connects various contrastive losses from the coordinate-wise optimization perspective.
In~\citet{tian2022understanding}, the contrastive covariance plays a central role in the theoretical analysis.
On the other hand, we use the similarity function defined in Section~\ref{subsec:assumptions and definitions}, and thus the approach of our analysis is different from the contrastive covariance of~\citet{tian2022understanding}.

\subsection{Proof of \eqref{eq:normalized cut rkhs}}
\label{appsubsec:proof of (7)}

\begin{proof}[Proof of \eqref{eq:normalized cut rkhs}]
  For the proof of \eqref{eq:normalized cut rkhs}, we closely follow the approaches presented in Section~5 of~\citet{von2007tutorial}.
  Since we consider the population-level normalized cut, we present the proof of \eqref{eq:normalized cut rkhs} for the sake of completeness.
  
  Let us define the identity operator $D:L^{2}(\mathbb{X},P_{\mathbb{X}})\to L^{2}(\mathbb{X},P_{\mathbb{X}})$ as $D\psi=\psi$ for $\psi\in L^{2}(\mathbb{X},P_{\mathbb{X}})$.
  From the definitions of $D$ and $U$, we have
  \begin{equation*}
    DUe_{i}=
    \begin{cases}
      \frac{\mathbbm{1}_{\mathbb{V}_{i}}(\cdot)}{\sqrt{\textup{vol}(\mathbb{V}_{i})}}&\quad (i\leq K),\\
      0&\quad (i>K).
    \end{cases}
  \end{equation*}
  Hence, we have the following for $i\leq K$:
  \begin{equation*}
    \langle U^{\dag}DUe_{i},e_{i}\rangle_{\mathcal{H}_{k}}=\int \frac{w(x)\mathbbm{1}_{\mathbb{V}_{i}}(x)^{2}}{\textup{vol}(\mathbb{V}_{i})}d\nu_{\mathbb{X}}(x)=1.
  \end{equation*}
  On the other hand, for $i\leq K$ we have
  \begin{align*}
    &\;\;\;\;\langle U^{\dag}AUe_{i},e_{i}\rangle_{\mathcal{H}_{k}}\\
    &=\int \int (w(x,x')-\lambda w(x)w(x'))\frac{\mathbbm{1}_{\mathbb{V}_{i}}(x)}{\sqrt{\textup{vol}(\mathbb{V}_{i})}}\frac{\mathbbm{1}_{\mathbb{V}_{i}}(x')}{\sqrt{\textup{vol}(\mathbb{V}_{i})}}d\nu_{\mathbb{X}}(x')d\nu_{\mathbb{X}}(x)
  \end{align*}
  Therefore, we obtain the following:
  \begin{align*}
    \textup{Tr}(U^{\dag}(D-A)U)&=\sum_{i=1}^{\infty}\langle U^{\dag}(D-A)Ue_{i},e_{i}\rangle_{\mathcal{H}_{k}}\\
    &=\sum_{i=1}^{K}\langle U^{\dag}(D-A)Ue_{i},e_{i}\rangle_{\mathcal{H}_{k}}\\
    &=\lambda K+\frac{1}{2}\sum_{i=1}^{K}\int\int \left(w(x,x')-\lambda w(x)w(x')\right)\left(\frac{\mathbbm{1}_{\mathbb{V}_{i}}(x)}{\sqrt{\textup{vol}(\mathbb{V}_{i})}}-\frac{\mathbbm{1}_{\mathbb{V}_{i}}(x')}{\sqrt{\textup{vol}(\mathbb{V}_{i})}}\right)^{2}d\nu_{\mathbb{X}}^{\otimes 2}(x,x')\\
    &=\lambda K+\sum_{i=1}^{K}\int_{x\in\mathbb{V}_{i}}\int_{x'\in\mathbb{V}_{i}^{c}} \frac{\left(w(x,x')-\lambda w(x)w(x')\right)}{\textup{vol}(\mathbb{V}_{i})}d\nu_{\mathbb{X}}(x')d\nu_{\mathbb{X}}(x)\\
    &=\lambda K+\sum_{i=1}^{K}\int_{x\in\mathbb{V}_{i}}\int_{x'\in\mathbb{V}_{i}^{c}} \frac{\textup{sim}(x,x';\lambda)}{\textup{vol}(\mathbb{V}_{i})}w(x)w(x')d\nu_{\mathbb{X}}(x')d\nu_{\mathbb{X}}(x)\\
    &=\lambda K+\sum_{i=1}^{K}\frac{W(\mathbb{V}_{i},\mathbb{V}_{i}^{c})}{\textup{vol}(\mathbb{V}_{i})}.
  \end{align*}
  Hence we end the proof.
\end{proof}

\section{Experiments}
\label{appsec:experiment details}

\subsection{Experimental setup}
\label{appsubsec:experimental setup}

We provide the setting of the experiments presented in this paper.
The code used in our experiments is based on the official implementation of SimSiam\footnote{\url{https://github.com/facebookresearch/simsiam} (Last accessed: March 25, 2023)} 
and written with PyTorch~\citep{paszke2019pytorch}.
We basically follow the experimental setting of~\citet{chen2021exploring}.
For the sake of completeness, we provide the detail of the setup used in our experiments.
During the stage of pretraining, we construct a trainable encoder model as follows: following~\citet{chen2021exploring}, we use a backbone architecture whose parameters are initialized, followed by the MLP that consists of linear layers, batch normalization~\citep{ioffe2015batch}, and the ReLU activation function.
Note that this type of MLP is called \textit{projection head}~\citep{chen2020simple}.
The output of a trainable encoder model is normalized using the Euclidean norm as several works do~\citep{chen2020simple,dwibedi2021little}.
On the other hand, during the stage of linear evaluation~\citep{chen2020simple,chen2021exploring}, the additional MLP is removed from a trained encoder model, and then a linear classification head is added to the encoder model.
The parameters of the trained encoder model are frozen in this stage, and only the linear head is trained.
In all of the experiments reported in this paper, we use ResNet-18~\citep{he2016deep} as the backbone architecture.
We use the 2-layer multi-layer perceptron for the projection head, where the first linear layer is bias-free, and the last linear layer has the bias term.

For the pretraining, we use the same data augmentation techniques as~\citet{chen2020improved}; Chen and He~\citet{chen2021exploring}.
Note that following Chen and He~\citet{chen2021exploring}, for the CIFAR-10 experiments, we exclude the Gaussian blur augmentation.
For the linear evaluation, we also follow the data augmentation techniques of~\citet{chen2021exploring}.
Note that in both the stage of pretraining and linear evaluation, we set drop\_last to True in the training data loader.

For optimization during both the stage of pretraining and linear evaluation, following~\citet{chen2021exploring}, we use the SGD optimizer.
Inspired by~\citet{haochen2021provable}, we use the cosine-decay learning rate scheduler~\citep{loshchilov2017sgdr} with warmup~\citep{goyal2017accurate}.
Note that we use the cosine-decay learning rate scheduler in both the pretraining and linear evaluation and apply warmup in the pretraining.
Following the implementation of the learning rate scheduler of the official implementation of SCL\footnote{\url{https://github.com/jhaochenz/spectral_contrastive_learning/blob/ee431bdba9bb62ad00a7e55792213ee37712784c/optimizers/lr_scheduler.py} (Last accessed: March 25, 2023)}, we also define our learning rate scheduler by the number of iterations.

\subsubsection{Configurations}
\label{appsubsubsec:configurations}
In all of the experiments reported in this paper, we use the following configurations: 
\paragraph{Pretraining}
For the learning rate, we set the initial learning rate to 0.0005, the base learning rate to 0.05, and the warmup epochs to 10.
Following~\citet{chen2021exploring}, we use the linear scaling~\citep{goyal2017accurate} for the learning rate.
For the setting of the SGD optimizer, we also follow the setting of~\citet{chen2021exploring} used for their CIFAR-10 experiments (see Appendix~D in their paper): the momentum is set 0.9, and the weight decay is set 0.0005.
For the output dimension of encoders, we set 512.
\paragraph{Linear Evaluation}
We follow the configurations of~\citet{chen2021exploring} for linear evaluation: for the SGD optimizer, the momentum is 0.9, the weight decay is 0, the batch size is fixed to 256, and the learning rate is 30.0, where the linear scaling~\citep{goyal2017accurate} is applied to the learning rate.
Note that we train the linear head for 100 epochs.

\subsubsection{Kernel Functions}
\label{appsubsubsec:kernel functions used in the experiments}
In the experiments, we use the following kernel functions:

\paragraph{Gaussian Kernel.} The Gaussian kernel $k_{\textup{Gauss}}$ is defined as,
\begin{align*}
    k_{\textup{Gauss}}(z,z')=\exp\left(-\frac{\|z-z'\|_{2}^{2}}{\sigma^{2}}\right),
\end{align*}
where $\sigma^{2}>0$ is the bandwidth parameter.

\paragraph{Quadratic Kernel.} The Quadratic kernel $k_{\textup{Gauss}}$ is defined as,
\begin{align*}
    k_{\textup{Quad}}(z,z')=\left(z^{\top}z'\right)^{2}.
\end{align*}

\subsubsection{Loss Functions}
\label{appsubsubsec:algorithms for our experiments}

For the implementation of the kernel contrastive loss, we implement the empirical kernel contrastive loss~\eqref{def:empirical kernel contrastive loss}.
Note that in our experiments, the KCL frameworks with the Gaussian kernel and quadratic kernel are called Gaussian KCL (GKCL), and Quadratic KCL (QKCL), respectively.

For comparison, we also perform several reproducing experiments for SimCLR~\citep{chen2020simple} and SCL~\citep{haochen2021provable}.
For the implementation of the objective function of SimCLR, we use lightly.loss.NTXentLoss\footnote{Note that in our experiments, we use Lightly 1.2.25.} of Lightly~\citep{susmelj2020lightly}.
For the implementation of the spectral contrastive loss of SCL, we adapt the implementation of the official SCL code\footnote{\url{https://github.com/jhaochenz/spectral_contrastive_learning/blob/ee431bdba9bb62ad00a7e55792213ee37712784c/models/spectral.py} (Last accessed: March 25, 2023)}.

\subsubsection{Datasets}
\label{appsubsubsec:datasets used in experiments}
In the experiments, we use the following datasets: CIFAR-10~\citep{krizhevsky2009learning}, STL-10~\citep{coates2011analysis}, and ImageNet-100~\citep{tian2020contrastive}.
Note that ImageNet-100 is a subset of the ImageNet-1K dataset~\citep{deng2009imagenet}, where the ImageNet-100 dataset contains images categorized in 100 classes~\citep{tian2020contrastive}.
When extracting images from the original the ImageNet-1K dataset to create the ImageNet-100 dataset, we select the 100 classes used in~\citet{tian2020contrastive}.
We also remark that for the experiments with the STL-10 dataset, we use the mixed dataset that consists of the unlabeled images and the labeled training images for pretraining, the labeled training images for the training of the linear head in the stage of linear evaluation, and the labeled test images for computing the accuracy in linear evaluation.
Throughout the experiments, we use the following image size for each dataset: $32\times 32$ pixels for CIFAR-10, $96\times96$ pixels for STL-10, and $224\times 224$ pixels for ImageNet-100, where the image sizes of CIFAR-10 and STL-10 are the same as the sizes of the original images, respectively, and the image sizes for ImageNet-100 are inspired by those for the ImageNet-1K dataset used in~\citet{chen2020simple,chen2021exploring}.

\subsubsection{Detail of Architectures for the CIFAR-10 Experiments}
\label{appsubsubsec:detail of architectures}
In the experiments with the CIFAR-10 dataset, following the settings of~\citet{he2016deep,chen2021exploring,haochen2021provable}, we modify the original ResNet-18~\citep{he2016deep} as follows: in the implementation code of ResNet\footnote{\url{https://github.com/pytorch/vision/blob/eac3dc7bab436725b0ba65e556d3a6ffd43c24e1/torchvision/models/resnet.py} (Last accessed: March 26, 2023)} of torchvision~\citep{torchvision2016}, we replace the first convolution layer with that whose kernel size is 3, stride is 1, and padding is 1, and the maxpool layer with that whose kernel size is 1 and stride is 1.

\subsubsection{Supplementary Information of the Implementation}
\label{appsubsubsec:information of the implementation}
We use the following packages for the experiments: PyTorch~\citep{paszke2019pytorch}, torchvision~\citep{torchvision2016}, NumPy~\citep{harris2020array}, Lightly~\citep{susmelj2020lightly}, Matplotlib~\citep{hunter2007matplotlib}, and seaborn~\citep{waskom2021seaborn}.

\subsection{Results of Linear Evaluation}
\label{appsubsec:results of linear evaluation}

We perform pretraining and linear evaluation with the CIFAR-10, STL-10, and ImageNet-100 datasets.
In the stage of pretraining, we train the encoder models for 800 epochs.
For the experiments with the CIFAR-10 and STL-10 datasets, the batch sizes are set 256.
Besides, for the experiments with the ImageNet-100 dataset, we set 512 for the batch sizes.
We select the following hyperparameters of the KCL frameworks for all the experiments reported in this subsection: $\sigma^{2}=1$ and $\lambda=8$ for GKCL, and $\lambda=4$ for QKCL.
For SCL, inspired by~\citet{haochen2021provable}, we select $3$ for the radius parameter.
For SimCLR, inspired by~\citet{chen2020simple}, we select $0.1$ for the temperature parameter.
These hyperparameters are also used for all the experiments in this subsection.

The results are shown in Table~\ref{tab:lin eval results}.
In Table~\ref{tab:lin eval results}, the experiments with the CIFAR-10 dataset are performed using one Quadro P6000 GPU.
Besides, the experiments with STL-10 and ImageNet-100 are performed using one Tesla V100S GPU.

\begin{table}[!t]
    \caption{Top-1 and Top-5 accuracy (\%) in linear evaluation for each method.
    For the CIFAR-10 and STL-10 experiments, we perform three trials of "pretraining+linear evaluation," and the results indicate the mean$\pm$standard deviation.
    For the ImageNet-100 experiments, we perform one trial of "pretraining+linear evaluation."
    All the results reported below are obtained as follows: we trained the linear heads for 100 epochs and evaluated the final classification accuracy of the models with the corresponding validation or test dataset.
    The word "repro." is the abbreviation for "reproducing," meaning that we performed several reproducing experiments to compare to the performance of KCL.
    }
    \label{tab:lin eval results}
    \begin{center}
    \begin{small}
    \begin{tabular}{lcccccccccc}
        \toprule
         & \multicolumn{3}{c}{CIFAR-10} & \; & \multicolumn{3}{c}{STL-10} & \; & \multicolumn{2}{c}{ImageNet-100} \\
        \cline{2-4}
        \cline{6-8}
        \cline{10-11}
         & Top-1 & \, & Top-5 & \; & Top-1 & \, & Top-5 & \; & Top-1 & Top-5\\
        \midrule
        SimCLR (repro.)& 90.09$\pm$0.05 & \, & 99.70$\pm$0.01 & \; & 87.23$\pm$0.35 & \, & 99.56$\pm$0.04 & \; & 77.26 & 94.06 \\
        SCL (repro.) & 91.53$\pm$0.10 & \, & 99.71$\pm$0.05 & \; & 86.68$\pm$0.12 & \, & 99.49$\pm$0.07 & \; & 75.22 & 93.36 \\
        \midrule
        GKCL & 90.87$\pm$0.08 & \, & 99.63$\pm$0.00 & \; & 86.69$\pm$0.09 & \, & 99.38$\pm$0.01 & \; & 76.40 & 93.20 \\
        QKCL & 90.62$\pm$0.10 & \, & 99.59$\pm$0.05 & \; & 87.07$\pm$0.20 & \, & 99.37$\pm$0.03 & \; & 77.12 & 93.96 \\
        \bottomrule
    \end{tabular}
    \end{small}
    \end{center}
\end{table}

\subsection{More Experiments}
\label{appsubsec:ablation study}

\subsubsection{Ablation Study on the Weight Parameters and Batch Sizes}
\label{appsubsubsec:ablation study on the weight and batch sizes}
We investigate how the selection of $\lambda$ and batch sizes affect the quality of learned representations.
In the experiments, we use the CIFAR-10 dataset.
We select $\{1, 2, 4, 8, 16, 32\}$ for $\lambda$, and $\{64, 128, 256, 512, 1024\}$ for the batch sizes.
In each run, we pretrain an encoder model for 200 epochs.
We use both GKCL and QKCL and evaluate those results.

\begin{figure}[H]
    \centering
    \includegraphics[width=0.7\linewidth]{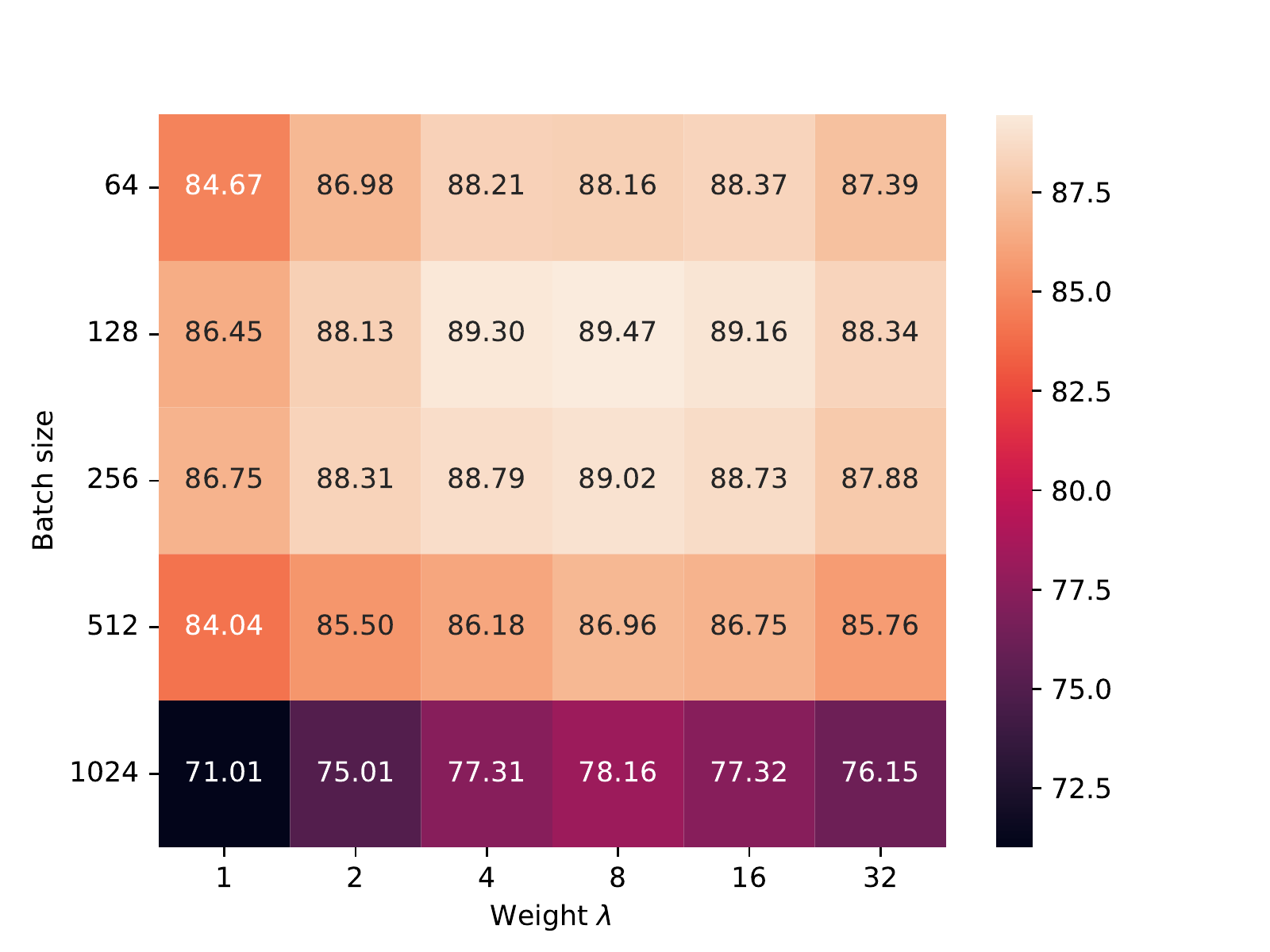}
    \caption{The results of ablation study for GKCL. The number in each cell indicates the Top-1 accuracy (\%).}
    \label{fig:batch size weight ablation gkcl}
\end{figure}
\begin{figure}[h]
    \centering
    \includegraphics[width=0.7\linewidth]{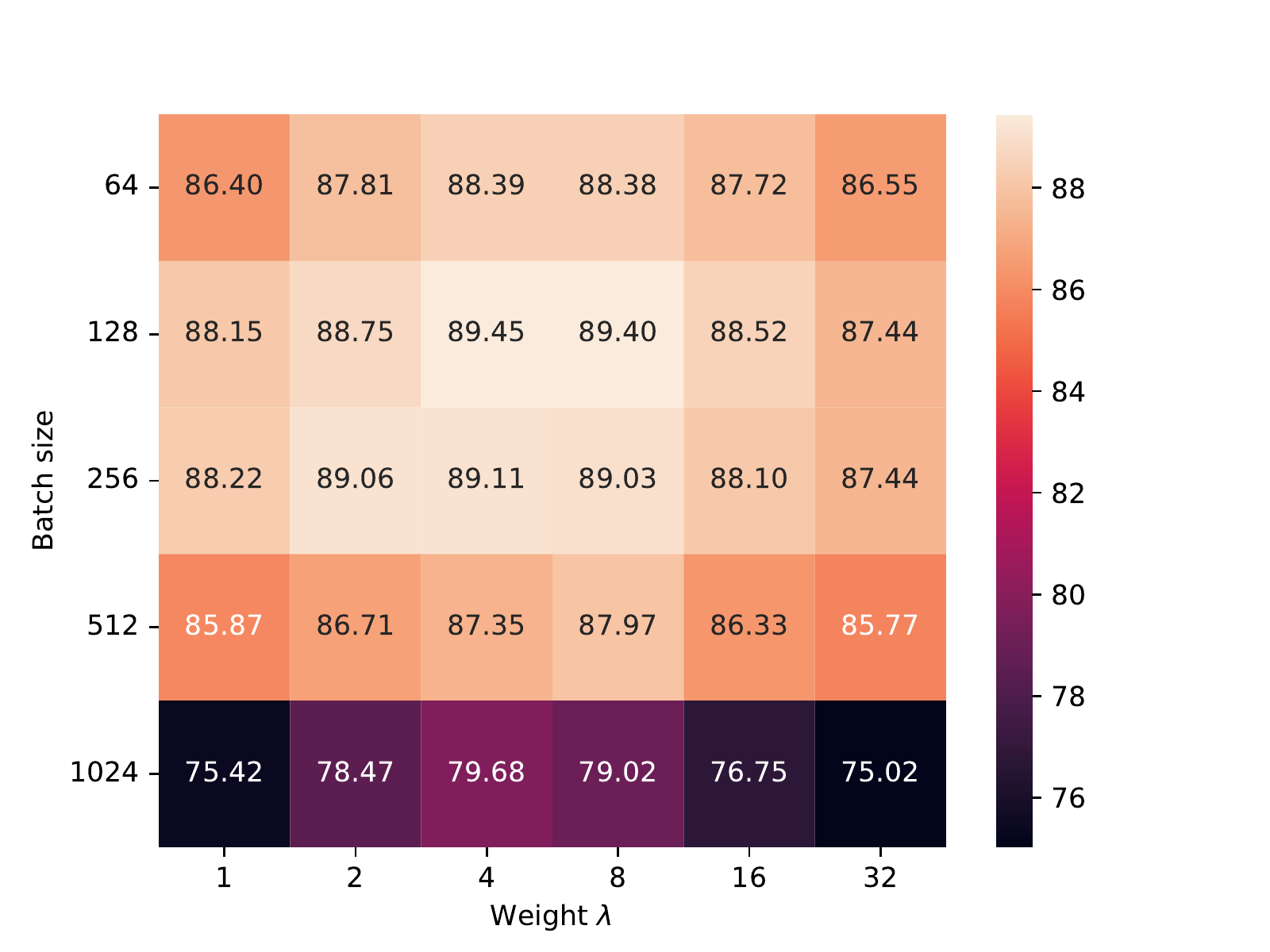}
    \caption{The results of ablation study for QKCL. The number in each cell indicates the Top-1 accuracy (\%).}
    \label{fig:batch size weight ablation qkcl}
\end{figure}

The Top-1 accuracy computed at the end of linear evaluation for GKCL and QKCL are shown in Figure~\ref{fig:batch size weight ablation gkcl} and \ref{fig:batch size weight ablation qkcl}, respectively.
Note that the experiments reported in Figure~\ref{fig:batch size weight ablation gkcl} and \ref{fig:batch size weight ablation qkcl} are performed by using one Tesla V100S GPU.
The results of the experiments indicate that 1) the selection of the value for $\lambda$ affects the quality of the representations learned by KCL, 2) the small batch sizes (e.g., 128 and 256) are more efficient when pretraining encoders, while the large batch sizes (e.g., 1024) degrade the performance.
Note that~\citet{chen2021intriguing} showed similar findings to the first point for the generalized NT-Xent loss.
Besides, \citet{chen2021exploring} point out the efficiency of SimSiam with small batch sizes.

\subsubsection{How Does $\lambda$ Influence the Geometry of Representations Learned?}
\label{appsubsubsec:how does lambda influences the geometry of representations learned}

From Theorem~\ref{thm:decomposition}, minimization of the kernel contrastive loss makes $\lambda\cdot\mathfrak{c}(f)$ smaller, which can imply that the means of the clusters tend to distribute uniformly as $\lambda$ increases.
Motivated by this result, in this subsubsection, we simulate how the mean of the feature vectors belonging to each cluster distributes.
In the experiments, we use the STL-10 dataset.
In the stage of unsupervised pretraining, we use the combination of the unlabeled images and the labeled training images in the STL-10 dataset.
We use GKCL for the pretraining.
The weights used in the experiments are $\{1,2,4,8,16,32,64\}$.
We pretrain the encoder model for 400 epochs in each run.
We set the batch sizes to 256.
After the stage of pretraining, we compute the mean for each class and calculate the cosine similarities between those means.
Since the clusters $\mathbb{M}_{1},\cdots,\mathbb{M}_{K}$ are hard to obtain for the STL-10 dataset, we instead use the labels included in the labeled training images of the STL-10 dataset to compute the mean over the feature vectors of augmented data transformed from raw data in each class.
Note that we draw an augmented image from each raw image when computing the means.
The experiments in this subsubsection are performed by using one Tesla V100S GPU.

\begin{figure}[h]
    \centering
    \includegraphics[width=0.8\linewidth]{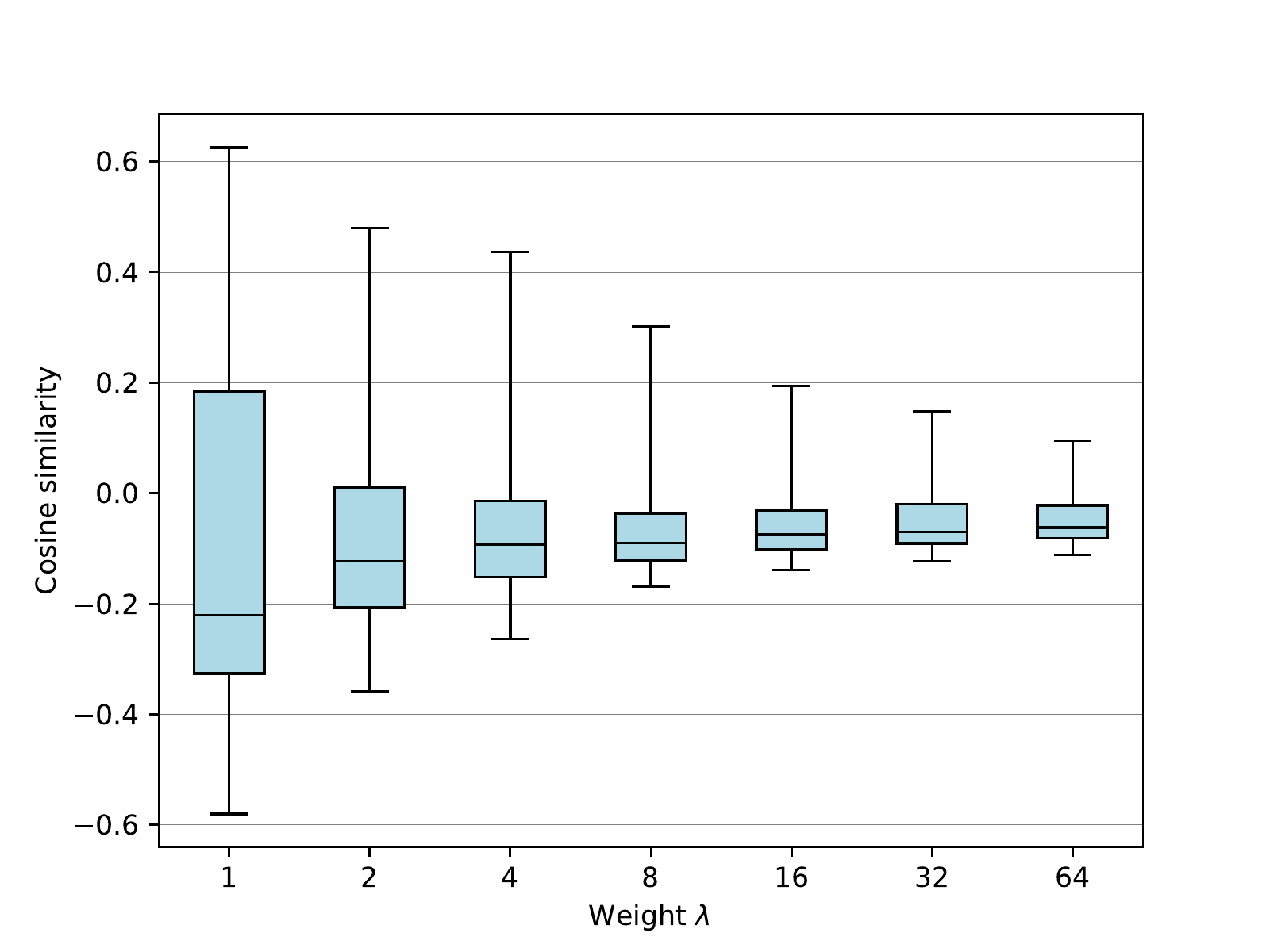}
    \caption{The box plot of the cosine similarities between the means of the different classes for each encoder model pretrained with different $\lambda$. 
    Note that the horizontal lines in each bar represent, in the order from bottom to top, the minimum value, the first quartile, the median, the third quartile, and the maximum value, respectively.}
    \label{fig:boxplot}
\end{figure}

The results are summarized in Table~\ref{fig:boxplot} as a box plot.
The results indicate that the variation becomes smaller as $\lambda$ increases.
Thus, larger $\lambda$ makes the means to distribute more uniformly in this experimental setting.
Note that this result may imply that the clusters $\mathbb{M}_{1},\cdots,\mathbb{M}_{K}$ and the subsets defined with labels have some relation.
We leave the investigation of this question as future work.

\subsection{Too Large $\lambda$ Degrades the Performance in Downstream Classification Tasks}
\label{appsubsec:larger lambda result in bad result}

In this subsection, we report the results of the experiments with different values for $\lambda$.
In the experiments, we use $\{1,2,4,8,16,32,64,128,256,512\}$ for the weight $\lambda$.
For the contrastive learning framework, we use GKCL.
We use two datasets, CIFAR-10 and STL-10, and pretrain the encoder during 400 epochs in each run.
In the stage of pretraining, we set 128 for the batch size.
Each experiment reported in this subsection is performed using one Tesla V100S GPU.

\begin{table}[H]
    \centering
    \caption{Top-1 accuracy (\%) in the results of linear evaluation, where the encoder is pretrained with different $\lambda$ for each run.}
    \label{tab:weight effects}
    \begin{tabular}{lcccc}
    \toprule
     & & \multicolumn{3}{c}{Top-1 Accuracy} \\
    \cline{3-5} 
        $\lambda$ & \; & CIFAR-10 & \; & STL-10   \\
    \midrule
        1 & \; & 89.06 & \; & 83.20 \\
        2 & \; & 90.29 & \; & 84.54 \\
        4 & \; & 90.77 & \; & 85.36 \\
        8 & \; & 90.66 & \; & 85.33 \\
        16 & \; & 90.52 & \; & 84.19 \\
        32 & \; & 89.78 & \; & 83.39 \\
        64 & \; & 88.85 & \; & 81.50 \\
        128 & \; & 86.93 & \; & 79.71 \\
        256 & \; & 85.24 & \; & 77.89 \\
        512 & \; & 82.43 & \; & 76.26 \\
    \bottomrule
    \end{tabular}
\end{table}

The results on the Top-1 accuracy at the end of the linear evaluation are presented in Table~\ref{tab:weight effects}.
The results indicate that too large $\lambda$, such as 512, degrades the performance in the downstream task.

\end{document}